\pdfoutput=1
\documentclass[twoside,11pt]{article}

\usepackage{jmlr2e_stripped}
\usepackage{enumitem}
\usepackage{amsmath,amsfonts}
\usepackage{bm}
\usepackage[usenames,dvipsnames]{xcolor}
\usepackage[colorlinks,
            linkcolor=blue,
            citecolor=dark-green,
            urlcolor=magenta,
            linktocpage,
	    raiselinks=true,
            plainpages=false]{hyperref}
\usepackage[capitalize]{cleveref}
\usepackage{transparent}  
\usepackage{stmaryrd}

\newcounter{dummy}  

\renewcommand{\epsilon}{\varepsilon}

\newtheorem{theorem}{Theorem}[section]
\newtheorem{fact}[theorem]{Fact}
\newtheorem{assumption}{Assumption}

\newtheorem{example}[theorem]{Example}
\newenvironment{myexample}[1][]{\begin{example}\normalfont%
\ifx\newenvironment#1\newenvironment\else{\bfseries {\bf (#1)}}\fi}{\hfill\BlackBox\\[2mm]\end{example}}

\DeclareMathOperator*{\E}{{\bf E}}
\DeclareMathOperator*{\Var}{{\bf Var}}
\DeclareMathOperator*{\argmin}{arg\,min}
\DeclareMathOperator*{\lowinf}{\vphantom{p}inf}    
\DeclareMathOperator{\co}{co}         
\DeclareMathOperator*{\convhull}{\co}  

\newcommand{\half}{\frac{1}{2}}
\newcommand{\thalf}{\tfrac{1}{2}}
\newcommand{\der}{\mathrm{d}}

\newcommand{\reals}{\mathbb{R}}
\newcommand{\nonnegreals}{\ensuremath{[0,\infty)}}
\newcommand{\cX}{\mathcal{X}}
\newcommand{\cZ}{\mathcal{Z}}
\newcommand{\cS}{\mathcal{S}}

\newcommand{\cP}{\mathcal{P}}
\newcommand{\closedcP}{\bar{\cP}}

\newcommand{\cY}{\mathcal{Y}}
\newcommand{\cA}{\mathcal{A}}

\newcommand{\cF}{\mathcal{F}}
\newcommand{\model}{\mathcal{F}}
\newcommand{\normaldist}{\mathcal{N}}   

\newcommand{\dpfour}{\ensuremath{(\ell,\cP,\cF,\cF_\D)}}
\newcommand{\dpfoursame}{\ensuremath{(\ell,\cP,\cF,\cF)}} 
\newcommand{\dpthree}{\ensuremath{(\ell,\cP,\cF)}}

\newcommand{\gBernstein}{u}
\newcommand{\gComparator}{v}

\newcommand{\ie}{\emph{i.e.}}

\newcommand{\cG}{\mathcal{G}}

\newcommand{\actions}{\mathcal{A}}

\newcommand{\loss}{\ell}
\newcommand{\logloss}{\loss^\textnormal{log}}
\newcommand{\sqloss}{\loss^\textnormal{sq}}
\newcommand{\brierloss}{\loss^\textnormal{Brier}}
\newcommand{\regloss}{\loss^\textnormal{reg}}
\newcommand{\classloss}{\loss^\textnormal{class}}
\newcommand{\zoloss}{\loss^\textnormal{01}}
\newcommand{\absloss}{\loss^\textnormal{abs}}

\newcommand{\commentout}[1]{}

\newcommand{\comp}{\phi}
\newcommand{\pred}{\psi}
\let\Exp\E

\newcommand{\pseudoset}{\cP_\model(\eta)}  

\newcommand{\D}{\textnormal{d}}  
\newcommand{\decisionset}{\cF_\D}
\newcommand{\discdecisionset}{\ddot{\cF}_\D}   
\newcommand{\discf}{\ddot{f}}   
\newcommand{\ind}[1]{\mathop{\llbracket #1 \rrbracket}} 
\newcommand{\union}{\cup}

\newcommand{\Prob}{\operatorname{\mathit{P}}}

\newcommand{\Probn}{\operatorname{\mathit{P_n}}}

\DeclareMathOperator*{\interior}{\mathrm{int}}

\DeclareMathOperator*{\bmin}{\wedge}  
\DeclareMathOperator*{\opwedge}{\wedge}

\newcommand{\nat}{\mathbb{N}}

\newcommand{\F}{\mathcal{F}}
\newcommand{\G}{\mathcal{G}}

\newcommand{\lossatof}[2]{\loss \bigl( #1, #2 \bigr)}

\newcommand{\lossof}[1]{\loss(\cdot, #1)}
\newcommand{\lossclass}{\loss \circ \F}

\renewcommand{\Pr}{\mathbf{Pr}}

\newcommand{\erm}{\hat{f}_{\mathbf{Z}}}

\newcommand{\N}{\mathcal{N}}

\newcommand{\capacity}{\mathcal{C}}
\newcommand{\bound}{V}

\newcommand{\xslossat}[1]{(\loss_f - \loss_{f^*})(#1)}

\definecolor{dark-green}{rgb}{0,0.6,0}

\jmlrheading{1}{2000}{1-48}{4/00}{10/00}{T. van Erven,
  Peter D. Gr\"unwald, Nishant Mehta, Mark D. Reid, Robert C. Williamson}

\ShortHeadings{Fast Rates in Statistical and Online Learning}{Van
  Erven, Gr\"unwald, Mehta, Reid and Williamson}
\firstpageno{1}

\begin{document}

\title{Fast Rates in Statistical and Online Learning}

\author{\name Tim van Erven\thanks{Authors listed
    alphabetically. Preliminary versions of some parts of this work
    were presented at NIPS 2012 and at NIPS 2014 (see acknowledgments on page~\pageref{sec:acks}).}
	\email tim@timvanerven.nl \\
       \addr Mathematisch Instituut,
       Universiteit Leiden \\
       Leiden, 2300 RA, The Netherlands
       \AND
	\name Peter D. Gr\"unwald
        \email Peter.Grunwald@cwi.nl 
\\
	       \addr Centrum voor Wiskunde en Informatica and
               MI, Universiteit Leiden\\
	       Amsterdam, NL-1090 GB, The Netherlands
       \AND
       Nishant A. Mehta\thanks{Work performed while at ANU and NICTA.} \email mehta@cwi.nl\\
	       \addr Centrum voor Wiskunde en Informatica \\
	       Amsterdam, NL-1090 GB, The Netherlands
	\AND
	Mark D. Reid \email Mark.Reid@anu.edu.au\\
	\addr Australian National University and NICTA\\ Canberra,
	ACT 2601, Australia
	\AND
	Robert C. Williamson \email Bob.Williamson@anu.edu.au\\
	\addr Australian National University and NICTA\\
	Canberra, ACT 2601 Australia.
}
\editor{Vladimir N. Vapnik, Alexander J. Gammerman and Vladimir G. Vovk}

\maketitle

\begin{abstract}
  The speed with which a learning algorithm converges as it is
  presented with more data is a central problem in machine learning
  --- a fast rate of convergence means less data is needed for the same
  level of performance.  The pursuit of fast rates in online and
  statistical learning has led to the discovery of many conditions in
  learning theory under which fast learning is possible.  We show that
  most of these conditions are special cases of a single, unifying
  condition, that comes in two forms: the {\em central condition\/}
  for `proper' learning algorithms that always output a hypothesis in
  the given model, and {\em stochastic mixability\/} for online
  algorithms that may make predictions outside of the model. We show
  that under surprisingly weak assumptions both conditions are, in a
  certain sense, equivalent.  The central condition has a
  re-interpretation in terms of convexity of a set of
  pseudoprobabilities, linking it to density estimation under
  misspecification. For bounded losses, we show how the central
  condition enables a direct proof of fast rates and we prove its
  equivalence to the {\em Bernstein\/} condition, itself a
  generalization of the {\em Tsybakov margin condition}, both of which
  have played a central role in obtaining fast rates in statistical
  learning. Yet, while the Bernstein condition is two-sided, the
  central condition is one-sided, making it more suitable to deal with
  unbounded losses. In its stochastic mixability form, our condition
  generalizes both a {\em stochastic exp-concavity\/} condition
  identified by Juditsky, Rigollet and Tsybakov and Vovk's notion of
  {\em mixability}. Our unifying conditions thus provide a substantial
  step towards a characterization of fast rates in statistical
  learning, similar to how classical mixability characterizes constant
  regret in the sequential prediction with expert advice setting.
\end{abstract}

\begin{keywords}
statistical learning theory, fast rates, Tsybakov margin condition, mixability, exp-concavity
\end{keywords}

\section{Introduction}
\label{sec:intro}
Alexey Chervonenkis jointly achieved several significant milestones in
the theory of machine learning: the characterization of uniform
convergence of relative frequencies of events to their probabilities
\citep{Vapnik:1971aa}, the uniform convergence of means to their
expectations \citep{Vapnik:1981aa}, and the `key theorem in learning
theory' showing the relationship between the consistency of empirical
risk minimization (ERM) and the uniform one-sided convergence of means to
expectations \citep{Vapnik:1991aa}; \citep[Chapter
3]{Vapnik:1998aa}. Two outstanding features of these contributions are
that they \emph{characterized} the phenomenon in question, and the
quantitative results are \emph{parametrization independent} in the
sense that they do not depend upon how elements of the hypothesis
class $\cF$ are parameterized, only on global (effectively geometric)
properties of $\cF$. With his co-author Vladimir Vapnik, Alexey
Chervonenkis also presented quantitative bounds on the deviation
between the empirical and expected risk as a function of the sample
size $n$. These are used for the theoretical analysis of the
statistical convergence of ERM algorithms, which are central to machine 
learning.  
According to \citet[p. 695]{Vapnik:1998aa}, in his 1974 book co-authored 
by Chervonenkis
\citep{Vapnik:1974aa} they presented `slow' and `fast' bounds for
ERM when used with 0-1 loss.  They showed that in the realizable or 
`optimistic' case (where there is an $f \in \cF$ that almost surely predicts 
correctly, so that the minimum achievable risk is zero) one can achieve 
fast $O(1/n)$ convergence as opposed to the `pessimistic' case where one 
does not have such an
$f$ in the hypothesis class and the best \emph{uniform} bound is
$O(1/\sqrt{n})$ \citep[page 127]{Vapnik:1998aa}. 
This difference is important because if one is in such a `fast rate' regime, one can achieve good performance with less data.

The present paper makes several further contributions along this path
first delineated by Vapnik and Chervonenkis. We focus upon the
distinction between slow and fast learning. As shown in the special
case of squared loss by \citet{lee1998importance} and log loss by
\citet{li1999estimation}, if the hypothesis class is \emph{convex},
one can still attain fast $O(1/n)$ convergence even in the agnostic
(pessimistic) setting.\footnote{Throughout this work, implicit in our statements about rates is that the function class is not too large; we assume classes with at most logarithmic universal metric entropy, which includes finite classes, VC classes, and VC-type classes.}
Such convergence results, like those of Vapnik and Chervonenkis, are uniform --- they hold for all possible target distributions. 
When the hypothesis class is not convex, one cannot attain a uniform
fast bound for ERM \citep{mendelson2008lower}, and it is not known
whether fast rates are possible for any algorithm at all; however, one can obtain a \emph{non-uniform} bound \citep{mendelson2002agnostic,mendelson2008obtaining}.
Such bounds are necessarily dependent upon the relationships between 
the components $(\ell,\cP,\cF)$ of a statistical decision problem or learning
task. Here $\ell$ is the loss, $\cF$ the hypothesis
class, and $\cP$ the (possibly singleton) class of distributions which, 
by assumption, contains the unknown data-generating distribution.
Often one can assume large classes of $\cP$ and still obtain bounds
that are {\em relatively\/}
uniform, i.e. uniform over all $P \in \cP$.  We identify a {\em central
  condition\/} on decision problems $(\ell,\cP,\cF)$ --- where $\ell$ may be 
unbounded --- that, in its strongest form, allows $O(1/n)$ rates for so-called 
`proper' learning algorithms that always output a member of $\cF$. 
In weaker forms, it allows rates in between $O(1/\sqrt{n})$ and $O(1/n)$.

As a second contribution, we connect the above line of work (within
the traditional stochastic setting) to a parallel development in the
worst-case online sequence prediction setting. There, one makes no
probabilistic assumptions at all, and one measures convergence of the
regret, that is, the difference between the cumulative loss attained by a given
algorithm on a particular sequence with the best possible loss
attainable on that sequence \citep{CesaBianchiLugosi2006}. This work,
due in large part to
\cite{vovk1990aggregating,vovk1998game,vovk2001competitive}, shares
one aspect of Vapnik and Chervonenkis' approach --- it achieves a
\emph{characterization} of when fast learning is possible in the
online individual sequence-setting. Since there is no $\cP$ in this
setting, the characterization depends only upon the loss $\ell$, and
in particular whether the loss is \emph{mixable}. As shown in
Section~\ref{sec:four-conditions}, our second key condition, {\em
  stochastic mixability}, is  a generalization of Vovk's earlier
notion. Briefly, when $\cP$ is the
set of all distributions on a domain, stochastic mixability is equivalent to
Vovk's classical mixability.  Stochastic mixability of
$(\ell,\cP,\cF)$ for general $\cP$ then indicates that fast rates are possible 
in a stochastic on-line setting, in the worst-case over all $P \in \cP$.

The main contribution in this paper is to show, first, that a range of
existing conditions for fast rates (such as the Bernstein condition,
itself a generalization of the Tsybakov condition) are either special
cases of our central condition, or special cases of stochastic
mixability (such as original mixability and (stochastic)
exp-concavity); and second, to show that under surprisingly weak
conditions the central condition and stochastic mixability are in fact
equivalent --- thus there emerges essentially a {\em single\/}
condition that implies fast rates in a wide variety of situations.
Our central and stochastic mixability condition improve in several
ways on the existing conditions that they generalize and unify. For
example, like the uniform convergence condition in Vapnik and
Chervonenkis' original `key theorem of learning theory'
\citep{Vapnik:1991aa}, but unlike the Bernstein fast rate condition,
our conditions are {\em one-sided} which, as forcefully argued by
\cite{mendelson2014learning}, seems as it should be;
Example~\ref{ex:comparator-vs-bernstein} explains and illustrates the
difference between the two- and one-sided conditions.  Like Vapnik and
Chervonenkis' uniform convergence condition and Vovk's classical
mixability, but unlike the stochastic and individual-sequence
exp-concavity conditions, our conditions are \emph{parametrization
  independent} (Section~\ref{sec:expconcave}). Finally, unlike the
assumptions for classical mixability \citep{vovk1998game}, we do not
require compactness of the loss function's domain. We hasten to add
though that for unbounded losses, several important issues are still
unresolved --- for example, if under some $P \in \cP$ and with some $f
\in \cF$ the distribution of the loss has polynomial tails, then some
of our equivalences break down (Section~\ref{sec:comparator-vs-bernstein}).

One final historical precursor deserves mention.  Statistical
convergence bounds rely on bounds on the tails of certain random
variables. In Section~\ref{sec:fast-rates} we show how, for bounded
losses, the central condition (\ref{eqn:basiccomparatorconditionpre})
directly controls the behaviour of the cumulant generating function of
the excess loss random variable. The geometric insight behind this
result, Figure \ref{fig:cgf}, previously was used, unbeknownst to us when
carrying out the work originally \citep{mehta2014stochastic}, by
Claude Shannon (\citeyear{Shannon:1956aa}). It is fitting that our
tribute to Alexey Chervonenkis can trace its history to another such
giant of the theory of information processing.

\subsection{Why Read This Paper? Our Most Important Results}
Below, we highlight the core contributions of this work. A more
comprehensive overview is in Section~\ref{sec:setup}  and the diagram on
page~\pageref{fig:map-of-paper}, which summarizes all results from the
paper. 
\begin{itemize}
\item We introduce the {\em $v$-stochastic mixability\/} condition on decision
  problems (Equation (\ref{eq:stochmixpre}),
  Definition~\ref{def:stochastic-mixability}
  and~\ref{def:g-stoch-mixb}), a strict generalization of Vovk's {\em
  classical mixability\/}
\citep{vovk1990aggregating,vovk1998game,vovk2001competitive,vanerven2012mixability}, {\em exp-concavity\/}
\citep{Kivinen:1999aa,CesaBianchiLugosi2006}
and {\em stochastic exp-concavity}, a condition identified implicitly
by \citet{juditsky2008learning} and used by e.g. \cite{DalalyanT12}.
Here $v: \reals^+_0 \rightarrow \reals^+_0$ is a nondecreasing nonnegative
function. In the important special case that $v \equiv \eta$ is
constant, we say that {\em (strong) stochastic mixability
  holds}. Proposition~\ref{prop:stoch-mix-AA} shows that in that case,
with finite $\cF$, Vovk's aggregating
algorithm for on-line prediction in combination with an
online-to-batch conversion achieves a learning rate of $O(1/n)$; if the
$v$-condition holds for sublinear $v$ with $v(0) = 0$, intermediate rates
between $O(1/\sqrt{n})$ and $O(1/n)$ are obtained. These results hold
under no further conditions at all, in particular for unbounded
losses. {\em Interest:\/} the condition being a strict generalization
of earlier ones, it shows that we can get fast rates for some
situations for which this was was hitherto unknown.
\item We introduce the {\em $v$-central condition\/} (Equations
  (\ref{eqn:basiccomparatorconditionpre}),
  (\ref{eqn:basiccomparatorconditionb}), (\ref{eqn:affinity}),
  (\ref{eqn:pre-g-condition}), Definition~\ref{def:comparator}
  and~\ref{def:g-stoch-mix}). As we show in
  Theorem~\ref{thm:BernsteinComparator}, for bounded losses and $v$ of
  the form $v(x) = C x^{\alpha}$, it generalizes the {\em Bernstein
    condition\/} \citep{bartlett2006empirical}, itself a
  generalization of the {\em Tsybakov margin condition\/}
  \citep{tsybakov2004optimal}. If $v
  \equiv \eta$ is constant, we just say that the (strong) {\em central
    condition holds}. In that case, with (unbounded) log-loss, it
  generalizes a (typically nameless) condition used to obtain fast
  rates in Bayesian and \emph{minimum description length} (MDL)
  density estimation in misspecification contexts
  \citep{li1999estimation,zhang2006,zhang2006information,kleijn2006misspecification,grunwald2011safe,GrunwaldVanOmmen14}.
  These are all conditions that allow for fast rates for {\em
    proper\/} learning, in which the learning algorithm always outputs
  an element of $\cF$.  

  \subitem(i) For convex $\cF$, we prove 
  that the strong $\eta$-central condition  and the strong $\eta$-stochastic mixability
 are equivalent, under weak conditions (Theorem~\ref{thm:secondmain} in conjunction with
  Proposition~\ref{prop:frompredictortoconvexface} and
  Theorem~\ref{thm:convexfacetocomparator} in conjunction with
  Proposition~\ref{prop:fromsmtoppcc}).
  {\em Interest: \/} This shows that existing fast rate conditions for
  $O(1/n)$ rates in online learning  are related to fast rate conditions for $O(1/n)$
  rates for proper learning algorithms such as ERM --- even though
  such conditions superficially look very different and have very
  different interpretations: existence of a `substitution function'
  (mixability) vs. the exponential moment of a loss difference
  constituting a supermartingale (central condition).

\subitem(ii) We prove (a) that for bounded losses, the strong
central condition always implies fast $O(1/n)$ rates for ERM and the
$v$-central condition implies intermediate rates
(\cref{thm:finite-fast-rates}). The equivalence between
$\eta$-mixability and the central condition and Proposition~\ref{prop:stoch-mix-AA} mentioned above imply
that, (b), the central condition implies fast rates in many more conditions,
even with unbounded losses. We also show (c) that there exist decision
problems with unbounded losses in which the central condition holds, the Bernstein condition
does not hold, and we do get fast rates. {\em Interest:\/} first, while fast and
  intermediate
  rates under the $v$-central condition with bounded loss can also be
  derived from existing results, our proof is directly in terms of the
  central condition and yields better
  constants.  Second, results (a)-(c) above lead us to {\em
    conjecture} that there exist some very weak condition (much weaker
  than bounded loss) such that for sublinear $v$, the $v$-central condition together with
  this extra condition {\em always\/} implies sublinear
  rates. Establishing such a result is a major goal for future work. 

\item Under mild conditions, the $v$-central condition is equivalent
  to a third condition, the {\em pseudoprobability convexity
    (PPC) condition\/}  --- (\ref{eqn:basicconvexfacecondition}) and
  Definition~\ref{def:convex-face} and~\ref{def:g-stoch-mix}. {\em
    Interest:\/} for the constant $v \equiv \eta$ case ($O(1/n)$
  rates), the PPC condition provides a clear {\em geometric\/} and a
  {\em data-compression\/}
  interpretation of the $v$-central condition. For bounded losses and
  general $v$, it 
implies that a problem must have unique minimizers in a certain
sense (Proposition \ref{prop:nonunique}), giving further insight into
the fast rates phenomenon. 
\item In some cases with nonconvex $\cF$, ERM and other proper
  learning algorithms achieve a suboptimal $O(1/\sqrt{n})$ rate,
  whereas online methods combined with an online-to-batch convergence
  get $O(1/n)$ rates in expectation \citep{audibert2007progressive}. Now the
  implication `strong stochastic mixability $\Rightarrow $ strong
  central condition ' (Theorem~\ref{thm:convexfacetocomparator} in
  conjunction with Proposition~\ref{prop:fromsmtoppcc}, already
  mentioned under 2(i)) holds {\em whenever the risk minimizer within
    $\cF$ coincides with the risk minimizer within the convex hull of
    $\cF$}. Thus, as long as this is the case, there is no inherent
  rate advantage in improper learning --- if $\eta$-stochastic
  mixability holds so that (improper) online methods achieve an
  $O(1/n)$-rate, so will the (proper) ERM method.
  \cref{thm:finite-fast-rates} implies this for bounded losses; we
  conjecture that the same holds for unbounded losses. {\em
    Interest:\/} This insight helps understand when improper learning
  can and cannot be helpful for general losses, something that was
  hitherto only well-understood for the squared loss on a bounded
  domain \citep{lecue2011interplay}.
\end{itemize}

\newcommand{\assA}{\hyperref[ass:simple]{\scriptsize {A}}}
\newcommand{\assB}{\hyperref[ass:TODO]{\scriptsize {B}}}
\newcommand{\assC}{\hyperref[ass:minimax]{\scriptsize C}}
\newcommand{\assD}{\hyperref[ass:convexityContinuity]{\scriptsize D}}
\newcommand{\CC}{\hyperref[def:comparator]{\large\bf CC}}
\newcommand{\PC}{\hyperref[def:martingale-mixability]{\small PC}}
\newcommand{\PPCC}{\hyperref[def:convex-face]{\large\bf PPC}}
\newcommand{\vPPCC}{\hyperref[def:g-stoch-mix]{\small $v$-PPC}}
\newcommand{\MC}{\hyperref[sec:overview]{\small MC}}
\newcommand{\vCC}{\hyperref[def:g-stoch-mix]{\small $v$-CC}}
\newcommand{\uBC}{\hyperref[def:gen-bernstein]{\small $u$-BC}}
\newcommand{\JRT}{\hyperref[def:jrt-cond]{\small JRT}}
\newcommand{\SEC}{\hyperref[sec:expconcave]{\small SEC}}
\newcommand{\EC}{\hyperref[sec:expconcave]{\small EC}}
\newcommand{\SC}{\hyperref[sec:expconcave]{\small SC}}
\newcommand{\VM}{\hyperref[sec:classical-mixability]{\small VM}}
\newcommand{\AC}{\hyperref[sec:audibert]{\small AC}}
\newcommand{\CGC}{\hyperref[sec:fast-rates]{\small CGC}}
\newcommand{\FR}{\hyperref[sec:fast-rates]{\large\bf FR}}
\newcommand{\UM}{\hyperref[sec:nonunique]{\small UM}}
\newcommand{\SM}{\hyperref[def:stochastic-mixability]{\large\bf SM}}
\newcommand{\wSM}{\hyperref[def:stochastic-mixability]{\small wSM}}
\newcommand{\BMDL}{\hyperref[ex:log-loss-continued]{\small BMDL}}
\newcommand{\TM}{\hyperref[sec:intro]{\small TM}}
\newcommand{\VC}{\hyperref[sec:intro]{\small VC}}
\newcommand{\CON}{\hyperref[sec:intro]{\small CON}}
\begin{figure}[htp]
	\centering
	\vspace*{-0.7cm}\hspace*{-1.7cm}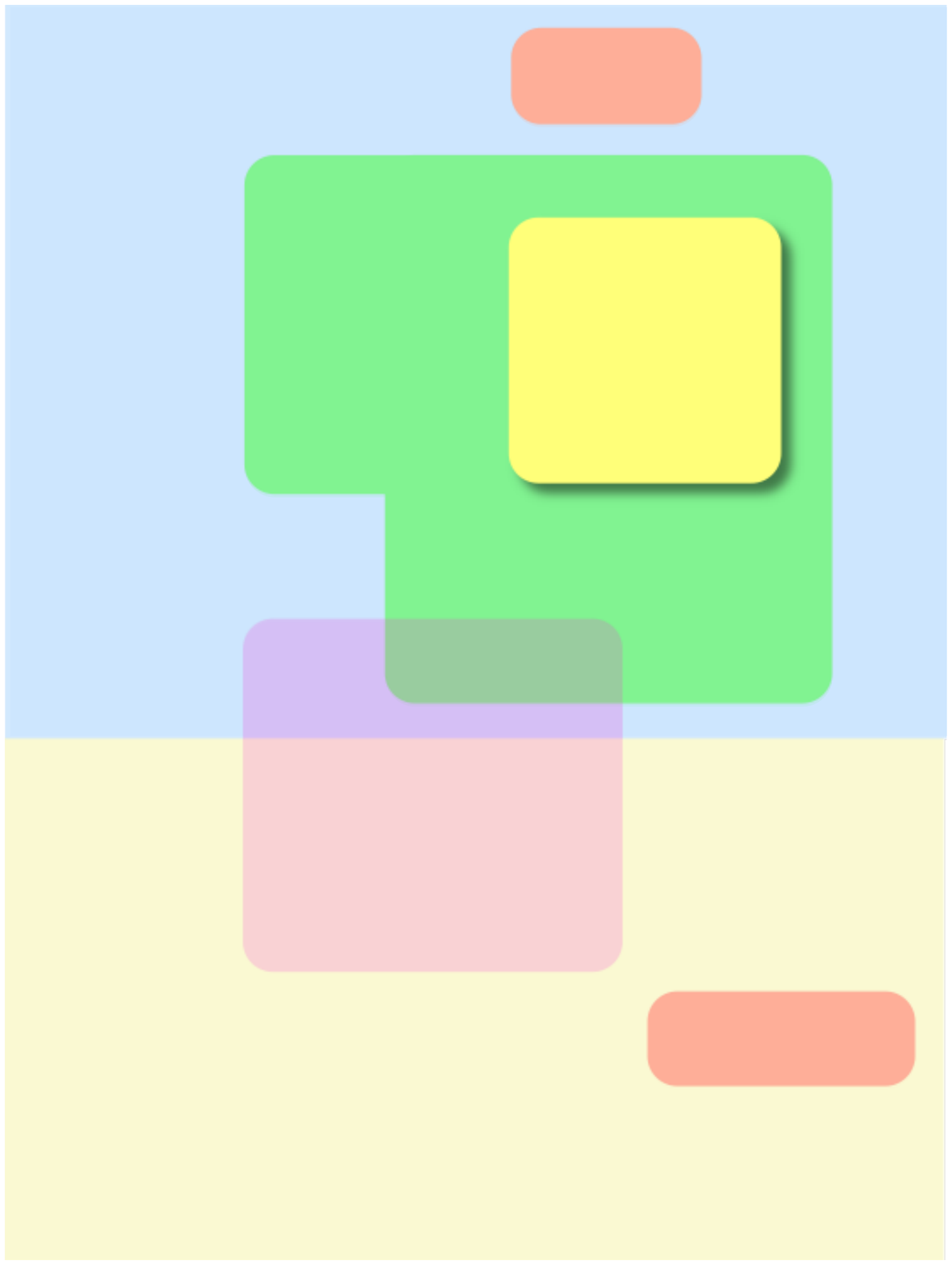
\caption{
\label{fig:map-of-paper}}
\end{figure}
\section{Introduction to and Overview of Results}
\label{sec:setup}
To facilitate reading of this long paper, we provide
an introductory summary of all our results. By reading this section
alongside the `map' of conditions and their relationships on page~\pageref{fig:map-of-paper}, the
reader should get a good overview of our results. We start below
with some notational and conceptual preliminaries, and continue in
Section~\ref{sec:overview} with a discussion of the central condition,
followed by a section-by-section description of the paper.

\subsection{Decision problems and Risk}\label{sec:dp}
We consider decision problems which, in their most general form, can be
specified as a four-tuple $\dpfour$ where $\cP$ is a set of
distributions on a sample space $\cZ$, and the goal is to make
decisions that are essentially as good as the best decision in the
\emph{model} $\cF$ ($\cF$ is often called an `hypothesis space' in
machine learning).  We will allow the decision maker to make decisions
in a \emph{decision set} $\decisionset$ which is usually taken equal to,
or a superset of, $\cF$ but for mathematical convenience is also
allowed to be a subset of $\cF$.  The quality of decisions will be
measured by a \emph{loss} function $\ell \colon \cF_{\ell}
\times \cZ \rightarrow [-B,\infty]$ for arbitrary $B \geq 0 $ where a smaller 
loss means better predictions, and $\cF_{\ell} \supseteq \cF \cup \decisionset$ is  the {\em domain\/} of the loss.  As further notation we introduce the component
functions $\ell_f(z) = \ell(f,z)$ and for any set $\cG$ we let
$\Delta(\cG)$ denote the set of distributions on $\cG$ (implicitly
assuming that $\cG$ is a measurable set, equipped with an appropriate
$\sigma$-algebra).
A loss function $\ell$ is called {\em bounded\/} if for some $B \geq 0$, for all $f \in \cF_{\ell}$ and all $P \in \cP$, we have $|\ell_f(Z)| \leq B$ almost surely when $Z \sim P$.
When $\cF_{\ell}$ is a set for which this is well-defined, for any $\cF \subset \cF_\ell$ we denote by $\convhull(\cF) \subseteq \cF_{\ell}$ the convex hull of $\cF$.

Now fix some decision problem $\dpfour$. The {\em risk\/} of a predictor
$f \in \cF_{\ell}$ with respect to $P \in \cP$ is defined, as usual, 
as
\begin{equation}
\label{eq:risk}
R(P,f) = \Exp_{Z \sim P}[\loss_f(Z)] , 
\end{equation}
where $Z$ is a random variable mapping to outcomes in $\cZ$ and, in
general, $R(P,f)$ may be infinite.  However, for the remainder of the paper
we will only consider tuples $\dpfour$ such that for all $P\in\cP$, there
exists\footnote{We allow the loss itself to be infinite which
  makes random variables and their expectations undefined when they
  evaluate to $\infty - \infty$ with positive probability. The
  requirement that $f^{\circ}$ exists for all $P$ ensures that we
  never encounter this situation in any of our formulas.  
} at least one $f^\circ \in \cF$ with $R(P,f^{\circ}) < \infty$ and
hence $P(\ell_{f^\circ}(Z) = \infty) = 0$. A {\em learning
  algorithm\/} or {\em estimator\/} is a (computable) function from
$\cup_{n \geq 0} \cZ^n$ to $\decisionset$ that, upon observing data
$Z_1, \ldots, Z_n$, outputs some $\hat{f}_n \in \decisionset$. 
Following standard terminology, we call a learning algorithm {\em proper\/}
\citep{lee1996efficient,alekhnovich2004learnability,UrnerBenDavid2014}
if its outputs are restricted to the set $\cF$, i.e. $\cF = \decisionset$.
Examples of this setting, which has also been called {\em in-model estimation\/}
\citep{GrunwaldVanOmmen14}, include ERM and Bayesian \emph{maximum a
  posteriori} (MAP) density estimation. For notational convenience, in
such cases we identify  a decision problem with the triple
$\dpthree$. We only consider $\cF \neq \decisionset$ in
Section~\ref{sec:four-conditions} and~\ref{sec:xyz} on on-line
learning, where $\decisionset$ is often taken to be $\convhull(\cF)$;
for example, $\cF$ may be a set of probability densities (Example~\ref{ex:log-loss}) and the
algorithm may be Bayesian prediction, which predicts with the Bayes
predictive distribution (Section~\ref{sec:convexityinterpretation}), a mixture of elements of $\cF$ which is hence
in $\convhull(\cF)$.  One of our main insights, discussed in
Section~\ref{sec:smppc},  is understanding when the \emph{weaker}
conditions that allow fast rates for improper learning transfer to
the proper learning setting.  In the stochastic setting, the {\em
  rate\/} (in expectation) of a learning algorithm is the quantity
\begin{equation}\label{eq:riskrate}
\sup_{P \in \cP} \ \ \left\{\ \Exp_{\mathbf{Z} \sim P} \left[R(P,\hat{f}_n)\right] -
\inf_{f \in \cF} R(P,f) \ \right\},
\end{equation}
where $\mathbf{Z} = (Z_1, \ldots,Z_n)$ are $n$ i.i.d.\ copies of $Z$.  The
rate of a learning algorithm can usually be bounded, up to $\log n$
factors, as $(\text{\sc comp}_n(\cF)/n)^{\alpha}$ for some $\alpha$
  between $1/2$ and $1$. Here $\text{\sc comp}_n(\cF)$ is some measure
  of the complexity of $\cF$ which may or may not depend on $n$, such
  as its codelength, its VC-dimension in classification, an upper bound on the KL-divergence between prior and posterior in PAC-Bayesian approaches, or the
  logarithm of the number of elements of an $\epsilon$-net, with
  $\epsilon$ determined by sample size, and so on. In the simplest
  case, with $\cF$ finite, complexity is invariably bounded
  independently of $n$ (usually as $\log | \cF|$), and whenever for a
  decision problem $\dpfour$ with finite $\cF$ there exists a learning algorithm achieving the rate $O(1/n)$,
  we say that the problem {\em allows for fast rates}.

In the remainder of this section we make the following simplifying assumption.
\begin{assumption}\label{ass:simple}{\bf (Minimal Risk Achieved)}
  For all $P \in \cP$, the minimal risk $R(P,f)$ over $\cF$ is
  achieved by some $f^* \in \cF$ depending on $P$, i.e.
\begin{equation}\label{eq:bayesact}
  R(P,f^*) = \inf_{f \in \model}
  R(P,f).\end{equation} 
\end{assumption}
Assumption \ref{ass:simple} is essentially a closure property that
holds in many cases of interest. We will call such $f^*$ {\em $\cF$-optimal for $P$\/} or simply {\em $\cF$-optimal}. When $P \in \cP$ and $\cF$ are clear from context, we will also simply say that $f^*$ is the {\em best predictor}. 
\begin{myexample}[Regression, Classification, (Relatively) Well-Specified and Misspecified Models]\label{ex:introductoryexample} In the standard
  statistical learning problems of {\em classification\/} and {\em
    regression}, we have $\cZ = \cX \times \cY$ for some `feature' or
  `covariate' space $\cX$ and $\cF$ is a set of functions from $\cX$
  to $\cY$. In classification, $\cY = \{0,1\}$ and one usually takes
  the standard classification loss $\classloss_f((x,y)) = |y- f(x)|$;
  in regression, one takes $\cY = \reals$ and the squared error loss
  $\regloss_f((x,y)) = \half (y - f(x))^2$. In
  Example~\ref{ex:log-loss} we show that density estimation also fits
  in our setting. For losses with bounded range $[0,B]$, if the
  optimal $f^*$ that exists by Assumption~\ref{ass:simple} has 0
  risk, we are in what \cite{Vapnik:1974aa} call the `optimistic'
  setting, more commonly known as the `deterministic' or `realizable' case (VC in Figure \ref{fig:map-of-paper} on page~\pageref{fig:map-of-paper}). We never make this
  strong an assumption and are thus always in the
  `agnostic' case. A strictly weaker assumption would be to assume
  that $f^*$ is the Bayes decision rule,  minimizing the risk $R(P,f^*)$ over the loss function's
  full domain $\cF_\ell$; in classification this means that $f^*$ is
  the {\em Bayes classifier\/} (minimizing risk over all functions
  from $\cX$ to $\cY$), in regression it implies that $f^*$ is the {\em
    true regression function}, i.e.\ $f^*(x) = \E_{(X,Y) \sim P} [Y \mid X=x]$, in
  density estimation (see below) that $f^*$ is the density of the
  `true' $P$. Borrowing terminology from statistics, we then say that the
  model $\cF$ is {\em well-specified}, or simply {\em
    correct}. Although this assumption is often made in statistics and
  sometimes in statistical learning (e.g.\ in the original Tsybakov condition
  \citep{tsybakov2004optimal} and in the analysis of strictly convex
  surrogate loss functions for $0/1$-loss 
  \citep{bartlett2006convexity}), all of our results are applicable to
  incorrect, {\em misspecified\/} $\cF$ as well. We will, however, in
  some cases make the much weaker Assumption~\ref{ass:preciseBayes}
  (page~\pageref{ass:preciseBayes}) that $\cF$ is well-specified {\em
    relative to $\decisionset$}, or equivalently $\cF$ is {\em as good as\/} $\decisionset$, meaning that for all $P \in \cP$,
  $\min_{f \in \decisionset} R(P,f) = \min_{f \in \cF} R(P,f)$. In all
  our examples, if $\cF \neq \decisionset$ we can take, without loss
  of generality, $\decisionset =
  \convhull(\cF)$, and then a sufficient (but by no means necessary) condition
  for relative well-specification is that $\cF$ is either convex or correct.
\end{myexample}
We now turn to an overview of the main results and concepts of this
paper, which are also highlighted in Figure~\ref{fig:map-of-paper} on page~\pageref{fig:map-of-paper}.
\subsection{Main Concept: The Central Condition}
\label{sec:overview}
We focus on decision problems $\dpthree$ satisfying the
simplifying Assumption~\ref{ass:simple} by fixing any
such decision problem and letting $P \in \cP$ and $f^*$ be $\cF$-optimal 
for $P$.  We may now ask this $f^*$ to
satisfy a stronger, supermartingale-type
property where for some $\eta > 0$ we require
\begin{equation}\label{eqn:basiccomparatorconditionpre}
  \E_{Z \sim P} 
\left[e^{\eta \left( \loss_{f^*}(Z)- \loss_{f}(Z) \right) } \right] \leq 1 
  \qquad \text{for all $f \in \model$}.
\end{equation}
This type of property plays a fundamental role in the
study of fast rates because it controls the higher moments of 
the negated excess loss $\loss_{f^*}(Z)- \loss_{f}(Z)$. 
Note that by our conventions regarding infinities
(Section~\ref{sec:dp}) this implies that $P(\ell_{f^*}(Z) = \infty) =
0$.  

There are several motivations for studying the requirement in 
\eqref{eqn:basiccomparatorconditionpre}.
In the case of classification loss, it can be seen to be a special, extreme 
case of the {\em Bernstein condition\/} (see below). 
In the case of log loss, the requirement becomes a
standard (but usually unnamed) condition which we call the {\em
  Bayes-MDL Condition} which is used in proving convergence rates of Bayesian
and MDL density estimation (Example~\ref{ex:log-loss}).
Finally, under a bounded loss assumption the condition
\eqref{eqn:basiccomparatorconditionpre} implies one our main results,
\cref{thm:finite-fast-rates}, a fast rates result for statistical
learning over finite classes (the situation for unbounded losses is
more complicated and is discussed after Example~\ref{ex:log-loss}).

Note that to satisfy \cref{ass:simple} it is sufficient to require
that the property \eqref{eqn:basiccomparatorconditionpre} holds for
{\em some\/} $f^* \in \cF$ since, by Jensen's inequality, this $f^*$
must then automatically be $\cF$-optimal as in
(\ref{eq:bayesact}). We will require
(\ref{eqn:basiccomparatorconditionpre}) to hold for all $P \in \cP$
(where $f^*$ may depend on $P$). This is the simplest form of our
central condition, which we call the {\em the $\eta$-central
  condition}. We note that if
(\ref{eqn:basiccomparatorconditionpre}) holds for all $f \in \cF$ then it
must also hold in expectation for all {\em distributions\/} on $\cF$.
Thus, the $\eta$-central condition can be restated as follows:
\begin{equation}\label{eqn:basiccomparatorconditionb}
\forall P \in \cP \; \exists f^* \in \cF \;  \forall \Pi \in \Delta(\cF):\;
\E_{Z \sim P} \E_{f \sim \Pi}
    \left[e^{\eta \left( \loss_{f^*}(Z)- \loss_{f}(Z) \right) } \right] \leq 1.
\end{equation}
This rephrasing of the central condition will be useful when comparing it to 
conditions introduced later in the paper.

The central condition is easiest to interpret for density estimation
with the logarithmic loss. In this case the condition for $\eta = 1$ is
implied by $\model$ being either well-specified or convex, as the following
example shows.
\begin{myexample}[Density estimation under well-specified or convex models]
\label{ex:log-loss} 
Let $\model$ be a set of probability densities on $\cZ$ and take $\loss$ to be log loss, so that $\loss_f(z) = -\log f(z)$.

For log loss, statistical learning becomes equivalent to density
estimation. 
Satisfying the central condition then becomes equivalent to, for all $P \in \cP$, 
finding an $f^* \in \cF$ such that
\begin{align}\label{eqn:affinity}
\E_{Z \sim P} \left(\frac{f(Z)}{f^*(Z)}\right)^{\eta} \leq 1
\end{align}
for all $f \in \model$.
If the model $\model$ is correct, it trivially holds that $\dpthree$
satisfies the 1-central condition as we choose $f^*$ to be the density
of $P$, so that the densities in the expectation and the denominator cancel. Even when the model is misspecified, \cite{li1999estimation}
showed that \eqref{eqn:affinity} holds for $\eta = 1$ provided the
model is convex. We will recover this result in
Example~\ref{ex:log-loss-continued} in Section~\ref{sec:convex-face},
where we review the central role that (\ref{eqn:affinity}) plays in
convergence proofs of MDL and Bayesian estimation. Even if the set of densities is neither correct nor convex, the central condition often still holds for some $\eta \neq 1$. In Example~\ref{ex:dunno} we explore this for the set of normal densities with variance $\tau^2$ when the true distribution is either Gaussian with a different variance, or subgaussian.
\end{myexample}
We show in Section~\ref{sec:fast-rates} that for bounded losses the $\eta$-central
condition implies fast $O(1/n)$ rates for finite $\cF$. But what about
unbounded losses such as log loss? In the log loss/density estimation
case, as shown by
\cite{barron1991minimum,zhang2006,grunwald2007minimum} and others, 
fast rates can be obtained in a weaker sense. Specifically, in the worst-case 
over $P \in \cP$, the squared Hellinger distance or R\'enyi divergences between
$\hat{f}_n$ and the optimal $f^*$ converge as $O(1/n)$ for ERM 
when $\cF$ is finite, and like $O(\text{\sc comp}_n/n)$ for general $\cF$
and for 2-part MDL and Bayes MAP-style algorithms.  If the goal is to
obtain fast rates in the stronger sense (\ref{eq:riskrate}) for general
unbounded loss functions some additional assumptions are needed.
\cite{zhang2006,zhang2006information} provides such results for
penalized ERM and randomized estimators (see also the discussion in
Section~\ref{sec:discussion}). Importantly, as explained by
\cite{grunwald2012safe}, the proofs for fast rates in all the works
mentioned here crucially, though sometimes implicitly, employ the
$\eta$-central condition at some point.

\subsection{Overview of the Paper}

\subsubsection*{\rm {\em Section~\ref{sec:convex-face} ---Fast Rates for
  Proper Learning:\/} PPC Condition, Bayesian Interpretation,
  Relation to Bayes-MDL Condition.} In
Section~\ref{sec:convex-face}, we give a second condition, the {\em
  pseudoprobability convexity (PPC) condition}, a variation of
(\ref{eqn:basiccomparatorconditionb}) stating that:
  \begin{equation}\label{eqn:basicconvexfacecondition}
\forall P \in \cP \;  \forall \Pi \in \Delta(\cF) \; \exists f^* \in \cF  :\;
\E_{Z \sim P }[\loss_{f^*}(Z)] \leq \E_{Z \sim P }\left[- \frac{1}{\eta}\log \E_{f \sim \Pi} e^{-\eta \loss_{f}(Z)}\right].
\end{equation}
Clearly, if the condition holds, then it
will hold by choosing, for every $P \in \cP$, $f^*$ to be
$\cF$-optimal relative to $P$.  The name `pseudoprobability'  stems
from the interpretation of $p_f(Z) := e^{-\ell_f(Z)}$ as
`pseudo-probability associated with $f$, similar to the
`entropification' of $f$ introduced by \cite{Grunwald99a}. The full
`pseudoprobability convexity' stems from the  interpretation
illustrated by and explained around Figure~\ref{fig:convexitycartoon}
on page~\pageref{fig:convexitycartoon}. We show that, under
simplifying Assumption~\ref{ass:simple}, the central and PPC
conditions are equivalent. One direction of this equivalence is 
trivial, while the other direction is our first main result,
Theorem~\ref{thm:convexfacetocomparator}. We also explain how the
rightmost expression in (\ref{eqn:basicconvexfacecondition}) strongly
resembles the expected log-loss of a Bayes predictive distribution,
and how this leads to a `pseudo-Bayesian' or
`pseudo-data compression' interpretation of the pseudoprobability
convexity condition, and hence of the central condition. Versions of
this interpretation were highlighted earlier by
\citet{grunwald2012safe,GrunwaldVanOmmen14}. Thus, we can think of both
conditions as a single condition with dual interpretations: a
frequentist one in terms of exponentially small deviation
probabilities (which follow by applying Markov's inequality to $\Exp_{Z \sim P} [ e^{\eta (\ell_{f^*(Z)} - \ell_f(Z))} ]$), and a pseudo-Bayesian one in terms of convexity
properties of $\cF$. Further, we give a few more examples of the central/PPC condition in this section, and we discuss in detail its special case, the  Bayes-MDL
condition (Example~\ref{ex:log-loss}).

Crucially, all algorithms that we are aware of for which fast rates have
been proven by means of the $\eta$-central condition are `proper' in
that they always output a (possibly randomized) element of $\cF$
itself. This includes ERM, two-part MDL, Bayes MAP and randomized
Bayes algorithms 
\citep{barron1991minimum,zhang2006,zhang2006information,grunwald2007minimum}
and PAC-Bayesian methods
\citep{audibert2004pac,catoni2007pac}. Thus, the central condition is
appropriate for {\em proper learning}. This is in contrast to the
stochastic mixability condition which is defined and studied in
Section~\ref{sec:four-conditions}.

\subsubsection*{{\em Section~\ref{sec:four-conditions} --- Fast Rates for Online
  Learning:\/}  {\rm (Stochastic) Mixability and Exp-Concavity}.}
In online learning with bounded losses, \emph{strong convexity} of the loss is an
oft-used condition to obtain fast rates because it is 
naturally related to gradient and mirror descent methods
\citep{HazanAgarwalKale2007,HazanRakhlinBartlett2007,ShalevShwartzSinger2007}.
If we allow more general algorithms, however, then fast rates are also
possible under the condition of \emph{exp-concavity} which is weaker
than strong convexity \citep{HazanAgarwalKale2007}. Exp-concavity in
turn is a special case of Vovk's classical mixability condition
\citep{vovk2001competitive}, the main difference being that the
definition of exp-concavity depends on the choice of parametrization
of the loss function whereas the definition of classical mixability
does not.  Whether classical mixability can really be strictly weaker
than exp-concavity in an `optimal' parametrization is an open
question \citep{WilliamsonZhangParameswaran2015,vanerven2012blog}.
Strong convexity, exp-concavity and classical mixability are all
individual sequence notions, allowing for fast rates in the sense
that, if $\cF$ is finite, then there exist (improper) learning
algorithms for which the worst-case cumulative regret over all
sequences, that is $\sup_{z_1, \ldots, z_n \in \cZ^n} \; \left\{ \sum_{i=1}^n
  \left( \ell_{\hat{f}_{i-1}}(z_i) \right) - \inf_{f \in \cF}
  \sum_{i=1}^n \ell_f(z_i) \; \right\}$, is bounded by a constant. This implies
that the worst-case cumulative regret per outcome at time $n$ is
$O(1/n)$.

One may obtain learning algorithms for statistical learning by converting algorithms
for online learning using a process called \emph{online-to-batch
conversion}
\citep{CesaBianchiConconiGentile2004,Barron1987,YangBarron1999}. This
process preserves rates, in the sense that 
if the worst-case regret per outcome at time $n$ of a method is $r_n$  then the
rate of the resulting learning algorithm in the sense of
(\ref{eq:riskrate}) will also be $r_n$. However, for  
this purpose, it suffices to use a much weaker stochastic analogue of
mixability that only holds in expectation instead of holding for all outcomes. 
This analogue is \emph{$\eta$-stochastic mixability},
which we define (note the similarity to (\ref{eqn:basicconvexfacecondition})) as
\begin{align}\label{eq:stochmixpre}
\forall \Pi \in \Delta(\cF) \; \exists f^* \in \decisionset  \; \forall P \in \cP  :\;
\E_{Z \sim P }[\loss_{f^*}(Z)] \leq \E_{Z \sim P }\left[- \frac{1}{\eta} \log \E_{f \sim \Pi} e^{-\eta \loss_{f}(Z)}\right].
\end{align}
Under this condition, Vovk's Aggregating Algorithm (AA) achieves fast
rates in expectation under any $P \in \cP$ in sequential on-line
prediction, without any further conditions on $\dpfour$; in
particular there are no boundedness restrictions on the loss. If we
take $\cP$ to be the set of all distributions on $\cZ$, we recover
Vovk's original individual-sequence $\eta$-mixability. Note that,
based on data $Z_1, \ldots, Z_n$, the AA outputs $f$ that are not
necessarily in $\cF$ but can be in some different set $\decisionset$
(in all applications we are aware of, $\decisionset =
\convhull(\cF)$, the convex hull of $\cF$).  Online-to-batch
conversion has been used, amongst others, by
\citet{juditsky2008learning,DalalyanT12} and \citet{audibert2009fast}
to obtain fast rates in model selection aggregation. In
Sections~\ref{sec:jrt} and \ref{sec:audibert} we relate their
conditions to stochastic mixability. We show that results by
\citet{juditsky2008learning} employ a \emph{stochastic
  exp-concavity} condition, a special case of our
stochastic mixability condition, in a manner similar to the way  exp-concavity
is a special case of classical mixability. Given these
applications to statistical learning, it is not surprising that
stochastic mixability is closely related to the conditions for
statistical learning discussed above. We will show in
Proposition~\ref{prop:fromsmtoppcc} that under certain assumptions
it is equivalent to our central condition
\eqref{eqn:basiccomparatorconditionb} and hence also the PPC
condition~\eqref{eqn:basicconvexfacecondition}.  The proposition
shows that this holds unconditionally in the proper learning 
setting:  stochastic mixability implies the pseudoprobability
convexity condition which, in turn, implies the central condition
under some weak restrictions.  The proposition also gives a condition
under which these relationships continue to hold in the more
challenging case when $\cF \neq \decisionset$. In general, making
predictions in $\decisionset$ gives more power, and the central
condition can only be used to infer fast rates for proper learning
algorithms which always play in $\cF$. Thus, if $\eta$-stochastic
mixability for $\dpfour$ implies $\eta$-PPC for $\dpthree$ then
there is no rate improvement for learning algorithms that are
allowed to predict in $\decisionset$ instead of
$\cF$. Proposition~\ref{prop:fromsmtoppcc} gives a central insight of
this paper by showing that this implication holds under
Assumption~\ref{ass:preciseBayes}: {\em $\eta$-stochastic mixability
  for $\dpfour$ implies the $\eta$-PPC and $\eta$-central conditions
  for $\dpthree$ whenever $\cF$ is well-specified {\em relative\/} to
  $\decisionset$} --- relative well-specification was defined in
Example~\ref{ex:introductoryexample}, where we indicated that this a
much weaker condition than mere correctness of $\cF$; in all cases we
are aware of, a sufficient condition is that $\cF$ is convex.  In
Example~\ref{ex:audibertrate} we explore the implications of
Proposition~\ref{prop:fromsmtoppcc} for the question whether fast
rates can be obtained both in expectation and in probability --- as is
the case under the central condition --- or only in expectation --- as
is sometimes the case under stochastic mixability.

For the implication from the central condition to stochastic
mixability, we first define an intermediate, slightly stronger generalization of
classical mixability that we call the \emph{$\eta$-predictor
  condition}, which looks like the central condition, but with its
universal quantifiers interchanged:
\begin{align}
\forall \Pi \in \Delta(\cF) \; \exists f^* \in \decisionset  \; \forall P \in \cP  :\;
\E_{Z \sim P} \E_{f \sim \Pi}
\left[e^{\eta \left(\loss_{f^*}(Z)- \loss_{f}(Z) \right) } \right]
\leq 1. 
\end{align}
In our second main result, Theorem~\ref{thm:secondmain}, we show that
the central condition implies the predictor condition whenever the
decision problem satisfies a certain minimax identity, which holds under Assumption~\ref{ass:minimax} or its weakening Assumption~\ref{ass:convexityContinuity}. And since (by a trivial
application of Jensen's inequality) the predictor condition in turn
implies stochastic mixability, we come full circle and see that, under
some restrictions, all four of our conditions in the 
`central quadrangle' of Figure~\ref{fig:map-of-paper} (page~\pageref{fig:map-of-paper}) are really equivalent.

\subsubsection*{{\em Section~\ref{sec:comp-mix-margin} --- Intermediate Rates:\/}\ {\rm  
Weakening to $\gComparator$-central condition, connection to Bernstein and Tsybakov Conditions --- can be read independently from Section~\ref{sec:four-conditions}.}} In
Section~\ref{sec:comp-mix-margin}, we weaken the $\eta$-central condition to a condition which we
call the $\gComparator$-central condition: rather than requiring that a fixed
$\eta$ exists such that (\ref{eqn:basiccomparatorconditionpre}) holds, we
only require that it holds (for all $P \in\cP$) up to some `slack' $\epsilon$, where we
require that the slack must go to $0$ as $\eta \downarrow
0$. Specifically, we require that there is some increasing nonnegative function
$\gComparator$ such that
\begin{equation}\label{eqn:pre-g-condition}
  \E_{Z \sim P} 
\left[e^{\eta \left( \loss_{f^*}(Z)- \loss_{f}(Z) \right) } \right] \leq e^{\eta \epsilon}
  \qquad \text{for all $f \in \model$, all $\epsilon > 0$, with $\eta := \gComparator(\epsilon)$}.
\end{equation}
As shown in this section
(Example~\ref{ex:tsybakov}), the $\gComparator$-central condition is
associated with rates of order $w(C/n)$ where $C > 0$ is some
constant, and $w$ is the inverse of $x \mapsto x \gComparator(x)$ --- taking
constant $\gComparator(x) = \eta$ we see that this generalizes the
situation for the $\eta$-central condition which for fixed $\eta$ allows
rates of order $O(1/n)$. In our third main result,
Theorem~\ref{thm:BernsteinComparator}, we then show that, for bounded
loss functions, this condition is equivalent to a 
\emph{generalized Bernstein condition} (see Definition \ref{def:gen-bernstein}),
which itself is a generalization of the Tsybakov margin condition \citep{tsybakov2004optimal} 
to classification settings in which $\cF$ may be misspecified, and to
loss functions different from $0/1$-loss \citep{bartlett2006empirical}. Specifically, for given
function $\gComparator$, a decision problem satisfies the
$\gComparator$-central condition if and only if it satisfies the
$\gBernstein$-generalized Bernstein condition for a function
\begin{equation}\label{eq:bernie}
\gBernstein(x) \asymp \frac{x}{\gComparator(x)} ,
\end{equation}
where for functions $a, b$ from $[0,\infty)$ to $[0,\infty)$, $a(x)
\asymp b(x)$ denotes that there exist constants $c, C > 0$ such that,
for all $x \geq 0$, $c a(x) \leq b(x) \leq C a(x)$.
\begin{myexample}[Classification]\label{ex:classification}
  Let $\dpthree$ represent a classification problem with $\ell$ the
  $0/1$-loss that satisfies the $\gComparator$-central condition 
for $\gComparator(x)
  \asymp x^{1- \beta}$, $0 \leq \beta \leq 1$. Then (\ref{eq:bernie})
  holds with $\gBernstein$ of form $\gBernstein(x) = B x^{\beta}$.
  This is equivalent to the standard $(\beta,B)$-Bernstein condition
  (which, if $\cF$ is well-specified, corresponds to the Tsybakov margin condition with exponent
  $\beta/(1-\beta)$), 
which is known to guarantee rates of
  $O\left(n^{-1/(2-\beta)}\right)$. This is consistent with the rate
  $w(C/n)$ above, since if $\gComparator(x) \asymp x^{1 - \beta}$, then its inverse $w$ satisfies
$w(x) \asymp x^{1/(2-\beta)}$.
\end{myexample}
For the case of unbounded losses, the generalized Bernstein and
central conditions are not equivalent.
Example~\ref{ex:comparator-vs-bernstein} gives a simple case in which
the Bernstein condition does not hold whereas, due to its
one-sidedness, the central condition does hold and fast rates for ERM
are easy to verify; Example~\ref{ex:lastminute} shows that the
opposite can happen as well.

In this section we also extend $\eta$-stochastic mixability to
$\gComparator$-stochastic-mixability and show that another fast-rate
condition identified by \cite{juditsky2008learning} is a special case.
For unbounded losses, the $v$-stochastic mixability and the $v$-central
condition become quite different, and it may be that the $u$-Bernstein condition does imply $v$-mixability;
whether this is so is an open problem. Finally, using
Theorem~\ref{thm:BernsteinComparator}, we characterize the
relationship between the $\eta$-central condition and the existence of
unique risk minimizers for bounded losses.

\subsubsection*{{\em Section~\ref{sec:xyz} ---From Actions to Predictors.}}\label{page:unconditional}
The classical mixability literature usually considers the {\em
  unconditional\/} setting where observations and actions are points from $\cZ$
  and $\cA$, respectively.
  For example, one may
  consider the squared loss with $\ell_a(y) = (y-a)^2$ for $y,a \in [0,1]$. 
  It is often easy to establish stochastic mixability for a decision problem
in this unconditional setting. An interesting question is whether this
automatically implies that stochastic mixability (and hence, under
further conditions, also the central condition) holds in the
corresponding \emph{conditional} setting where $\cZ = \cX \times \cY$ and
the decision set contains predictors $f : \cX \to \cA$ that map features $x\in\cX$ to
actions.
Here, an example loss function might be $\regloss_f((x,y)) = \half (y - f(x))^2$ as considered
in Example~\ref{ex:introductoryexample}. In this section, we show that
the answer is a qualified `yes' --- in general, the set $\decisionset$
may need to be a large set such as $\cA^\cX$, but with some additional assumptions it remains manageable.

\subsubsection*{{\em Section~\ref{sec:fast-rates} --- Fast Rate Theorem.}}
In \cref{sec:fast-rates}, we show how for bounded losses the central
condition enables a direct proof of fast rates in statistical learning
over finite classes.  The path to our fast rates result,
\cref{thm:finite-fast-rates}, involves showing that, for each function
$f \in \cF$, the central condition implies that the empirical excess
loss of $f$ exhibits one-sided concentration at a scale related to the
excess loss of $f$. This one-sided concentration result is achieved by
way of the Cram\'er-Chernoff method \citep{boucheron2013concentration}
combined with an upper bound on the \emph{cumulant generating function} (CGF) of
the negative excess loss of $f$ evaluated at a specific point.
The upper bound on the CGF is given in
\cref{thm:stochastic-mixability-concentration} which shows that if the
absolute value of the excess loss random variable is bounded by 1, its CGF evaluated at some $-\eta < 0$ takes the
value $0$, and its mean $\mu$ is positive, then the central condition
implies that the CGF evaluated at $-\eta/2$ is upper bounded by a
universal constant times $-\eta \mu$. By way of a careful localization
argument, the fast rates result for finite classes also extends to 
VC-type classes, as presented in \cref{thm:vc-type-fast-rates}.

\subsubsection*{{\em Final Section --- Discussion.}}
The paper ends with a discussion of what has been achieved and a list
of open problems.

\section{The Central Condition in General and a Bayesian Interpretation via the PPC Condition}
\label{sec:convex-face}

In this section we first generalize the definitions of the central and
pseudoprobability convexity (PPC) conditions beyond the case of the
simplifying Assumption~\ref{ass:simple}. We give a few examples and list some 
of their basic properties. We then show that the central
condition trivially implies the PPC condition, under no conditions on
the decision problem at all. Additionally, in our first main theorem, we show
that if Assumption~\ref{ass:simple} holds or the loss is bounded, then
the converse result is also true.  Importantly, this equivalence between
the central condition and the PPC condition
allows us to interpret the PPC condition as the
requirement that a particular set of \emph{pseudoprobabilities} is
convex on the side that `faces' the data-generating distribution $P$
(Figure~\ref{fig:convexitycartoon}). This leads to a
(pseudo)-Bayesian interpretation, which says that the
(pseudo)-Bayesian predictive distribution is not allowed to be better
than the best element of the model.

\subsection{The Central and Pseudoprobability Convexity Conditions in General}
We now extend the definition (\ref{eqn:basiccomparatorconditionpre})
of the central condition to the case that our simplifying
Assumption~\ref{ass:simple} may not hold. In such cases, it may be
that there is no fixed comparator that satisfies
(\ref{eqn:basiccomparatorconditionpre}), but there does exist a
sequence of comparators $f^*_1, f^*_2, \ldots$ that satisfies
(\ref{eqn:basiccomparatorconditionb}) in the limit. By introducing a
function $\comp$ that maps $P$ to $f^*$ this leads to the following
definition of the general $\eta$-central condition:
\begin{definition}[Central Condition]
\label{def:comparator} 
Let $\eta > 0$ and $\epsilon \geq 0$. We say that $\dpthree$ satisfies 
the \emph{$\eta$-central condition up to $\epsilon$} if
there exists a
\emph{comparator selection} function $\comp\colon \cP \rightarrow
\cF$ such that
\begin{align}\label{eqn:comparator}
\E_{Z \sim P} \E_{f \sim \Pi}
\left[e^{\eta \left( \ell_{\comp(P)}(Z)- \ell_{f}(Z) \right) } \right] \leq 
 e^{\eta \epsilon}
 \qquad
\text{for all $P \in \cP$ and distributions $\Pi \in \Delta(\model)$.}
\end{align}
If it satisfies the $\eta$-central condition up to $0$, we say that
the \emph{strong $\eta$-central condition} or simply the
\emph{$\eta$-central condition} holds. If it satisfies the $\eta$-central
condition up to $\epsilon$ for all $\epsilon > 0$, we say that the
\emph{weak $\eta$-central condition} holds; this is equivalent to
\begin{align}\label{eqn:supinfsup-comparator}
\sup_{P \in \cP} \lowinf_{f^* \in \cF}  \sup_{\Pi \in \Delta(\cF)} 
\E_{Z \sim P} \E_{f \sim \Pi}
\left[e^{\eta \left( \loss_{f^*}(Z)- \loss_{f}(Z) \right) } \right] \leq 1.
\end{align}
\end{definition}
Note that we explicitly identify the situation in which the condition does not
actually hold in the strong sense but will if some slack $\epsilon > 0$ is introduced.
We will do
the same for the other fast rate conditions identified in this paper,
and we will also establish relations between the `up to $\epsilon >
0$' versions. This will become useful throughout
Section~\ref{sec:comp-mix-margin} and, in particular, Section~\ref{sec:JRTother}.

The PPC condition generalizes analogously to
the central condition and features
\begin{equation}\label{eqn:mixloss}
  m^\eta_\Pi(z) = -\frac{1}{\eta} \log \E_{f \sim \Pi}\left[e^{-\eta
  \loss_f(z)}\right] ,
\end{equation}
a quantity that plays a crucial role in the analysis of online learning algorithms
\citep{vovk1998game,vovk2001competitive},
\citep[Theorem~2.2]{CesaBianchiLugosi2006} and has been called the
\emph{mix loss} in that context by \Citet{rooij2014follow}.
\begin{definition}[Pseudoprobability convexity condition]
\label{def:convex-face} 
Let $\eta > 0$ and $\epsilon \geq 0$. We say that $\dpthree$ satisfies
the  \emph{$\eta$-pseudoprobability convexity condition up to $\epsilon$} if there
exists a function $\comp\colon \cP \rightarrow \cF$ such that
\begin{align}\label{eqn:strong-convexface}
  \E_{Z \sim P}
  \left[\ell_{\comp(P)}(Z) \right] \leq \E_{Z \sim P}
  \left[m^\eta_\Pi(Z)\right] + \epsilon
  \qquad\text{for all $P \in \cP$ and $\Pi \in \Delta(\cF)$.}
\end{align}
If it satisfies the $\eta$-pseudoprobability convexity condition up to $0$, we say that
the \emph{strong $\eta$-pseudoprobability convexity condition} or simply
the \emph{$\eta$-pseudoprobability convexity condition} holds. If it satisfies the
$\eta$-pseudoprobability convexity condition up to $\epsilon$ for all $\epsilon > 0$, we
say that the \emph{weak $\eta$-pseudoprobability convexity condition} holds; this is
equivalent to
\begin{align}\label{eqn:supsupinf-convexface}
\sup_{\Pi \in \Delta(\cF)}  \ \sup_{P \in \cP} \lowinf_{f \in \cF} 
\ \E_{Z \sim P} \left[ \ell_{f}(Z)- m^{\eta}_{\Pi}(Z) \right] \leq 0.
\end{align}
\end{definition}
Under Assumption~\ref{ass:simple} this condition simplifies and
implies the essential uniqueness of optimal predictors 
(cf. Section~\ref{sec:convexityinterpretation}).
\begin{proposition}{ \bf (PPC condition implies uniqueness of risk minimizers)\
  }\label{prop:unique}
  Suppose that Assumption~\ref{ass:simple} holds, and that $\dpthree$
  satisfies the weak $\eta$-pseudoprobability convexity condition.
  Then it also satisfies the strong $\eta$-pseudoprobability convexity
  condition, and for all $P \in \cP$, the $\cF$-optimal $f^*$
  satisfying (\ref{eq:bayesact}) is essentially unique, in the sense
  that, for any $g^* \in \cF$ with $R(P,g^*) = R(P,f^*)$, we have that
  $\ell_{g^*}(Z) = \ell_{f^*}(Z)$ holds $P$-almost surely.
\end{proposition}
\begin{proof}
  Assumption~\ref{ass:simple} implies that if
  (\ref{eqn:strong-convexface}) holds at all, then it also holds with
  $\comp(P)$ equal to any $\cF$-risk minimizer $f^*$ as in
  (\ref{eq:bayesact}). Thus, if it holds for all $\epsilon >0$, it
  holds for all $\epsilon > 0$ with the fixed choice $f^*$, and hence
  it must also hold for $\epsilon = 0$ with the same $f^*$.

As to the second part, consider a distribution $\Pi$ that puts mass
$1/2$ on $f^*$ and $1/2$ on $g^*$. Then the strong
$\eta$-pseudoprobability condition implies that 
  \begin{align*}
    \min_{f \in \cF} \E_{Z \sim P} [\loss_f(Z)]
      &\leq \E_{Z \sim P} \left[-\frac{1}{\eta} \log\big(\half
          e^{-\eta \loss_{f^*}(Z)} + \half e^{-\eta
          \loss_{g^*}(Z)}\big)\right]\\
      &\leq \E_{Z \sim P} \left[\half
          \loss_{f^*}(Z) + \half
          \loss_{g^*}(Z)\right]
      = \min_{f \in \cF} \E_{Z \sim P} [\loss_f(Z)],
  \end{align*}
  where we used convexity of $- \log$ and Jensen's inequality.  Hence
  both inequalities must hold with equality. By strict convexity of
  $-\log$, we know that for the second inequality this can only be the
  case if $\loss_{f^*} = \loss_{g^*}$ almost surely, which was
  to be shown.
\end{proof}
Finally, we will often make use of the following trivial but important
fact.
\begin{fact}\label{fact:ccppcc} Fix $\eta > 0, \epsilon \geq 0$ and let $\dpthree$ be an arbitrary decision problem that
  satisfies the $\eta$-central condition up to $\epsilon$. Then for
  any $0 < \eta' \leq \eta$ and any $\epsilon' \geq \epsilon$ and for any $\cP' \subseteq \cP$,
  $(\ell,\cP', \cF)$ satisfies the $\eta'$-central condition up to $\epsilon'$. The same
  holds with `central' replaced by `PPC'.
\end{fact}
We proceed to give some examples.
\begin{myexample}[Squared Loss, Unrestricted Domain]\label{ex:unbounded-sq-cc}
  Consider  squared loss $\sqloss_f(z) = \half(z-f)^2$ with $\cZ = \cF =
  \reals$, and let $\cP = \{\normaldist(\mu,1) : \mu
  \in \reals\}$ be the set of normal distributions with unit
  variance and arbitrary means $\mu$. Estimating
  the mean of a normal model is a standard inference problem for which a
  squared error risk of order $O(1/n)$ is obtained by the sample mean.
  We would therefore expect the central condition to be satisfied
  and, indeed, this is the case for $\eta \leq 1$ via a reduction to
  Example~\ref{ex:log-loss}. To see this, consider the well-specified
  setting for the log loss $\logloss_{f'}$ with densities $f' \in \cF' =
  \cP$, and note that the squared loss for $f$ equals the log loss for
  $f'$ up to a constant when $f$ is the mean of $f'$:
  \begin{equation*}
    \sqloss_f(z) = -\log e^{-(z-f)^2/2} = \logloss_{f'}(z) -
    \log\sqrt{2\pi}.
  \end{equation*}
  Since the log loss satisfies the $1$-central condition in the
  well-specified case (see Example~\ref{ex:log-loss}), the squared loss
  must also satisfy the  $1$-central condition.
\end{myexample}
Not surprisingly, the central condition still  holds if we replace the
Gaussian assumption by a subgaussian assumption. 
\begin{myexample}\label{ex:dunno}
For $\sigma^2 > 0$ let  $\cP_{\sigma^2}$ be an arbitrary  subgaussian
collection of distributions over $\reals$. That is, for all $t \in \reals$
and $P \in \cP_{\sigma^2}$
\begin{equation}\label{eq:moment-bound} 
	\E_{Z\sim P}\left[e^{t(Z-\mu_P)}\right]
	\le e^{\sigma^2 t^2/2} ,
\end{equation}
where $\mu_{P} = \E_{Z \sim P} [Z]$ is the mean of $Z$. 
Now consider the squared loss $\sqloss_f(z) = \half(z-f)^2$ again, with
$\model = \cZ = \reals$. Then
\begin{equation}\label{eq:hurry}
  \sqloss_f(z) - \sqloss_{f'}(z) = \frac{1}{2}\delta(2(z-f) - \delta),
  \qquad \text{where $\delta = f' - f$.}
\end{equation}
Taking $f = \mu_P$ gives
\begin{equation}\label{eq:deadline}
  \E_{Z\sim P}\left[e^{\eta\left(\sqloss_f(Z)-\sqloss_{f'}(Z)\right)}\right]
    = e^{-\eta \delta^2/2} \E_{Z \sim P}\left[e^{\eta \delta (Z-\mu_P)}\right]
    \leq e^{-\eta \delta^2/2} e^{\sigma^2 \eta^2 \delta^2/2}.
\end{equation}
The right-hand side is at most $1$ if $\eta \leq 1/\sigma^2$, and
hence to satisfy the strong $\eta$-central condition with
substitution function $\comp(P) = \mu_P$, it suffices to take $\eta
\leq 1/\sigma^2$. Note that $\comp$ maps $P$ to the $\cF$-optimal
predictor for $\cP$ --- a fact which holds generally, as shown in
Proposition~\ref{prop:unique} above. Note also that, just like
Example~\ref{ex:unbounded-sq-cc}, the example can be reduced to the
log-loss setting in which the densities are all normal densities with
means in $\reals$ and variance equal to $1$. In
Example~\ref{ex:lastminute} we shall see that if $\cP$ contains $P$
with polynomially large tails, then the $\eta$-central condition may fail.
\end{myexample}

\begin{myexample}[Subgaussian Regression] \label{ex:regression}
  Examples~\ref{ex:log-loss}, \ref{ex:unbounded-sq-cc} and
  \ref{ex:dunno} all deal with the unconditional setting
  (cf. page~\pageref{page:unconditional}) of estimating a mean without
  covariate information. The corresponding conditional setting is
  regression, in which $\cF$ is a set of functions $f: \cX \rightarrow
  \cY$, $\cZ = \cX \times \cY$, $\cY = \reals$ and $\regloss_f( (x,y))
  := \sqloss_{f(x)}(y)$. Analogously to Example~\ref{ex:dunno}, fix
  $\sigma^2 > 0$ and let $\cP$ be a set of distributions on $\cX
  \times \cY$ such that for each $P \in \cP$ and $x \in \cX$, $P(Y
  \mid X=x)$ is subgaussian in the sense of (\ref{eq:moment-bound}).
  Now consider a decision problem $(\regloss,\cP,\cF)$.
  Example~\ref{ex:dunno} applies to this regression setting, provided that, 
  for each $P \in \cP$, the model $\cF$ contains the true
  regression function $f^*_P(x) := \E_{(X,Y) \sim P}[Y \mid X=x]$. To
  see this, note that then for all $P \in \cP$, all $f' \in \cF$,
\begin{align*}
 \E_{(X,Y) \sim P} \left[e^{\eta\left(\regloss_{f^*_P}(X,Y)-\regloss_{f'}(X,Y)\right)}\right]
& = \E_{P(X)} \E_{P(Y\mid X)} \left[e^{\eta\left(\sqloss_{f^*_P(X)}(Y)-\sqloss_{f'(X)}(Y)\right)}\right] \\
&    \leq \E_{P(X)} \left[e^{-\eta \delta^2/2} e^{\sigma^2 \eta^2 \delta^2/2} \right] \leq 1,  
\end{align*}
where the final inequality holds as long as $\eta \leq
1/\sigma^2$. Thus the $1/\sigma^2$-central condition holds. Although
it is often made, the assumption that $\cF$ contains the Bayes decision
rule (\ie, the true regression function) is quite strong. In
Section~\ref{sec:xyz} we will encounter Example~\ref{ex:misspecified} where,
under a compactness restriction on $\cP$, the central condition still
holds even though $\cF$ may be misspecified.
\end{myexample}
\begin{myexample}[Bernoulli, $0/1$-loss and the margin condition] \label{ex:prototsybakov} 
  Let $\cZ = \cF = \{0,1\}$, for any $0 \leq \delta
  \leq 1/2$ let $\cP_{\delta}$ be the set of distributions $P$ on $\cZ$ with $|P(Z= 1) - 1/2| \geq  \delta$,
  and let $\zoloss$ be the $0/1$-loss with  $\zoloss(y,f) = |y - f|$.
  For every $\delta > 0$, there is an $\eta >  0$ such that
  the $\eta$-central condition holds for
  $(\zoloss,\cP_{\delta},\cF)$. To see this, let $f^*$ be the Bayes act
  for $P$, \ie,  $f^* =1$ if and only if $P(Z=1) > 1/2$, and, for $f \neq f^*$,
  define $A(\eta)  = \E_{Z \sim P} \left[e^{\eta (\zoloss_{f^*}(Z)-
      \zoloss_{f}(Z))}\right]$. Then $A(0) = 1$  and the derivative
  $A'(0)$ is easily seen to be negative, which  implies the
  result. However, as $\delta \downarrow 0$, so does  the largest
  $\eta$ for which the central condition holds. For $\delta = 0$, the
  central condition does not hold any  more. Since the central
  condition and the PPC condition are  equivalent, this also follows
  from Proposition~\ref{prop:unique}: if  $\delta = 0$, then there
  exist $P \in \cP$ with $P(Z=1 ) = 1/2$, and  for this $P$ both $f
  \in \cF = \{0,1\}$ have equal risk so there is no unique  minimum.  For each
  $\delta > 0$, the restriction to $\cP_{\delta}$ may also be
  understood as saying that a {\em Tsybakov margin condition\/}
  \citep{tsybakov2004optimal} holds with noise exponent $\infty$, the
  most stringent case of this condition that has long been known to
  ensure fast rates. As will be seen in Example~\ref{ex:tsybakov}
  the Tsybakov margin condition
  can also be thought of as a Bernstein condition with $\beta = 0$
  and $B \uparrow \infty$ as $\delta \downarrow 0$ (in practice, however, this
  condition is usually applied in the conditional setting with
  covariates $X$).  Finally, just like the squared loss
  examples, this example can be recast in terms of log-loss as well.
  Fix $\beta > 0$ and let $\cF_{\beta} $ be the subset of the
  Bernoulli model containing two symmetric probability mass functions,
  $p_1$ and $p_0$, where $p_1(1) = p_0(0) = e^{\beta}/(1+ e^{\beta}) >
  1/2$.  Then the log loss Bayes act for $P$ is $p_1$ if and only if $P(Z=1) >
  1/2$. For $P \in \cP_{\delta}$ and $f' \neq f^*$, $\E_{Z \sim P}
  \left[ e^{\eta (\logloss_{f^*}(Z)- \logloss_f(Z))}\right] = A (\beta
  \eta)$, which by the same argument as above can be made $< 1$ if
  $\eta> 0$ is chosen small enough (provided $\delta > 0$).
\end{myexample}

\subsection{Equivalence of Central and Pseudoprobability Convexity Conditions}
\label{sec:equivalenceCCandPPC}

The following result shows that no additional assumptions are required for
the central condition to imply the pseudoprobability convexity
condition.
\begin{proposition}\label{prop:comparatortoconvexface}
Fix an arbitrary decision problem $\dpthree$ and
$\epsilon \geq 0$. If the
$\eta$-central condition holds up to $\epsilon$ then the  $\eta$-pseudoprobability convexity
condition holds up to $\epsilon$.  In particular the (strong)
$\eta$-central condition implies the (strong) $\eta$-pseudoprobability convexity
condition.
\end{proposition}

\begin{proof} Let $P \in \cP$ and $\Pi \in \Delta(\cF)$ be
  arbitrary. Assume the $\eta$-central condition holds up to $\epsilon$. Then
\begin{align*}
\E_{Z \sim P} \left[ \ell_{\comp(P)}(Z) - m^\eta_\Pi(Z) \right]
& = 
\frac{1}{\eta} \E_{Z \sim P} \log  \E_{f \sim \Pi}
\left[e^{\eta \left( \ell_{\comp(P)}(Z)- \ell_{f}(Z) \right) } \right] \\ & \leq  
\frac{1}{\eta} \log \E_{Z \sim P} \E_{f \sim \Pi}
\left[e^{\eta \left( \ell_{\comp(P)}(Z)- \ell_{f}(Z) \right) } \right] \leq 
 {\epsilon}.
\end{align*}
where the first inequality is Jensen's and the second inequality
follows from the central condition \eqref{eqn:comparator}.
\end{proof}
To obtain the reverse implication we require either Assumption~\ref{ass:simple}
(\ie, that minimum risk within $\cF$ is achieved) or, if
Assumption~\ref{ass:simple} does not hold, the boundedness of the 
loss\footnote{%
  We suspect this latter requirement can be weakened, at the cost of
  considerably complicating the proof.}. Below we use the term
`essentially unique' in the sense of Proposition~\ref{prop:unique}
and call any $g^*$ such that $\ell_{g^*}(Z)
= \ell_{f^*}(Z)$ occurs $P$-almost-surely a {\em version\/} of $f^*$.

\begin{theorem}\label{thm:convexfacetocomparator}
Let $\dpthree$ be a decision problem. Then the following statements both hold:
	\begin{enumerate}
\item  If $\ell$ is bounded,
then the weak $\eta$-pseudoprobability convexity
  condition implies the weak $\eta$-central condition. 
\item Moreover, if Assumption~\ref{ass:simple} holds, then (irrespective of
  whether the loss is bounded) the weak $\eta$-pseudoprobability
  convexity condition implies the strong $\eta$-central condition
  with comparator function $\comp(P) := f^*$ for $\cF$-optimal $f^*$. 
  That is, $f^*$ can be any version of the essentially
  unique element of $\cF$ that satisfies (\ref{eq:bayesact}).
\end{enumerate}
\end{theorem}
The proof of Theorem~\ref{thm:convexfacetocomparator} is deferred to
Appendix~\ref{app:convex-faceproofs}. It generalizes a result for log
loss from the PhD thesis of \citet[Theorem~4.3]{li1999estimation} and
\cite{Barronpersonal}.\footnote{Under Assumption~\ref{ass:simple}, the
proof of Theorem~\ref{thm:convexfacetocomparator}
shows that it is actually sufficient if the weak pseudoprobability
convexity condition only holds for distributions $\Pi$ on $f^*$ and single
 $f \in \model$. Via
Proposition~\ref{prop:comparatortoconvexface} we then see that this
actually implies weak pseudoprobability convexity for all distributions
$\Pi$.}
Theorem~\ref{thm:convexfacetocomparator} leads to the following useful
consequence.
\begin{corollary}\label{cor:teeth}
  Consider a decision problem $\dpthree$ and suppose that
  Assumption~\ref{ass:simple} holds. Then the following are
  equivalent:
  \begin{enumerate}
    \item The weak $\eta$-central condition is satisfied. \item The
    strong $\eta$-central condition is satisfied with comparator function
    $\comp$ as given by Theorem~\ref{thm:convexfacetocomparator}.
    \item The weak $\eta$-pseudoprobability convexity condition is satisfied.
    \item The strong $\eta$-pseudoprobability convexity condition is satisfied.
  \end{enumerate}
If any of these statements hold, then for all $P \in \cP$,
the corresponding optimal $f^*$ is essentially unique in the
sense of Proposition~\ref{prop:unique}. 
\end{corollary}
\begin{proof} Suppose that the $\eta$-(weak) pseudoprobability
  convexity condition holds and that
  Assumption~\ref{ass:simple} holds. This implies that the
  infimum in \eqref{eqn:supsupinf-convexface} is always achieved, from
  which it follows that the strong $\eta$-pseudoprobability convexity
  condition holds. The assumption also lets us apply
  Theorem~\ref{thm:convexfacetocomparator} which implies that the
  strong $\eta$-central condition holds with $\comp$ as
  described. This immediately implies the weak $\eta$-central
  condition which, via Proposition~\ref{prop:comparatortoconvexface},
  implies the weak $\eta$-pseudoprobability convexity condition.
\end{proof}
The corollary establishes the equivalence of the weak and
strong central and pseudoprobability convexity conditions which we
assumed in Section~\ref{sec:overview}. The result prompts the question
whether {\em non-\/}uniqueness of the optimal $f^*$ might imply
that the four conditions do {\em not\/} hold. While this is not true
in general, at least for bounded losses it is `almost' true if we
replace the $\eta$-fast rate conditions by the weaker notion of
$\gComparator$--fast rate conditions of
Section~\ref{sec:comp-mix-margin} (see
Proposition~\ref{prop:nonunique}).

\subsection{Interpretation as Convexity of the Set of
Pseudoprobabilities and a Bayesian Interpretation}
\label{sec:convexityinterpretation}
As we will now explain both the pseudoprobability convexity condition and, by the
equivalence from the previous section, the central condition
may be interpreted as a partial convexity requirement. For simplicity, we 
restrict ourselves to the setting of Assumption~\ref{ass:simple} from
Section~\ref{sec:overview}.
We first present this interpretation for the logarithmic loss from
Example~\ref{ex:log-loss} on page~\pageref{ex:log-loss}, for which it is
most natural and can also be given a Bayesian interpretation.
\begin{myexample}[Example~\ref{ex:log-loss} continued: convexity
interpretation for log loss]\label{ex:log-loss-continued}
  Let $P \in \cP$ be arbitrary. Under Assumption~\ref{ass:simple} the
  strong $1$-pseudoprobability convexity condition for log loss says that
  \begin{align}
    \E_{Z \sim
    P}\left[-\log f^*(Z)\right]
      &\leq \min_{\Pi \in \Delta(\model)} \E_{Z \sim P}\left[-\log \E_{f \sim
    \Pi}[f(Z)]\right], \ \ \text{\ie,}\ \ \ \notag\\
    \min_{f \in \model} \E_{Z \sim P}\left[-\log f(Z)\right]
      &= \min_{f \in \convhull(\model)} \E_{Z \sim P}\left[-\log f(Z)\right],
	  \label{eqn:loglossconvexity}
  \end{align}
  where $f^* = \phi(P)$ and $\convhull(\model)$ denotes the convex
  hull of $\model$ (\ie, the set of all mixtures of densities
  in $\cF$). This may be interpreted as the requirement that a convex
  combination of elements of the model $\model$ is never better than
  the best element in the model. This means that the model is
  essentially convex with respect to $P$ (\ie,
  `in the direction facing' $P$ --- see Figure~\ref{fig:convexitycartoon}).
  
  In particular, in the context of Bayesian inference, the
  \emph{Bayesian predictive distribution\/} after observing data $Z_1,
  \ldots, Z_n$ is a mixture of elements of the model according to the
  posterior distribution, and therefore must be an element of
  $\convhull(\model)$. The pseudoprobability convexity condition thus
  rules out the possibility that the predictive distribution is
  strictly better (in terms of expected log loss or, equivalently,
  KL-divergence) than the best single element in the model. This
  might otherwise be possible if the posterior was spread out
  over different parts of the model. This interpretation is explained
  at length by \cite{GrunwaldVanOmmen14} who provide a simple
  regression example in which (\ref{eqn:loglossconvexity}) does not hold and
  the Bayes predictive distribution is, with substantial
  probability, better than the best single element $f^*$ in the model,
  and the Bayesian posterior does not concentrate around this optimal
  $f^*$ at all.
\end{myexample}
For log loss,  the
convexity requirement \eqref{eqn:loglossconvexity} is, by
Corollary~\ref{cor:teeth}, equivalent to the strong $1$-central condition  and can thus be written as
  \begin{equation}\label{eq:bayesmdl}
    \E_{Z \sim P} \left[\frac{f(Z)}{f^*(Z)}\right] \leq
    1 
  \end{equation}
for all $f \in \cF$. Recognizing
  \eqref{eqn:affinity} we therefore also recover the result by
  \citet{li1999estimation} mentioned in Example~\ref{ex:log-loss}.
\begin{myexample}[Bayes-MDL Condition]
  The $1$-central condition (\ref{eq:bayesmdl}) for log loss plays a
  fundamental role in establishing consistency and fast rates for
  Bayesian and related methods. Due to its use in a large number of
  papers on convergence of MDL-based methods
  \citep{grunwald2007minimum} and Bayesian methods and lack of a standard
  name, we will henceforth call it the {\em Bayes-MDL condition}. 
  Most of the papers using this condition make the traditional assumption 
  that the model is well-specified, \ie, for every $P \in \cP$, $\cF$ contains
  the density of $P$. As already mentioned in
  Example~\ref{ex:log-loss}, the condition then holds automatically,
  so one does not see (\ref{eq:bayesmdl}) stated in those papers as an explicit
  condition. Yet, if one tries to generalize the
  results of such papers to the misspecified case, one invariably sees
  that the only step in the proofs needing adjustment is the step
  where (\ref{eq:bayesmdl}) is implicitly employed. If the model is
  incorrect yet (\ref{eq:bayesmdl}) holds, then the proofs invariably still
  go through, establishing convergence towards the $f^*$ that
  minimizes KL divergence to the true $P$.  This happens, for
  example, in the MDL convergence proofs of
  \cite{barron1991minimum,zhang2006,grunwald2007minimum} as well as in
  the pioneering paper by \cite{doob1949application} on Bayesian
  consistency.  The dependence on (\ref{eq:bayesmdl}) becomes more
  explicit in works explicitly dealing with misspecification such as
  those by \cite{li1999estimation,kleijn2006misspecification,grunwald2011safe}. 
  For example, in order to guarantee convergence of the posterior around
  the best element $f^*$ of misspecified models,
  \citet{kleijn2006misspecification} impose a highly technical
  condition on $\dpthree$. If, however, (\ref{eq:bayesmdl}) holds then
  this complicated condition simplifies to the standard, much simpler
  condition from \citep{GhosalGV00} which is sufficient for
  convergence in the well-specified case. The same phenomenon is seen
  in results by \cite{ramamoorthi2013posterior,de2013bayesian}.
  \cite{GrunwaldL04}
  and \cite{GrunwaldVanOmmen14} give examples in which the condition
  does not hold, and Bayes and MDL estimators fail to converge.
\end{myexample}

\begin{figure}
  \centering
  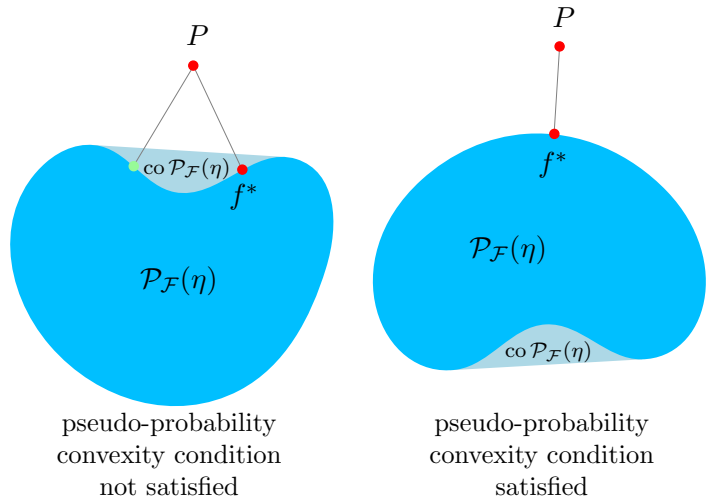   
  \caption{The pseudoprobability convexity condition interpreted as convexity of the set
  of pseudoprobabilities with respect to $P$.}
  \label{fig:convexitycartoon}
\end{figure}

The convexity interpretation for log loss may be generalized to other
loss functions via loss dependent `pseudoprobabilities'. These play a
crucial role both in online learning \citep{vovk2001competitive} and the
PAC-Bayesian analysis of the Bayes posterior and the MDL estimator by
\citet{zhang2006}. For log loss, we may express the ordinary
densities in terms of the loss as $f(z) = e^{-\loss_f(z)}$. This
generalizes to other loss functions by letting $\eta \loss_f(z)$ play
the role of the log loss, where $\eta > 0$ is the scale factor that
appears in all our definitions. We thus obtain the set of
\emph{pseudoprobabilities}
\begin{equation*}
  \pseudoset = \left\{z \mapsto e^{-\eta \loss_f(z)} : f \in \model\right\},
\end{equation*}
which are non-negative, but do not necessarily integrate to $1$. The
only feature we need of these pseudoprobabilities is that their log loss
is equal to $\eta$ times the original loss, because, analogously to
\eqref{eqn:loglossconvexity}, this allows us to write the strong
$\eta$-pseudoprobability convexity condition as
\begin{equation*}
  \min_{f \in \pseudoset} \E_{Z \sim P}\left[-\log f(Z)\right]
    \leq \min_{f \in \convhull(\pseudoset)} \E_{Z \sim P}\left[-\log
    f(Z)\right].
\end{equation*}
Figure~\ref{fig:convexitycartoon} provides a graphical illustration of this 
condition. Thus, for
any loss function we can interpret the pseudoprobability convexity
condition as the requirement that the set of pseudoprobabilities is
essentially convex with respect to $P$. As suggested by
\citet{vovk2001competitive,zhang2006}, one can also run Bayes on such
pseudoprobabilities, and then the pseudoprobability convexity condition
again implies that the resulting pseudo-Bayesian predictive
distribution cannot be strictly better than the single best element of
the model. The log loss achieved with such pseudoprobabilities, 
and hence $\eta$ times the original loss,
can be given a code length interpretation, essentially allowing
arbitrary loss functions to be recast as versions of logarithmic loss \citep{Grunwald08b}.

\section{Online Learning}
\label{sec:four-conditions}

In this section, we discuss conditions for fast rates that are related
to online learning. Our key concept is introduced in
Section~\ref{sec:stochmix}, where we define \emph{stochastic
mixability}, the natural stochastic generalization of Vovk's notion of
mixability, and show (in Section~\ref{sec:relations}) how it unifies
existing conditions in the literature.
Section~\ref{sec:OnlineStatisticalRelations} contains the main results
for this section, which connect stochastic mixability to the central
condition and to pseudoprobability convexity. As an intermediate step,
these results use a fourth condition called the \emph{predictor
condition}, which is related to the central condition via a minimax
identity. We show that, under appropriate assumptions, all four
conditions are equivalent. This equivalence is important because it relates the
generic condition for fast rates in online learning (stochastic
mixability) to the generic condition that enables fast rates for proper
in-model estimators in statistical learning (the central condition).

\subsection{Stochastic Mixability in General}\label{sec:stochmix}

Stochastic mixability generalizes from (\ref{eq:stochmixpre}) similarly to the way we have
generalized the central condition and pseudoprobability convexity. Let
$m^\eta_\Pi(z)$ be the mix loss, as defined in \eqref{eqn:mixloss}.

\begin{definition}[The Stochastic Mixability Condition]
\label{def:stochastic-mixability} 
Let $\eta > 0$ and $\epsilon \geq 0$. We say that
\dpfour\/ is \emph{$\eta$-stochastically mixable}
up to $\epsilon$  if there exists a \emph{substitution function}
$\pred\colon \Delta(\cF) \rightarrow \decisionset$ such that
\begin{align}\label{eqn:mix}
  \E_{Z \sim P}
  \left[\ell_{\pred(\Pi)}(Z) \right] \leq \E_{Z \sim P}
  \left[m^\eta_\Pi(Z)\right] + \epsilon
  \qquad\text{for all $P \in \cP$ and $\Pi \in \Delta(\cF)$.}
\end{align}
If it is $\eta$-stochastically mixable up to $0$, we say that it is {\em
  strongly $\eta$-stochastically mixable\/} or simply {\em $\eta$-stochastically mixable}. If it is
$\eta$-stochastically mixable up to $\epsilon$ for all $\epsilon > 0$, we say
that it is {\em weakly $\eta$-stochastically mixable}; this is equivalent to 
\begin{align}\label{eqn:supinfsup-protomix}
\sup_{\Pi \in \Delta(\cF)}  \lowinf_{f \in \decisionset}  \sup_{P \in \cP} 
\ \E_{Z \sim P} \left[ \ell_{f}(Z)- m^{\eta}_{\Pi}(Z) \right] \leq 0.
\end{align}
\end{definition}
Unlike for the central and pseudoprobability convexity conditions (see
Corollary~\ref{cor:teeth}), for stochastic mixability it is not clear
whether the weak and strong versions become equivalent under the
simplifying Assumption~\ref{ass:simple}. We do have a trivial yet important extension of Fact~\ref{fact:ccppcc}:
\begin{fact}\label{fact:smpc} Fix $\eta > 0, \epsilon \geq 0$ and let $\dpfour$ be an arbitrary decision problem that
  is $\eta$-stochastically mixable up to $\epsilon$. Then for
  any $0 < \eta' \leq \eta$, any $\epsilon' \geq \epsilon$  and for any $\cP' \subseteq \cP$, $\cF'
  \subseteq \cF$ and $\decisionset' \supseteq \decisionset$,
  $(\ell,\cP', \cF',\decisionset')$ is $\eta'$-stochastically mixable up to $\epsilon'$. 
\end{fact}

\subsection{Relations to Conditions in the Literature}
\label{sec:relations}

As explained next, stochastic mixability generalizes Vovk's notion of
(non-stochastic) mixability, and correspondingly implies fast rates. Its
most important special case is stochastic exp-concavity, for which
\citet{juditsky2008learning} give sufficient conditions, and which is used by, e.g., \cite{DalalyanT12}. Stochastic
mixability is also equivalent to a special case of a
condition introduced by \citet{audibert2009fast}.

\subsubsection{Generalization of Vovk's Mixability and Fast Rates for
Stochastic Prediction with Expert
Advice}\label{sec:classical-mixability}

If we take $\epsilon = 0$ and let $\cP$ be the set of all possible
distributions, then \eqref{eqn:mix} reduces to
\begin{equation}\label{eq:vovk}
  \ell_{\pred(\Pi)}(z) \leq m^\eta_\Pi(z)
  \qquad\text{for all $z \in \cZ$ and $\Pi \in \Delta(\cF)$,}
\end{equation}
which is Vovk's original definition of (non-stochastic) mixability
\citep{vovk2001competitive}. It follows that Vovk's mixability implies
strong stochastic mixability for all sets $\cP$.

\begin{myexample}[Mixable Losses]\label{ex:mixable}
  Losses that are classically mixable in Vovk's sense, include the
  squared loss $\sqloss(f,z) = \half(z-f)^2$ on a bounded domain
  $\cZ = \decisionset \supseteq \model = [-B,B]$, which is $1/B^2$-mixable
  \citep[Lemma~3]{vovk2001competitive}\footnote{Taking into account the
  factor of $\frac{1}{2}$ difference between his definition of squared
  loss as $(z-f)^2$ and ours.}, and the logarithmic loss, which is
  $1$-mixable for $\decisionset \subseteq \convhull(\model)$ with
  substitution function equal to the mean $\pred(\Pi) = \E_{f \sim \Pi} [f]$. 
  The \emph{Brier score} is also $1$-mixable \citep{VovkZhdanov2009,Erven2012};
  this loss function is defined 
  for all possible probability distributions
  $\decisionset = \model$ on a finite set of outcomes
  $\cZ$
  according to $\brierloss_f(z) = \sum_{z' \in \cZ} (f(z') -
  \delta_z(z'))^2$, where $\delta_z$ denotes a point-mass at $z$.
\end{myexample}
\begin{myexample}[$0/1$ Loss: Example~\ref{ex:prototsybakov}, Continued]
  \label{ex:protototsybakov} Fix $0 \leq \delta \leq 1/2$ and consider
  a decision problem $(\zoloss,\cP_{\delta}, \cF)$ where $\zoloss$ is
  the $0/1$-loss, $\cZ = \cF = \{0,1\}$ and $\cP_{\delta}$ is as in
  Example~\ref{ex:prototsybakov}. The $0/1$-loss is not $\eta$-mixable
  for any $\eta > 0$ \citep{vovk1998game}, and it is also easily shown
  that $(\zoloss,\cP_{\delta}, \cF, \cF)$ is  not
  $\eta$-stochastically mixable for any $\eta > 0$; nevertheless, if
  $\delta > 0$, then $(\zoloss,\cP_{\delta}, \cF)$ does satisfy the
  $\eta$-central condition for some $\eta > 0$. In
  Section~\ref{sec:OnlineStatisticalRelations} we show that, under some
  conditions, the $\eta$-central condition and $\eta$-stochastic
  mixability coincide, but this example shows that this cannot always
  be the case.
\end{myexample}
Vovk defines the \emph{aggregating algorithm}
(AA) and shows that it achieves constant regret in the setting of
prediction with expert advice, which is the online learning
equivalent of fast rates, provided that \eqref{eq:vovk} is
satisfied. In prediction with expert advice, the data
$Z_1,\ldots,Z_n$ are chosen by an adversary, but one may define a
stochastic analogue by letting the adversary instead choose
$P_1,\ldots,P_n \in \cP$, where the choice of $P_i$ may depend on the
player's predictions on rounds $1,\ldots,i-1$, and letting $Z_i \sim
P_i$ for all $i=1,\ldots,n$.
It turns out that under no further conditions, stochastic mixability
implies fast rates for the expected regret under $P_1,\ldots,P_n$ in
this stochastic version of prediction with expert advice. In
particular, there is no requirement that losses are bounded. 

\begin{proposition}\label{prop:stoch-mix-AA}
Let $(\loss, \cP, \model, \decisionset)$ be $\eta$-stochastically
mixable up to $\epsilon$ with substitution function $\pred$. Assume
the data $Z_1, \ldots, Z_n$ are distributed as $Z_j \sim P_j \in \cP$
for each $j \in [n]$, where the $P_j$ can be adversarially
chosen. Then the AA, playing $f_j \in \decisionset$ in round $j$,
achieves, for all $f \in \cF$, regret 
\begin{align*}
\sum_{j=1}^n \E_{Z_j \sim P_j} \left[ \loss_{f_j}(Z_j) - \loss_f(Z_j) \right] \leq \frac{\log |\cF|}{\eta} + n \epsilon .
\end{align*}
In particular, in the statistical learning (stochastic i.i.d.) setting
where $P_1, \ldots, P_n$ all equal the same $P$, online-to-batch
conversion yields the bound $\frac{\log |\cF|}{\eta n} + \epsilon$ on
the expected regret and hence on the rate (\ref{eq:riskrate}) of the AA
is $O(\frac{\log |\cF|}{\eta n} + \epsilon)$.
\end{proposition}
\begin{proof}
For $\epsilon = 0$, the first result follows by replacing every occurrence of mixability with stochastic mixability in Vovk's proof (see Section 4 of \cite{vovk1998game} or the proof of Proposition 3.2 of \cite{CesaBianchiLugosi2006}). The case of $\epsilon > 0$ is handled simply by adding a slack of $\epsilon$ to the RHS of the first equation after equation (18) of \cite{vovk1998game}. 
The online-to-batch conversion of the second result is well-known and can be found e.g.~in the proof of Lemma 4.3 of \cite{audibert2009fast}.
\end{proof}

\subsubsection{Special Case: Stochastic Exp-concavity}
\label{sec:expconcave}

In online convex optimization, an important sufficient condition for
fast rates requires the loss to be \emph{$\eta$-exp-concave} in $f$
\citep{HazanAgarwalKale2007}, meaning that $\cF = \decisionset$ is convex and
that
\begin{equation}\label{eq:expconcave}
  e^{-\eta \loss_f(z)}
  \qquad \text{is concave in $f$ for all $z \in \cZ$.}
\end{equation}
We may equivalently express this requirement as
\begin{align*}
  e^{-\eta \loss_{\E_{f \sim \Pi} [f]}(z)}
    &\geq \E_{f \sim \Pi} \left[ e^{-\eta \loss_f(z)}\right] , \ \ \text{or}\\
  \loss_{\E_{f \sim \Pi} [f]}(z) &\leq m^\eta_\Pi(z) ,
\end{align*}
for all distributions $\Pi \in \Delta(\cF)$ and all $z \in \cZ$. This shows that
exp-concavity is a special case of mixability, where we require the
function $\pred$ to map $\Pi$ to its mean:
\begin{equation*}
  \pred(\Pi) = \E_{f \sim \Pi} [f].
\end{equation*}
Because the mean $\E_{f \sim \Pi} [f]$ depends not only on the losses $\loss_f$,
but also on the choice of parameters $f$, we therefore see that
exp-concavity is \emph{parametrization-dependent}, whereas in general
the property of being mixable is unaffected by the choice of
parametrization.   The parametrization dependent
nature of exp-concavity  is explored in detail by
\citet{vernet2011composite,WilliamsonZhangParameswaran2015}; see also
\citet{Erven2012,vanerven2012blog}. 

\begin{myexample}[Exp-concavity]\label{ex:exp-concavity}
  Consider again the mixable losses from Example~\ref{ex:mixable}.
  Then the log loss is $1$-exp concave. The squared loss, in its
  standard parametrization, is \emph{not} $1/B^2$-exp-concave, but it is
  $1/(4 B^2)$-exp-concave, losing a factor of $4$
  \citep[Remark~3]{vovk2001competitive}. By continuously reparametrising
  the squared loss, however, it can be made $1/B^2$-exp-concave after
  all \citep{WilliamsonZhangParameswaran2015,vanerven2012blog}. It is
  not known whether there exists a parametrization that makes the Brier
  score $1$-exp-concave.
\end{myexample}

The natural generalization of exp-concavity to stochastic exp-concavity becomes:
\begin{definition}\label{def:expconcave}
Suppose $\decisionset \supseteq \co(\cF)$. Then we say that
$(\loss, \cP, \cF, \decisionset)$ is $\eta$-stochastically exp-concave up to
$\epsilon$ or strongly/weakly $\eta$-stochastically exp-concave if it
satisfies the corresponding case of stochastic mixability with substitution function
$\pred(\Pi) = \E_{f \sim \Pi} [f]$.
\end{definition}

\subsubsection{The JRT Conditions Imply Stochastic Exp-concavity}
\label{sec:jrt}
\citet*{juditsky2008learning} introduced two conditions that
guarantee fast rates in model selection aggregation. For now we focus on
the following condition, mentioned in their Theorem~4.2, which we
henceforth refer to as the \emph{JRT-II condition}, returning to the
JRT-I condition, mentioned in their Theorem~4.1, in
Section~\ref{sec:JRTother}.
\begin{definition}[JRT-II condition]
\label{def:jrt-cond}
Let $\eta > 0$. We say that $(\loss, \cP, \cF)$ satisfies 
the \emph{$\eta$-JRT-II condition} if 
there exists a function $\gamma: \cF \times \cF \rightarrow \reals$ satisfying 
(a) for all $f \in \cF$, $\gamma(f,f) = 1$, 
(b) for all $f \in \cF$, the function $g \mapsto \gamma(f, g)$ is concave, 
and (c)
\begin{align}\label{eqn:jrt-cond}
  \text{for all\ } P \in \cP \text{ and } f,g \in \cF\colon \/ \E_{Z \sim P}
  \left[e^{\eta \left( \loss_f(Z)- \loss_{g}(Z) \right) } \right] \leq \gamma(f, g).
\end{align}
\end{definition}
This condition has been used to obtain fast $O(1/n)$ rates for the
mirror averaging estimator in model selection aggregation, which is
statistical learning against a finite class of functions $\cF =
\{f_1,\ldots,f_m\}$ \citep{juditsky2008learning}. One may interpret
their approach as using Vovk's aggregating algorithm to get $O(1)$
expected regret, and then applying online-to-batch conversion
\citep{CesaBianchiConconiGentile2004,Barron1987,YangBarron1999}, which
leads to an estimator whose risk is upper bounded by the expected regret
divided by $n$. This use of the AA is allowed, because, if $\decisionset
\supseteq \convhull(\cF)$, then the JRT-II condition implies strong
stochastic exp-concavity, as already shown by \citet{audibert2009fast}
as part of the proof of his Corollary~5.1:
\begin{proposition}\label{prop:jrt}
If $(\loss, \cP, \cF)$ satisfies the $\eta$-JRT-II condition, then 
$(\loss, \cP, \cF, \decisionset)$ satisfies the strong $\eta$-stochastic 
exp-concavity condition for any $\decisionset \supseteq \convhull(\cF)$.
\end{proposition}
\begin{proof}
From the JRT-II condition, for all $P \in \cP$ and $\Pi \in \Delta(\cF)$
\begin{align*}
\E_{g \sim \Pi} \E_{Z \sim P} e^{\eta(\loss_{\pred(\Pi)}(Z) - \loss_g(Z))} 
\leq \E_{g \sim \Pi} \gamma(\pred(\Pi),g) ,
\end{align*}
which from the concavity of $\gamma$ in its second argument is at most
\begin{align*}
\gamma \left(\pred(\Pi), \E_{g \sim \Pi} g \right) 
= \gamma \bigl( \pred(\Pi), \pred(\Pi) \bigr) 
= 1 ,
\end{align*}
by the definition of $\pred$ and part (a) of the JRT-II condition. Thus,
we have
\begin{equation*}
\E_{g \sim \Pi} \E_{Z \sim P} e^{\eta(\loss_{\pred(\Pi)}(Z) - \loss_g(Z))} 
  \leq 1.
\end{equation*}
Applying Jensen's inequality to the exponential function completes the
proof.
\end{proof}
\cite{juditsky2008learning} use the JRT-II condition in the proof of their
Theorem 4.2 as a sufficient condition for another condition, which is
then shown to imply $O(1/n)$ rates for finite classes $\cF$. After
some basic rewriting, this other condition (which requires the formula
below Eq. (4.1) in their
paper to be $\leq 0$) is seen to be equivalent to strong stochastic exp-concavity as
defined in Definition~\ref{def:expconcave}, i.e. it requires that
(\ref{eqn:mix}) holds with $\epsilon = 0$ and substitution function
$\pred(\Pi) = \E_{f \sim \Pi} [f]$. The JRT-I condition, which we define in
Section~\ref{sec:JRTother}, can be related to stochastic exp-concavity
with nonzero $\epsilon$, thus we may say that the {\em underlying\/}
condition that JRT work with is equivalent to our
stochastic exp-concavity condition, albeit that they restrict themselves to a finite class of functions.

\subsubsection{Relation to Audibert's Condition}
\label{sec:audibert}

\citet[p.\,1596]{audibert2009fast} presented a condition
which he called the {\em variance inequality}. It is defined relative to
a tuple $(\loss, \cP, \cF, \decisionset)$ and has the following requirement as a special
case (in Audibert's notation, this corresponds to $\delta_{\lambda}=0$
and $\hat{\Pi}$ a Dirac distribution on some $f \in \decisionset$):
\begin{equation*}
\forall \Pi \in \Delta(\cF) \; \exists f \in \decisionset \; \sup_{P \in \cP} \
\E_{Z\sim P} \log \E_{g \sim \Pi} \left[ e^{\eta(\loss_f(Z) -
    \loss_g(Z))} \right] \leq 0.
\end{equation*}
Rewriting 
\begin{equation*}
  \E_{Z\sim P} \log \E_{g \sim \Pi} \left[ e^{\eta(\loss_f(Z) -
    \loss_g(Z))} \right]
    = \eta \E_{Z\sim P}[\loss_f(Z) - m^\eta_\Pi(Z)],
\end{equation*}
this is seen to be precisely equivalent to strong stochastic mixability.

\subsection{Relations with Central and Pseudoprobability Convexity Conditions}
\label{sec:OnlineStatisticalRelations}

We now turn to the relations between stochastic mixability and the two
main conditions from Section~\ref{sec:convex-face}: the central
condition and pseudoprobability convexity. We first define the
predictor condition, which will act as an intermediate step, and then
show the following implications:
\begin{equation*}
 \text{predictor }  \Rightarrow \text{stochastic mixability }
 \Rightarrow \text{PPC } \Rightarrow \text{CC } \Rightarrow
 \text{predictor}
  \qquad \text{(under assumptions.)}
\end{equation*}
The implication from pseudoprobability convexity to the central
condition was shown in Theorem~\ref{thm:convexfacetocomparator} from
Section~\ref{sec:equivalenceCCandPPC}; we will consider the other ones
in turn in this section. The second implication is of special interest
since, in the online setting, there is extra power because predictions
may take place in a set $\decisionset$ that can be larger than
$\cF$. The conditions of the second implication will identify
situations in which this additional power is not helpful.

\subsubsection{The Predictor Condition in General}

We define the general predictor condition as follows:
\begin{definition}[Predictor Condition]
\label{def:martingale-mixability} 
Let $\eta > 0$ and $\epsilon \geq 0$. {\sloppy We say that \dpfour\/ satisfies the
\emph{$\eta$-predictor condition} up to $\epsilon$ if there
exists a \emph{prediction function} $\pred\colon \Delta(\cF) \rightarrow
\cF_\D$ such that}
\begin{equation}\label{eqn:martingale-mixability}
  \E_{Z \sim P} \E_{f \sim \Pi}
  \left[e^{\eta \left( \ell_{\pred(\Pi)}(Z)- \ell_{f}(Z) \right) } \right] \leq 
e^{\eta \epsilon}
\qquad\text{for all $P \in \cP$ and distributions $\Pi$ on $\cF$.}
\end{equation}
If it satisfies the $\eta$-predictor condition up to $0$, we say that
the {\em strong $\eta$-predictor condition\/} or simply the
{\em $\eta$-predictor condition} holds. If it satisfies the
$\eta$-predictor condition up to $\epsilon$ for all $\epsilon > 0$, we
say that the {\em weak $\eta$-predictor condition\/} holds; this is
equivalent to
\begin{align}\label{eqn:supinfsup-martingale-mixability}
\sup_{\Pi \in \Delta(\cF)} \lowinf_{f \in \cF_\D}  \sup_{P \in \cP} 
\E_{Z \sim P}
\E_{g \sim \Pi}
\left[e^{\eta \left(\loss_{f}(Z)- \loss_{g}(Z) \right) } \right]
\leq 1.
\end{align}
\end{definition}
Comparing \eqref{eqn:supinfsup-martingale-mixability} to the central condition, we see that the predictor condition
looks similar, except that the suprema over $\Pi$ and $P$ are
interchanged. We note that, trivially, Fact~\ref{fact:smpc} extends
from $\eta$-stochastic mixability to the $\eta$-predictor condition.

\subsubsection{Predictor Implies Stochastic Mixability}

By an application of Jensen's inequality, the predictor condition always
implies stochastic mixability, without any assumptions:
\begin{proposition}\label{prop:frompredictortoconvexface}
  Suppose that $(\cP,\ell,\model,\decisionset)$ satisfies the
  $\eta$-predictor condition up to some $\epsilon \geq 0$. Then it is
  $\eta$-stochastically mixable up to $\epsilon$. In particular, the (strong)
  $\eta$-predictor condition implies (strong) $\eta$-stochastic mixability.
\end{proposition}

\begin{proof}
  Let $P \in \cP,\Pi \in \Delta(\cF)$ and $\epsilon \geq 0$ be
  arbitrary. Then, by Jensen's inequality, the $\eta$-predictor
  condition up to $\epsilon$ implies
  \begin{equation*}
    e^{\eta \epsilon} \geq
      \E_{\substack{Z \sim P\\f \sim \Pi}}\left[e^{\eta \left(
      \loss_{\pred(\Pi)}(Z)- \loss_{f}(Z) \right) }\right]
      = \E_{Z \sim P}\left[e^{\eta \left(
      \loss_{\pred(\Pi)}(Z)- m_\Pi^\eta(Z) \right) }\right]
      \geq e^{\eta \E_{Z \sim P}\left[\loss_{\pred(\Pi)}(Z)
        - m_\Pi^\eta(Z)\right]}.
  \end{equation*}
  Taking logarithms on both sides leads to
   $\E_{Z \sim P}\left[\loss_{\pred(\Pi)}(Z)\right]
        \leq \E_{Z \sim P}\left[m_\Pi^\eta(Z)\right] + \epsilon$,
  which is $\eta$-stochastic mixability up to $\epsilon$.
\end{proof}

\subsubsection{Stochastic Mixability Implies Pseudoprobability Convexity}
\label{sec:smppc}
In Proposition~\ref{prop:fromsmtoppcc} below, we show that, under the
right assumptions, stochastic mixability implies pseudoprobability
convexity.

A complication in establishing this implication is that stochastic
mixability is defined relative to a four-tuple $\dpfour$, and allows us
to play in a decision set that is different from $\cF$, whereas the
pseudoprobability convexity is defined relative to the triple
$\dpthree$. The proposition automatically holds if one takes $\cF = \decisionset$, and then the implication
follows trivially. In practice, however, we may have a non-convex model
$\cF$ --- as is quite usual in e.g.\ density estimation --- whereas the
decision set $\decisionset$ for which we can establish that $\dpfour$ is
$\eta$-stochastically mixable is equal to the convex hull of $\cF$. It
would be quite disappointing if, in such cases, there would be no hope of
getting fast rates for in-model statistical learning algorithms. The second part of the proposition shows
that, luckily, fast rates are still possible under the following
assumption: 
\begin{assumption}\label{ass:preciseBayes}{\bf (model $\cF$ and decision
set $\decisionset$ equally good --- $\cF$ well-specified {\em relative\/} to $\decisionset$)}
  We say that Assumption~\ref{ass:preciseBayes} holds weakly for
  $\dpfour$, if, for
  all $P \in \cP$, 
  \begin{equation}\label{eqn:preciseBayes}
    \inf_{f \in \model} R(P,f) = \inf_{f \in \decisionset} R(P,f).
  \end{equation}
  We say that Assumption~\ref{ass:preciseBayes} holds strongly if
  additionally, for all $P \in \cP$, both infima are achieved:
  $\min_{f \in \model} R(P,f) = \min_{f \in \decisionset} R(P,f)$.
\end{assumption}
The strong version of Assumption~\ref{ass:preciseBayes}
implies Assumption~\ref{ass:simple} and will be used further on in Theorem~\ref{thm:minimaxSufficient}.
In a typical application of  the proposition below,  the weak Assumption~\ref{ass:preciseBayes} would be assumed relative to a
$\decisionset$ such that $\cF \subset \decisionset$.
\begin{proposition}\label{prop:fromsmtoppcc}
  Suppose that Assumption~\ref{ass:preciseBayes} holds weakly for 
  $\dpfour$. If $\dpfour$ is $\eta$-stochastically mixable up to some
  $\epsilon \geq 0$, then $\dpthree$ satisfies the
  $\eta$-pseudoprobability convexity condition up to $\delta$ for any
  $\delta > \epsilon$; in particular, weak $\eta$-stochastic mixability of $\dpfour$ implies the weak $\eta$-PPC condition for $\dpthree$. 
  Moreover, if Assumption~\ref{ass:simple} also holds and
  $\dpfour$ satisfies strong $\eta$-stochastic mixability, then $\dpthree$
  satisfies the strong $\eta$-PPC condition.
\end{proposition}
If Assumption~\ref{ass:simple} and the weak version of
Assumption~\ref{ass:preciseBayes} both hold, then, using this
proposition, if we have $\eta$-stochastic mixability for $\dpfour$
we can directly conclude from
Theorem~\ref{thm:convexfacetocomparator} that we also have the
$\eta$-central condition for $\dpthree$.   So when does
Assumption~\ref{ass:preciseBayes} hold?  Let us assume that $\dpfour$
satisfies $\eta$-stochastic mixability. In all cases we are aware of,
it then also satisfies $\eta$-stochastic mixability for
$(\ell,\cP,\cF,\decisionset')$, where $\decisionset'$ is equal to, or
an arbitrary superset of, $\convhull(\cF)$ --- in the special case of
$\eta$-stochastic exp-concavity this actually follows by definition.
An extreme case occurs if we take $\decisionset' := \cF_{\ell}$ to be
the set of all functions that can be defined on a domain
(Example~\ref{ex:introductoryexample}). Then
Assumption~\ref{ass:preciseBayes} expresses that the model $\cF$ is
well-specified. But the assumption is weaker: assuming again that
$\decisionset$ can be taken to be the convex hull of $\cF$, it also holds if $\cF$
is itself convex and contains, for all $P \in \cP$, a risk
minimizer; and also, if, more weakly still, $\cF$ is convex `in the
direction facing $P$'.  Note that, for the log-loss, we already knew
that the $1$-central condition holds under this condition, from the
Bayesian interpretation in
Section~\ref{sec:convexityinterpretation}. There we also established
a generalization to other loss functions: the $\eta$-central
condition holds if the set of pseudoprobabilities $\cP_{\cF}$ is
convex `in the direction facing $P$'
(Figure~\ref{fig:convexitycartoon}). But, for all loss functions
except log-loss, that was a condition involving {\em pseudo\/}probabilities
and {\em artificial\/} (mix) losses. The novelty of
Proposition~\ref{prop:fromsmtoppcc} is that, if $\eta$-stochastic
mixability holds for $\dpfour$ with $\decisionset = \convhull(\cF)$
(as e.g. when we have $\eta$-stochastic exp-concavity), then the
result generalizes further to `the $\eta$-central condition holds if
the set $\cF$ {\em itself\/} (rather than the artificial set $\cP_{\cF}$)
is convex in the direction facing $P$'. 
\begin{myexample}[Fast Rates in Expectation rather than Probability]\label{ex:audibertrate}
  Fast rate results proved under the $\eta$-central condition, such
  as our result in Section~\ref{sec:fast-rates} and the various
  results by \cite{zhang2006information} generally hold both in
  expectation and in probability. The situation is different for
  $\eta$-stochastic mixability: extending the analysis of Vovk's
  Aggregating Algorithm to tuples $\dpfour$ and using the
  online-to-batch conversion, we can only prove a fast rate result
  in expectation, and not in
  probability. \cite{audibert2007progressive} provides a by now
  well-known example $(\sqloss,\cP,\cF,\convhull(\cF))$ with squared
  loss in which the rate obtained by the exponentially weighted
  forecaster (the aggregating algorithm applied with $\pred(\Pi) =
  \E_{f \sim \Pi} [f]$) followed by online-to-batch conversion is
  $O(1/n)$ in expectation, yet only $\asymp 1/\sqrt{n}$ in
  probability; and ERM also gives a rate, both in-probability and
  in-expectation of $1/\sqrt{n}$ (Theorem 2 of
  \citep{audibert2007progressive}). As might then be expected, in
  Audibert's decision problem $\eta$-exp-concavity
  holds for some
  $\eta > 0$ yet the central condition does not hold for any $\eta >
  0$. Proposition~\ref{prop:fromsmtoppcc} then implies that
  Assumption~\ref{ass:preciseBayes} must be violated: the best $f \in
  \convhull(\cF)$ is better than the best $f \in \cF$. Inspection of
  the example shows that this indeed the case (a related point was
  made earlier by \cite{lecue2011interplay}).
\end{myexample}

\begin{proof}{\ \bf (of Proposition~\ref{prop:fromsmtoppcc})}
Note that (\ref{eqn:mix}), the definition of $\eta$-stochastic
mixability up to $\epsilon$, can be rewritten as
$$
\forall \Pi \in \Delta(\cF) \; \exists f \in \decisionset \; \forall P \in \cP: \; 
 \E_{Z \sim P}
  \left[\ell_{f}(Z) \right] \leq \E_{Z \sim P}
  \left[m^\eta_\Pi(Z)\right] + \epsilon.
$$
This trivially implies
\begin{equation}\label{eq:mork}\forall \Pi \in \Delta(\cF) \; \forall P \in \cP \; \exists f \in \decisionset: \; 
 \E_{Z \sim P}
  \left[\ell_{f}(Z) \right] \leq \E_{Z \sim P}
  \left[m^\eta_\Pi(Z)\right] + \delta,
\end{equation}
for any $\delta \geq \epsilon$.  This implies that for any $\delta >
\epsilon$,  we can
assume that the choice of $f$ in \eqref{eq:mork} only depends on $P$
and not on $\Pi$. We would therefore obtain $\eta$-pseudoprobability
convexity up to any $\delta> \epsilon$ of $\dpthree$ if we could
replace $\decisionset$ by $\model$, which is trivial if $\decisionset
= \model$ and allowed under Assumption~\ref{ass:preciseBayes} because
it implies that, for any $f \in \decisionset$ we can find $f' \in
\model$ such that $\E_{Z \sim P} \left[\ell_{f'}(Z) \right] - \E_{Z
  \sim P} \left[\ell_{f}(Z) \right] \leq \delta -\epsilon$.

For the final implication, note that under Assumption~\ref{ass:simple}
we can choose $\delta = \epsilon$, and by Corollary~\ref{cor:teeth} we can
choose $\epsilon = 0$.
\end{proof}

\subsubsection{The Central Condition Implies the Predictor Condition}

We proceed to study when the central condition implies the predictor
condition (with $\decisionset = \model$), which requires the strongest
assumptions among the implications we consider. We first identify a
minimax identity \eqref{eqn:minimaxEquality} that is sufficient by
itself (Theorem~\ref{thm:minimaxSufficient}), but difficult to
verify directly. We therefore weaken Theorem~\ref{thm:minimaxSufficient} to
Theorem~\ref{thm:secondmain} by providing sufficient conditions
(Assumption~\ref{ass:convexityContinuity}) for the minimax identity.

For any $\Pi$ and $\eta$, define the function
\begin{equation*}
  S_\Pi^\eta(P,f)
    = \E_{Z \sim P} \E_{g \sim \Pi}\left[e^{\eta\left(\loss_f(Z) -
    \loss_g(Z)\right)}\right],
\end{equation*}
which is the main quantity in the definitions of both the central
and the predictor condition.

\begin{assumption}[Minimax Assumption]\label{ass:minimax}
  For given $\eta > 0$, we say that the
  \emph{$\eta$-minimax assumption} is satisfied for $\dpfour$ if, for all
  $\Pi \in \Delta(\cF)$ and for all $C \geq 1$, the following implication holds:
  \begin{equation}\label{eqn:minimaxImplication}
    \sup_{P \in \cP} \lowinf_{f \in \decisionset} S_\Pi^\eta(P,f) \leq
   C
    \qquad \Longrightarrow \qquad
    \lowinf_{f \in \decisionset} \sup_{P \in \cP} S_\Pi^\eta(P,f) \leq
    C.
  \end{equation}
\end{assumption}
We call this the minimax assumption, because
\eqref{eqn:minimaxImplication} is implied by the minimax identity
\begin{equation}\label{eqn:minimaxEquality}
 \sup_{P \in \cP} \lowinf_{f \in \decisionset} S_\Pi^\eta(P,f)   
  = \lowinf_{f \in \decisionset} \sup_{P \in \cP} S_\Pi^\eta(P,f).
\end{equation}
Theorem~\ref{thm:minimaxSufficient} below implies that
Assumption~\ref{ass:minimax} is sufficient for the central condition
to imply the predictor condition, with $\decisionset = \model$.
Intuitively, Assumption~\ref{ass:minimax} should hold under broad
conditions --- just like standard minimax theorems hold under broad
conditions. Below we will identify the specific, less elegant but
more easily verifiable Assumption~\ref{ass:convexityContinuity} that
implies Assumption~\ref{ass:minimax}. However, like conditions for
standard minimax theorems, in some cases
Assumption~\ref{ass:convexityContinuity} requires $\decisionset
\subset \reals$ to be compact, yet we want to apply the theorem also
in cases where $\cF = \reals$. As shown in
Example~\ref{ex:unbounded-sq-pc}, in this case we can sometimes
still use Part (b) of the result, which implies that the assumption
is still sufficient if we take a smaller set $\decisionset \subset
\cF$ that satisfies Assumption~\ref{ass:preciseBayes}. Note that
Assumption~\ref{ass:preciseBayes} also played a crucial role in
going from stochastic mixability of $\dpfour$ to the PPC condition
for $\dpthree$.
\begin{theorem}\label{thm:minimaxSufficient}
  Consider a decision problem $\dpthree$. Suppose that $\dpfour$ is
  such that the the $\eta$-minimax assumption
  (Assumption~\ref{ass:minimax}) holds. Then

(a) if $\cF = \decisionset$ and the
  $\eta$-central condition holds up to some $\epsilon \geq 0$ for
  $\dpthree$, then the $\eta$-predictor condition holds up to any
  $\delta > \epsilon$ for $\dpfour$. In particular, the   weak $\eta$-central condition implies the weak $\eta$-predictor
  condition. Moreover, 

(b) if  $\cF \supseteq
  \decisionset$ and $\dpfour$
  satisfies the strong version of Assumption~\ref{ass:preciseBayes}, then the
  weak $\eta$-central condition for $\dpthree$ implies the weak $\eta$-predictor
  condition for $\dpfour$ and therefore also for $\dpfoursame$.
\end{theorem}
Once we establish that the $\eta$-predictor condition holds
for $\dpfour$ with $\decisionset \subset \cF$, by Fact~\ref{fact:smpc}
we can also infer that the $\eta$-predictor condition holds for
$(\ell,\cP,\cF,\decisionset')$ for any $\decisionset' \supset
\decisionset$, in particular for $\decisionset' = \cF$.

\begin{proof}
  For Part (a), from the $\eta$-central condition up to $\epsilon$ and the fact that
  the $\sup \inf$ never exceeds the $\inf \sup$ and that $\model = \decisionset$, we get
  \begin{equation}\label{eq:above}
    e^{\eta \epsilon} \geq \sup_{P \in \cP} \lowinf_{f \in \decisionset}  \sup_{\Pi \in
      \Delta(\cF)} S_\Pi^\eta(P,f)
      \geq \sup_{P \in \cP} \sup_{\Pi \in
      \Delta(\cF)} \lowinf_{f \in \decisionset} S_\Pi^\eta(P,f)
      = \sup_{\Pi \in \Delta(\cF)} \sup_{P \in \cP} \lowinf_{f \in \decisionset}
      S_\Pi^\eta(P,f).
  \end{equation}
This establishes that the
  premise of \eqref{eqn:minimaxImplication} holds with $C =
  e^{\eta\epsilon}$ for all $\Pi \in \Delta(\cF)$. Hence
  Assumption~\ref{ass:minimax} tells us that the conclusion of
  \eqref{eqn:minimaxImplication} must also hold for all $\Pi \in
  \Delta(\cF)$, and therefore
  \begin{equation*}
    \sup_{\Pi \in \Delta(\cF)} \lowinf_{f \in \model}\sup_{P \in \cP}
    S_\Pi^\eta(P,f) \leq e^{\eta \epsilon}.
  \end{equation*}
  Since we are not guaranteed that the infimum over $f$ is achieved,
  this implies the $\eta$-predictor condition up to any $\delta >
  \epsilon$, but not necessarily for $\delta = \epsilon$. We thus obtain
  the first part of the theorem. 

  For Part (b), we note that, by the premise,
  Assumption~\ref{ass:simple} must hold and we can apply
  Corollary~\ref{cor:teeth} which tells us that for all $P \in \cP$,
  the $f^*_P \in \cF$ minimizing $R(P,f)$ is essentially unique and that the strong $\eta$-central condition holds, i.e. for all $P \in \cP$, (\ref{eqn:basiccomparatorconditionpre}) holds. As explained below (\ref{eqn:basiccomparatorconditionpre}), this implies that $f'_P = \comp(P)$ is $\cF$-optimal for $P$, hence it follows that $f'_P = f^*_P$, $P$-almost surely.  The
  strong version of Assumption~\ref{ass:preciseBayes} then implies
  that $\decisionset$ contains a $g^*_P$ with $P(\ell_{f^*_P} =
  \ell_{g^*_P}) = 1$. We now have, by the strong $\eta$-central condition, that for all $\Pi \in \Delta(\cF)$,
\begin{align*}
1 & \geq \sup_{P \in \cP} \lowinf_{f \in \model}  \sup_{\Pi \in
      \Delta(\cF)} S_\Pi^\eta(P,f) = \sup_{P \in \cP}  \sup_{\Pi \in
      \Delta(\cF)}  S_\Pi^\eta(P,f'_P) = 
\sup_{P \in \cP}  \sup_{\Pi \in
      \Delta(\cF)}  S_\Pi^\eta(P,g^*_P) \\ & \geq \sup_{P \in \cP} \lowinf_{f \in \decisionset}  \sup_{\Pi \in
      \Delta(\cF)} S_\Pi^\eta(P,f).
\end{align*}
We have thus established the first inequality of (\ref{eq:above}) with
$\epsilon = 0$; we can  now proceed as in the first part. 
\end{proof}

We proceed to identify more concrete conditions that are sufficient
for Assumption~\ref{ass:minimax}. To this end, we will endow the set of
finite measures (including all probability measures) on $\cZ$ with the
\emph{weak topology} \citep{Billingsley1968,vandervaart1996weak}, for which
convergence of a sequence of measures $P_1,P_2,\ldots$ to $P$ means that
\begin{equation}\label{eqn:weakconvergence}
  \E_{Z \sim P_n}[h(Z)] \to \E_{Z \sim P}[h(Z)]
\end{equation}
for any bounded, continuous function $h \colon \cZ \to \reals$. To make
continuity of $h$ well-defined, we then also need to assume a topology
on $\cZ$. It is standard to assume that $\cZ$ is a \emph{Polish space}
(i.e.\ that it is a complete separable metric space), because then, from \citet{Prokhorov1956},
there exists a metric for which the set of finite measures on $\cZ$ is a Polish space as well 
and for which convergence in this metric is equivalent to \eqref{eqn:weakconvergence}. 
The weak topology is the topology induced by this metric.

We shall also assume that $\cP$ is \emph{tight}, which means that,
for any $\epsilon > 0$, there must exist a compact event $A \subseteq
\cZ$ such that $P(A) \geq 1-\epsilon$ for all $P \in \cP$. This is a
weaker condition than assuming that the whole space $\cZ$ is compact
because it allows some probability mass outside of the compact event
$A$.

\begin{assumption}\label{ass:convexityContinuity}
  Suppose the set of possible outcomes $\cZ$ is a Polish space. Let
  $\dpfour$, 
  $\Pi \in \Delta(\model)$ and $\eta > 0$ be given. Then assume 
  that all of the following are satisfied:
  \begin{enumerate}
    \item \label{it:continuity} For all $f \in \cF \union \decisionset$, $\loss_f(z)$ is continuous in $z$ and
    $\loss_f(z) \geq 0$.
    \item \label{it:fconvexity} The set $\decisionset$ is convex and, for any $z \in \cZ$,
    $e^{\eta \loss_f(z)}$ is convex in $f$ on $\decisionset$.
    \item The set $\cP$ is convex and tight.
    \item \label{it:equicont} Either a) $\cP$ is closed in the weak topology; or b) $\decisionset$ is a totally
      bounded metric space, and, for every compact subset $\cZ'$ of $\cZ$, the family of functions $\{
      f \mapsto \loss_f(z) :  z \in \cZ'\}$ is uniformly
      equicontinuous on $\decisionset$.
    \item \label{it:uniformintegrability}
    The random variables
    $\xi_{Z,f} = \E_{g \sim \Pi}\left[e^{\eta\left(\loss_f(Z) -
    \loss_g(Z)\right)}\right]$ are \emph{uniformly integrable} over $f \in
    \decisionset, P \in \cP$ in the sense that
    \begin{equation}\label{eqn:uniformintegrability}
      \lim_{b \to \infty}  \sup_{f \in \decisionset, P \in \cP} \E_{Z
      \sim P} \left[\xi_{Z,f} \ind{\xi_{Z,f} \geq b} \right] = 0.
    \end{equation}
  \end{enumerate}
\end{assumption}
While these assumptions may look daunting, they actually hold in many
situations even with unbounded losses, as our examples below
illustrate. In \ref{ass:convexityContinuity}.\ref{it:continuity},
continuity is automatic for finite and countable $\cZ$ as long as we
take the discrete topology.  In
\ref{ass:convexityContinuity}.\ref{it:fconvexity}, convexity of
$e^{\eta \loss_f(z)}$ in $f$ is implied by convexity of $\loss_f(z)$
in $f$. Regarding the fourth requirement,
\ref{ass:convexityContinuity}.\ref{it:equicont}: the condition that
$\cP$ is weakly closed is easily stated but hard to verify for
general $\cZ$ and $\cP$; the alternative condition is hard to state
but often straightforward to verify. And finally,
\ref{ass:convexityContinuity}.\ref{it:uniformintegrability} will
automatically hold for all bounded loss functions and for many
unbounded losses as well; for a discussion of uniform integrability
as used in
\ref{ass:convexityContinuity}.\ref{it:uniformintegrability}, 
see \citet[pp.\,188--190]{shiryaev1996probability}. In particular,
Lemma~3 on p.\,190, specialised to our context, implies the
following sufficient condition:
\begin{lemma}[Sufficient Condition for
\ref{ass:convexityContinuity}.\ref{it:uniformintegrability}]
  \label{lem:sufficientUniformIntegrability}
  For a fixed choice of $\Pi \in \Delta(\model)$, let $\xi_{Z,f}$ be as
  in
  Assumption~\ref{ass:convexityContinuity}.\ref{it:uniformintegrability}.
  Then \eqref{eqn:uniformintegrability} is satisfied if 
  \begin{equation*}
      \sup_{f \in \decisionset} \sup_{P \in \cP} \ \E_{Z \sim P}
      \left[G(\xi_{Z,f})\right]  <
      \infty
  \end{equation*}
  for any function $G \colon \nonnegreals \to \reals$ that is bounded below
  and is such that
  \begin{equation}\label{eqn:Gconditions}
    \text{$\frac{G(t)}{t}$ is increasing, \quad and \quad
      $\frac{G(t)}{t} \to \infty$.}
  \end{equation}
\end{lemma}
We may, for instance, take $G(t) = t^2$ or $G(t) = t \log t$.

\begin{proof}
  Without loss of generality, we may assume that $G$ is non-negative.
  Otherwise replace $G(t)$ by $\max\{G(t),0\}$, which preserves
  \eqref{eqn:Gconditions} and adds at most $-\inf_t G(t) < \infty$ to
  $\sup_{f \in \decisionset}  \sup_{P \in \cP} \; \E_{Z \sim P}
  \left[G(\xi_{Z,f})\right]$.
  
  Now let $M = \sup_{f \in \decisionset} \sup_{P \in \cP}\; \E_{Z \sim P}
  \left[G(\xi_{Z,f})\right]$ and, for any $\epsilon > 0$, take $b > 0$
  large enough that $G(t)/t \geq M/\epsilon$ for all $t \geq b$. Then
  \begin{align*}
    0 \leq \sup_{f \in \decisionset} \sup_{P \in \cP}\/ \E_{Z
      \sim P} \left[\xi_{Z,f}\ind{\xi_{Z,f} \geq b}\right]
    &\leq \frac{\epsilon}{M} \sup_{f \in \decisionset} \sup_{P \in \cP}\/ \E_{Z
      \sim P} \left[G(\xi_{Z,f})\ind{\xi_{Z,f} \geq b}\right]\\
    &\leq \frac{\epsilon}{M} \sup_{f \in \decisionset} \sup_{P \in \cP}\/ \E_{Z
      \sim P} \left[G(\xi_{Z,f})\right]
    \leq \epsilon,
  \end{align*}
  from which \eqref{eqn:uniformintegrability} follows by letting
  $\epsilon$ tend to $0$.
\end{proof}

Assumption~\ref{ass:convexityContinuity} is sufficient for the minimax
assumption, as our main technical result of this section (proof
deferred to \cref{app:minimaxproof}) shows:
\begin{lemma}\label{lem:convexityContinuityToMinimax}
  Fix $\dpfour$ and $\eta > 0$.
  If Assumption~\ref{ass:convexityContinuity} is satisfied for a given
  $\Pi \in \Delta(\cF)$, then \eqref{eqn:minimaxEquality} also holds.
  Consequently, if Assumption~\ref{ass:convexityContinuity} is satisfied
  for all $\Pi \in \Delta(\cF)$, then that implies
  Assumption~\ref{ass:minimax}.
\end{lemma}

Together, Theorem~\ref{thm:minimaxSufficient} and
Lemma~\ref{lem:convexityContinuityToMinimax} prove the following
theorem. 
\begin{theorem}\textnormal{({\bfseries Central to Predictor})}\label{thm:secondmain}
  Let $\eta > 0$ and suppose Assumption~\ref{ass:convexityContinuity}
  holds for $\dpfour$ for all $\Pi \in \Delta(\model)$. If either $\cF
  = \decisionset$ or the strong version
  of Assumption~\ref{ass:preciseBayes} holds and $\cF \supset \decisionset$, then the weak
  $\eta$-central condition implies the weak $\eta$-predictor
  condition.
\end{theorem}
We now provide some examples which indicate that while Assumption~\ref{ass:convexityContinuity}
covers several nontrivial cases --- including non-compact $\cF$ --- it
is probably still significantly more restrictive than needed.
\begin{myexample}[Logarithmic Loss]\label{ex:loglossassD}
  Consider a  set of distributions $\cP$ on some set $\cZ$ and
  let $\cF$ either be the densities or mass functions corresponding to $\cP$
  or an arbitrary convex set of densities on $\cZ$. By Example~\ref{ex:log-loss},
  $(\logloss,\cP,\cF)$ satisfies the $1$-central condition. If we
  further assume that $\cP$ is convex and tight and that there is a $\delta > 0$
  such that for all $z \in \cZ$, all $f \in \cF$, $f(z) \geq \delta$
  (so that the densities are bounded from below), then
  Assumption~\ref{ass:convexityContinuity} is readily verified and we
  can conclude from the theorem that the $1$-predictor condition and
  hence $1$-stochastic mixability holds for $\dpfoursame$. We know
  however, because log-loss is $1$-(Vovk-) mixable, that
  $1$-stochastic mixability must even hold if $\cP$ is neither convex
  nor tight; Assumption~\ref{ass:convexityContinuity}
  is not weak enough to handle this case, so the example suggests
  that a further weakening might be possible. Also, we know that
  $1$-stochastic mixability continues to hold if $\delta = 0$;
  verification of Assumption~\ref{ass:convexityContinuity} is not
  straightforward in this case, which suggests that a simplification
  of the assumption is desirable. 
\end{myexample}
\begin{myexample}[$0/1$-Loss ,
  Example~\ref{ex:prototsybakov}, Continued.] \label{ex:protototo}
  Consider the setting of Example~\ref{ex:prototsybakov} and
  Example~\ref{ex:protototsybakov} with  decision problem
  $(\zoloss,\cP_{\delta}, \cF)$ and $\delta > 0$. We established in
  Example~\ref{ex:prototsybakov} that the $\eta$-central
  condition then holds for some $\eta > 0$, but also, in
  Example~\ref{ex:protototsybakov}, that $(\zoloss,\cP_{\delta},
  \cF,\cF)$ is not $\eta$-stochastically mixable. We would thus expect
  Assumption~\ref{ass:convexityContinuity} to fail here, which it does, 
  since $\cF = \decisionset$ is not convex. 
\end{myexample}
\begin{myexample}[Squared Loss, Restricted Domain]\label{ex:bounded-sq-pc}
  Let  $\loss$  be the squared loss $\sqloss_f(z) :=
  \frac{1}{2}(z-f)^2$ on the restricted spaces $\cZ = \model =
  \decisionset = [-B,B]$ as in Example~\ref{ex:mixable},  and  take $\cP$ to be the set of all
  possible distributions on $\cZ$.
  Then the first three requirements of
  Assumption~\ref{ass:convexityContinuity} may be verified by observing
  that $\sqloss_f(z)$ (and therefore also $e^{\eta \sqloss_f(z)}$) is
  convex in $f$, and that $\cP$ is trivially tight by taking $A = \cZ$.
  Now $\cP$ is actually closed in the weak topology, but, in order to
  satisfy the fourth condition, we might also use that the mappings $\{
  f \mapsto \sqloss_f(z) : z \in \cZ \}$ are all Lipschitz with the
  same Lipschitz constant ($2B$), which implies that they are also
  uniformly equicontinuous. Finally, to see that the fifth requirement is 
  satisfied for any $\Pi \in \Delta(\model)$, we may appeal to
  Lemma~\ref{lem:sufficientUniformIntegrability} with $G(t) = t^2$ and
  use that $\sqloss$ is uniformly bounded.
  
  Then all parts of Assumption~\ref{ass:convexityContinuity} are
  satisfied for all $\Pi \in \Delta(\model)$. We know from
  Example~\ref{ex:mixable} that in this case classical
  $\eta$-mixability holds for $\eta = 1/B^2$. This implies strong
  $\eta$-stochastic mixability, which implies the strong
  $\eta$-pseudoprobability convexity condition (by
  Proposition~\ref{prop:fromsmtoppcc}). Since
  Assumption~\ref{ass:simple} holds, this in turn implies the strong
  $\eta$-central condition (by
  Theorem~\ref{thm:convexfacetocomparator}), and by applying
  Theorem~\ref{thm:secondmain} one can then infer the weak
  $\eta$-predictor condition.
\end{myexample} In the example above, the set $\cP$ was convex and,
by boundedness of $\cZ$, automatically tight and thus the
$\eta$-central condition and $\eta$-stochastic mixability both hold. In
Example~\ref{ex:unbounded-sq-cc} we established the $\eta$-central
condition for a set $\cP$ that is neither convex nor tight, so
Assumption~\ref{ass:convexityContinuity} fails and we cannot apply
Theorem~\ref{thm:secondmain} to jump from the $\eta$-central to the
$\eta$-predictor condition as in
Example~\ref{ex:bounded-sq-pc}. However, as the next example shows,
if we replace $\cP$ by its convex hull for a restricted range of
$\mu$, then we can recover the predictor condition via
Theorem~\ref{thm:secondmain} after all; restriction of $\cF$, however, is not needed. 
\begin{myexample}[Squared Loss, Unrestricted Domain: Example~\ref{ex:unbounded-sq-cc}, Continued.]\label{ex:unbounded-sq-pc}
Consider the squared loss $\sqloss_f(z) = \half(z-f)^2$, and  let
$\cZ = \reals$, $\cF= [-B,B]$ (later we will consider $\cF = \reals$), and let $\cP =
\convhull(\{\normaldist(\mu,1) : \mu \in [-M,M]\})$ be the convex
hull of the set of normal distributions with unit variance and means
bounded by $M \leq B$. We may represent any $P \in \cP$ as a mixture of
$\normaldist(\mu,1)$ under some distribution $w$ on $\mu$. Let $\mu_P$
be the mean of $P$. Then, for all $P \in \cP$ with corresponding $w$ and
all $t \in \reals$, 
\begin{equation*}
  \E_{Z\sim P}\left[e^{t(Z-\mu_P)}\right]
    = \int_{-M}^M \E_{Z\sim
    \normaldist(\mu,1)}\left[e^{t(Z-\mu_P)}\right] \der w(\mu)
  = e^{t^2/2} \int_{-M}^M e^{t(\mu-\mu_P)} \der w(\mu)
  \leq e^{t^2/2} e^{t^2 M^2/2},
\end{equation*}
where the last inequality follows from Hoeffding's bound on the moment
generating function and the observation that $\mu_P = \E_{\mu \sim w}[\mu]$. Thus
the elements of $\cP$ are all subgaussian with variance $\sigma^2 = 1 +
M^2$. Hence, by the argument in Example~\ref{ex:dunno}, the strong
$\eta$-central condition is satisfied for $\eta \leq 1/(1+M^2)$ and with substitution function $\comp(P) = \mu_P$.

In order to also get the predictor condition via Theorem~\ref{thm:secondmain}, we need to
verify Assumption~\ref{ass:convexityContinuity}. The first three parts
of this assumption may be readily verified, and part b) of
\ref{ass:convexityContinuity}.\ref{it:equicont} also holds, because the
mappings $\{f \mapsto \half (z-f)^2 : z \in [-A,A]\}$ are all
$(2A)$-Lipschitz, which implies their uniform
equicontinuity, for any choice of $A$.
Finally,
Assumption~\ref{ass:convexityContinuity}.\ref{it:uniformintegrability}
follows from Lemma~\ref{lem:sufficientUniformIntegrability} with $G(t) =
t^2$ and Jensen's inequality:
\begin{align*}
\sup_{f \in \model} &\sup_{P \in \cP} 
  \E_{Z \sim P} \left[\E_{g \sim \Pi} [e^{\eta (\sqloss_f(Z) -
  \sqloss_g(Z))}]\right]^2
 \leq \sup_{f \in \model} \sup_{P \in \cP} 
  \E_{Z \sim P} \E_{g \sim \Pi} \left[e^{2\eta (\sqloss_f(Z) -
  \sqloss_g(Z))}\right]\\
 &\leq \sup_{f,g \in \model} \sup_{P \in \cP} 
  \E_{Z \sim P} \left[e^{2\eta (\sqloss_f(Z) -
  \sqloss_g(Z))}\right]
 = \sup_{f,g \in \model} \sup_{P \in \cP} 
  \E_{Z \sim P} \left[e^{2\eta (f^2 + 2Z(g-f) - g^2)}\right]\\
 &\leq e^{2\eta B^2} \sup_{f,g \in \model} \sup_{P \in \cP} 
  \E_{Z \sim P} \left[e^{4\eta Z(g-f)}\right]
 \overset{(*)}{\leq} e^{2\eta B^2}  \sup_{f,g \in \model} \sup_{P \in \cP} 
e^{8 \eta^2 (g-f)^2 (1+M^2) +
 4\eta(g-f) \mu_P} < \infty,
\end{align*}
where $(*)$ follows from $(1 + M^2)$-subgaussianity.
Thus, Theorem~\ref{thm:secondmain} can be applied to establish the weak
$\eta$-predictor condition for squared loss on an unbounded domain $\cZ
= \reals$ for the choices of $\eta$, $\decisionset = \model$ and $\cP$
described above.

Now consider the case where we set $\cF = \reals = \cZ$ and leave
everything else unchanged. Then by the argument in Example~\ref{ex:dunno}, the strong
$\eta$-central condition is still satisfied for $\eta \leq 1/(1+M^2)$,
but we cannot directly use Theorem~\ref{thm:secondmain} to establish the weak
predictor condition for $\dpfoursame$. All steps of the above
reasoning go through except part b) of
\ref{ass:convexityContinuity}.\ref{it:equicont}, since $\cF$ is no longer compact.
However, if we take $\decisionset = [-B,B]$ for $B\geq M$,
then  Assumption~\ref{ass:convexityContinuity}.\ref{it:equicont} (which only refers to
$\decisionset$, not to $\cF$) holds after all. Moreover, the strong
version of Assumption~\ref{ass:preciseBayes} also holds, because
$\argmin_{f \in \reals} \E_{Z \sim P}(Z-f)^2 = \mu_P$.
We can thus use
Theorem~\ref{thm:secondmain} to conclude that $\dpfour$ satisfies the
weak $\eta$-predictor condition. 
It then follows by Fact~\ref{fact:smpc}
that $\dpfoursame$ satisfies the weak $\eta$-predictor condition as
well. We conclude that 
the implication $\eta$-central $\Rightarrow$ weak $\eta$-predictor goes
through, even though $\cF$ is not compact.  
\end{myexample}
This final example shows how Theorem~\ref{thm:secondmain} allows us to
find assumptions on $\cP$ that are sufficient for establishing the weak
predictor condition, and therefore weak stochastic mixability, for
squared loss on the unbounded domain $\reals$. As discussed by
\citet[Section~5]{vovk2001competitive}, this is a case where the
classical mixability analysis does not apply.

\section{Intermediate Rates: The Central Condition, the Margin Condition and the Bernstein condition}
\label{sec:comp-mix-margin}
In this section, we weaken the $\eta$-central and $\eta$-PPC
conditions to the $\gComparator$-central and $\gComparator$-PPC
conditions, which allow $\eta = v(\epsilon)$ to depend on $\epsilon$
according to a function $v$ that is allowed to go to $0$ as $\epsilon$ goes
to $0$.
 In the main result of
this section, Theorem~\ref{thm:BernsteinComparator} in
Section~\ref{sec:v-conditions}, we establish that
for bounded loss functions, these weakened versions of our conditions
are essentially equivalent to a generalized {\em Bernstein
  condition\/} which has been used before to
  characterize
  fast rates. Section~\ref{sec:comparator-vs-bernstein} shows that,
  for unbounded loss functions, the one-sidedness of our conditions
  allows them to capture situations in which fast rates are attainable
  yet the Bernstein condition does not hold --- although there are
  also situations in which the Bernstein condition holds whereas the
  $v$-central condition does not for any allowed $v$ (although the
  $v$-PPC condition does). Thus, as a corollary we find that the
  equivalence between the central and PPC condition breaks for the
  weaker, $v$-versions of these conditions.
  Section~\ref{sec:JRTother} illustrates that $\eta$-stochastic
  mixability can be weakened similarly to $\gComparator$-stochastic
  mixability and relates this to a condition identified by
  \cite{juditsky2008learning}. Finally, in Section~\ref{sec:nonunique}
  we apply Theorem~\ref{thm:BernsteinComparator} to show how the
  central condition is related to (non-) existence of unique risk
  minimizers.

\subsection{The $v$-Conditions and the Bernstein Condition}
\label{sec:v-conditions}

Empirical risk minimization (ERM) achieves fast rates if the random deviations
of the empirical excess risk are small compared to the true excess risk.
As shown by \citet{tsybakov2004optimal}, this is the
case in classification if the Bayes-optimal classifier is in the model
$\model$ and the so-called \emph{margin}, which measures the difference
between the conditional probabilities of the labels given the features
and the uniform distribution, is large. Technically, the random
deviations can be controlled in this case, because the second moment of
the excess loss can be bounded in terms of the first moment. In fact, as
shown by \citet{bartlett2006empirical}, this condition, which they call
the \emph{Bernstein condition}, is sufficient for fast rates for bounded
losses in general, even if the Bayes-optimal decision is not in the
model. Precisely, the standard Bernstein condition is defined as
follows:
\begin{definition}[Bernstein Condition]\label{def:bernstein}
Let $\beta \in (0, 1]$ and $B \geq 1$. Then $(\loss, P, \cF)$ satisfies the
\emph{$(\beta, B)$-Bernstein condition} if there exists an $f^* \in
\model$ such that
\begin{align}\label{eqn:bernstein}
\E_{Z \sim P} \left[ \bigl( \loss_f(Z) - \loss_{f^*}(Z) \bigr)^2 \right]
\leq B \left( \E_{Z \sim P} \left[ \loss_f(Z) - \loss_{f^*}(Z) \right] \right)^\beta
\qquad \text{for all $f \in \model$.}
\end{align}
\end{definition}
This standard definition bounds the second moment in terms of the
polynomial function $\gBernstein(x) = Bx^\beta$ of the first moment.\footnote{The Tsybakov condition with exponent $q$
\citep{tsybakov2004optimal} is the  special case that the
$(\beta,B)$-Bernstein condition holds for $B< \infty$, $q =
\beta/(1-\beta)$, 
additionally requiring $\ell$ to be classification
loss and $\cF$ to contain the Bayes classifier for $P$.} The exponent
$\beta$ is most important, because it determines the order of the
rates, whereas the scaling factor $B$ only matters for the
constants. To draw the connection with the central condition,
however, it will be clearer to allow general functions $\gBernstein$ instead of $x \mapsto B x^\beta$.
Following \citet{koltchinskii2006local} and
\citet{ArlotBartlett2011}, we then bound the variance instead of the
second moment, which is equivalent with respect to the rates that
can be obtained:
\begin{definition}[Generalized Bernstein Condition]\label{def:gen-bernstein}
  Let $\gBernstein: [0,\infty) \to [0,\infty)$ be a non-decreasing
  function such that $\gBernstein(x) > 0$ for all $x > 0$, and
  $\gBernstein(x)/x$ is non-increasing. We say that $\dpthree$
  satisfies the $\gBernstein$-\emph{Bernstein condition} if, for all
  $P \in \cP$, there exists an $\cF$-optimal $f^* \in \cF$ (satisfying
  (\ref{eq:bayesact})) such that
\begin{align}\label{eqn:gen-bernstein}
  \Var_{Z \sim P}\big( \loss_f(Z) - \loss_{f^*}(Z)\big) \leq \gBernstein
  \left( \E_{Z \sim P} \left[ \loss_f(Z) - \loss_{f^*}(Z) \right]
  \right)\qquad
\text{for all $f \in \model$}.
\end{align}
\end{definition}
In particular $\gBernstein(x) = B x^\beta$ is allowed for $\beta \in
[0,1]$, or, more generally, it is sufficient if $\gBernstein(0) = 0$
and $\gBernstein$ is a non-decreasing concave function, because then
the slope $\gBernstein(x)/x = (\gBernstein(x) - \gBernstein(0))/x$ is
non-increasing; for a concrete example see Example~\ref{ex:tsybakov}
below.

Similar generalizations have been proposed by
\citet{koltchinskii2006local} and \citet{ArlotBartlett2011}\footnote{They require $\gBernstein$ to be of the form $w^2$
where $w$ is a concave increasing function with $w(0) = 0$.  In their
examples, $w^2$ is also concave, a case which is subsumed by our
condition, but they additionally allow concave $w$ with convex $\gBernstein =
w^2$, which is not covered by our condition. On the other hand, our
condition allows $\gBernstein$ with non-concave $\sqrt{\gBernstein}$,
which is not covered by theirs. For example, $\gBernstein(x) =
(x-1/3)^3+ 1/27$ for $x \leq 1/2$ and $\gBernstein(x) = x/12$ for $x > 1/2$
satisfies our condition, but $\sqrt{\gBernstein(x)}$ is nonconcave. So, in
general, the conditions are incomparable.}.
For bounded losses, our generalized Bernstein condition is equivalent
to a generalization of the central condition in which $\eta =
\gComparator(\epsilon)$ is allowed to depend on $\epsilon$ according
to some function $\gComparator$, which in turn is equivalent to the
analogous generalization of the pseudoprobability-convexity
condition. We first introduce these generalized concepts and then show
how they are related to the Bernstein condition. They are defined as
immediate generalizations of their corresponding definitions,
Definition~\ref{def:comparator}, Equation \eqref{eqn:comparator} and
Definition~\ref{def:convex-face}, Equation
\eqref{eqn:strong-convexface}:
\begin{definition}[$\gComparator$-Central Condition and
  $\gComparator$-PPC
  Condition]\label{def:g-stoch-mix}
Let $\gComparator \colon \nonnegreals \to \nonnegreals$ be a bounded,
non-decreasing function satisfying $\gComparator(x) > 0$ for all $x > 0$. We
say that $\dpthree$ satisfies the
\emph{$\gComparator$-central condition} if, for all $\epsilon \geq 0$,
there exists a function $\phi: \cP \rightarrow
\cF$ such that \eqref{eqn:comparator} is satisfied with $\eta =
\gComparator(\epsilon)$. We
say that $\dpthree$ satisfies the
\emph{$\gComparator$-pseudoprobability convexity (PPC) condition} if, for all $\epsilon \geq 0$,
there exists a  function $\pred: \cP \rightarrow
\cF$ such that \eqref{eqn:strong-convexface} is satisfied with $\eta =
\gComparator(\epsilon)$. 
\end{definition}
If $\gComparator(x) = \eta$ for all $x > 0$ and $\gComparator(0) = 0$, then
the $\gComparator$-central condition is equivalent to the
weak $\eta$-central condition. If $\gComparator(x) = \eta$ for all $x \geq 0$,
then it is equivalent to the strong $\eta$-central condition.

Now consider a decision problem $\dpthree$ such that
Assumption~\ref{ass:simple}
holds. Theorem~\ref{thm:BernsteinComparator} below in conjunction with Proposition~\ref{prop:comparatortoconvexface}  implies that the
generalized Bernstein condition with function $\gBernstein$, the
$\gComparator$-central condition and the $\gComparator$-PPC condition
are then all equivalent for bounded losses in the sense that one
implies the other if
\begin{equation}\label{eq:aad}
  \gComparator(x) \cdot \gBernstein(x) = c \cdot x \qquad \text{for all sufficiently small $x$,}
\end{equation}
where $c$ is a constant whose value depends on whether we are going from
Bernstein to central or the other way around. In particular, if we
ignore the unimportant difference between the second moment of
$\loss_f(Z) - \loss_{f^*}(Z)$ and its variance, we see that the
$(1,B)$-Bernstein condition and the $\eta$-central condition are
equivalent for $\eta = c/B$.

Define the function $\kappa(x) := (e^x - x - 1)/x^2$ for $x \neq 0$,
extended by continuity to $\kappa(0)= 1/2$, which is positive and
increasing \citep{Freedman1975}. 

\begin{theorem}\label{thm:BernsteinComparator}
  For given \dpthree, suppose that the losses
  $\loss_f$ take values in $[0,a]$.
\begin{enumerate} \item If the $\gBernstein$-Bernstein condition
  holds for a function $\gBernstein$ satisfying the requirements of
  Definition~\ref{def:gen-bernstein} (so that Assumption~\ref{ass:simple} holds), then
  \begin{enumerate}
  \item[(a)] The $\gComparator$-central
  condition holds for
  \begin{equation*}
    \gComparator(x) = \frac{c_1^b x}{\gBernstein(x)} \bmin b,
  \end{equation*}
  where $b > 0$ can be any finite constant and $c_1^b = 1/\kappa(2b a)$; and if $\gBernstein(0) = 0$ we
  read $0/\gBernstein(0)$ as $\liminf_{x \downarrow 0} x/\gBernstein(x)$. \\
  \item[(b)] Additionally, for each $P \in \cP$, any $\cF$-optimal $f^*$
  for $P$, and any $\delta > 0$, we have $\E_{Z \sim P} [ e^{\eta
    (\ell_{f^*}(Z)- \ell_{f}(Z))} ] \leq 1$ for all $f$ with $R(P,f) -
  R(P,f^*) \geq \delta$, where $\eta = \gComparator(\delta)$.
  \end{enumerate}
\item 
  On the other hand, suppose that Assumption~\ref{ass:simple} holds. If the $\gComparator$-pseudoprobability convexity  condition holds for a function
  $\gComparator$ satisfying the requirements of
  Definition~\ref{def:g-stoch-mix} such that $x/\gComparator(x)$ is nondecreasing,
then the $\gBernstein$-Bernstein  condition holds for
  \begin{equation*}
    \gBernstein(x) = \frac{c_2 x}{\gComparator(x)},
  \end{equation*}
  where $c_2 = 6/\kappa(-2b a)$ for $b = \sup_x \gComparator(x) <
  \infty$; and if $\gComparator(0) = 0$ we read $0/\gComparator(0)$ as $\lim_{x \downarrow 0} x/\gComparator(x)$.
\end{enumerate}
\end{theorem}
We are mainly interested in Part 1(a) of the theorem and its essential converse, Part 2. 
Part 1(b) is a by-product of the proof of 1(a) 
that will be useful for the proof of Proposition~\ref{prop:nonunique} below as well as the proof of the later-appearing Corollary \ref{cor:finite-intermediate-rates}.  Part 2 assumes that the $\gComparator$-PPC condition holds for
$\gComparator$ such that $\sup_{x \geq 0} \gComparator(x) <
\infty$. This boundedness requirement is without essential loss of
generality, since we already assume that losses are in $[0,a]$. From
the definition this trivially implies that, if the
$\gComparator$-condition holds at all, then also the
$\gComparator'$-condition holds for $\gComparator'(x) =
\gComparator(x) \wedge a'$, for any $a'\geq a$.

\begin{myexample}[Example~\ref{ex:classification}
  and~\ref{ex:prototsybakov}, Continued] \label{ex:tsybakov} Let
  $\ell$ be a bounded loss function and suppose that the
  $\gBernstein$-Bernstein condition holds with $\gBernstein(x) = B
  x^{\beta}$ for some $\beta \in [0,1]$. We first note that if $\beta=
  0$, then the condition holds trivially for large enough $B$.
  Theorem~\ref{thm:BernsteinComparator} shows that, in this case, we
  have the $\gComparator$-central condition for some $\gComparator$
  being linear in a neighborhood of $0$, in particular $\liminf_{x \downarrow 0}
  \gComparator(x)/x < \infty$. Thus, for bounded losses, the
  $\gComparator$-central condition always holds for such
  $\gComparator$. Thus we will say that the $\gComparator$-central
  condition holds {\em nontrivially\/} if it holds for $\gComparator$
  with $\liminf_{x \downarrow 0} \gComparator(x)/x = \infty$. Since the
  trivial $\gComparator$-condition always holds, it provides no information and therefore, under this condition, one can only prove (using Hoeffding's inequality) the standard slow rate of
  $O(1/\sqrt{n})$.
  The other extreme is when we have the
  $\eta$-central condition, i.e.~the $\gComparator$-condition holds
  with constant $\gComparator$, which as we show in
  \cref{thm:finite-fast-rates} leads to rates of order $O(1/n)$.
  Moreover, as we show in Corollary
  \ref{cor:finite-intermediate-rates}, it also is possible to recover
  intermediate rates under the general case of the
  $\gComparator$-central condition. Specifically, under the
  $\gComparator$-central condition, we get in-probability rates of
  $O\left( w(1/n) \right)$,
  where we recall that $w$ is the inverse of the function $x \mapsto x
  \gComparator(x)$. In the special case of $\gComparator: \epsilon
  \mapsto \epsilon^{1 - \beta}$ (for which the behavior in terms of
  $\epsilon$ corresponds to the $(\beta, B)$-Bernstein condition as
  shown by \cref{thm:BernsteinComparator}), we get the rate
  $O(n^{-1/(2 - \beta)})$, just as we do from the $(\beta,
  B)$-Bernstein condition.
\end{myexample}

The proof of Theorem~\ref{thm:BernsteinComparator} is deferred until
Appendix~\ref{app:comp-mix-marginproofs}. It is based on the following
lemma, which adds a (non-surprising) lower bound to a
well-known upper bound used e.g.\ by \citet{Freedman1975} in the
context of concentration inequalities. Since most authors only require
the upper bound, we have been unable to find a reference for the lower
bound, except for Lemma~C.4 in our own work
\citep{KoolenVanErvenGrunwald2014}. Interestingly, the Lemma is
applied in the proof of Theorem~\ref{thm:BernsteinComparator} with a `frequentist' expectation over $Z \in \cZ$ to prove the
first part, and a `Bayesian' expectation over $f \in \cF$ to prove the
second part.
\begin{lemma}\label{lem:MomentTaylor}
  For any random variable $X$ taking values in $[-a,a]$,
  \begin{equation}
    \kappa(-2a) \Var(X) \leq \E[X] + \log \E[e^{-X}] \leq \kappa(2a) \Var(X),
  \end{equation}
  where the function $\kappa$ is as defined above
  Theorem~\ref{thm:BernsteinComparator}.
\end{lemma}

\begin{proof}
Define the auxiliary function $\kappa'(x) = e^x - x - 1$. Then
\begin{equation*}
  \E[X] + \log \E[e^{-X}] = \min_{\mu \in [-a,a]} \E[\kappa'(\mu - X)],
\end{equation*}
as may be checked by observing that $\E[\kappa'(\mu - X)] = e^\mu \E[e^{-X}] -
\mu + \E[X] - 1$ is minimized at $\mu = -\log \E[e^{-X}]$. Since $\kappa'(x) =
\kappa(x) x^2$ and $\kappa(x)$ is increasing \citep{Freedman1975}, we further
have
\begin{equation}
  \E[\kappa'(\mu - X)]
    \begin{cases}
      \leq \max_{\mu',x \in [-a,a]} \kappa(\mu'-x) \E[(\mu - X)^2]
      = \kappa(2a) \E[(\mu - X)^2]\\
      \geq \min_{\mu',x \in [-a,a]} \kappa(\mu'-x) \E[(\mu - X)^2]
      = \kappa(-2a) \E[(\mu - X)^2],
    \end{cases}
\end{equation}
from which the lemma follows upon observing that $\min_{\mu \in [-a,a]}
\E[(\mu - X)^2] = \Var(X)$.
\end{proof}
\subsection{Bernstein vs.\ Central Condition for Unbounded Losses - Two-sided vs. One-sided Conditions}
\label{sec:comparator-vs-bernstein}
Applying Proposition~\ref{prop:comparatortoconvexface} with $\eta =
\gComparator(\epsilon)$ for all $\epsilon > 0$ immediately gives that,
under no further assumptions, the $\gComparator$-central condition
implies the $\gComparator$-pseudoprobability convexity condition.
Combined with Theorem~\ref{thm:BernsteinComparator} this shows that
the central condition and the Bernstein condition are essentially
equivalent for bounded losses, so it is natural to ask how the
$v$-versions of our conditions are related to the Bernstein conditions
for unbounded losses. In that case there are two essential
differences. One difference is that the variance or second moment in
the Bernstein condition is \emph{two-sided} in the sense that it is
large both if the excess loss $\loss_f(Z) - \loss_{f^*}(Z)$ gets
largely negative with significant probability, but also if the excess loss is
large, whereas the central condition is \emph{one-sided} in that large
excess losses only make it easier to satisfy. This difference is
illustrated by Example~\ref{ex:comparator-vs-bernstein} below, where
fast rates can be obtained and the central condition holds, but the
Bernstein condition fails to be satisfied. The second difference is
that the $v$-central condition essentially requires the probability
that $\loss_{f^*}(Z) - \loss_{f}(Z)$ is large is {\em exponentially\/}
small. Hence, if the loss is unbounded and has only polynomial tails,
then the $v$-central condition cannot hold. Yet
Example~\ref{ex:lastminute} shows that in such a case, the
$u$-Bernstein condition can very well hold for nontrivial $u$.
However, we should note that the $v$-PPC condition and the
$v$-stochastic mixability conditions (introduced in the next
subsection) also do not require exponential tails; hence it may still be
that whenever the $u$-Bernstein condition holds,
$v$-stochastic mixability also holds with $u(x) \cdot v(x) \asymp x$;
we do not know whether this is the case.
\begin{myexample}[Central without Bernstein for Unbounded
  Loss] \label{ex:comparator-vs-bernstein} Consider density estimation
  for the log loss. For $f_\mu$ the univariate normal density with
  mean $\mu$ and variance $1$, let $\cP$ be the normal location family
  and let $\cF = \{f_\mu : \mu \in \reals\}$ be the set of
  densities of the distributions in $\cP$. Then, for any $P \in \cP$
  with density $f_\nu$, the  risk $R(P,f)$ is minimized by $f^* =
  f_\nu$, since the model is well-specified.

Let $Z_1, \ldots, Z_n$ be an iid sample from $P \in \cP$. Then, 
as can be verified by direct calculation, the  empirical risk minimizer/maximum
likelihood estimator relative to $\cF$, $\hat{\gamma}_n := \frac{1}{n} 
\sum_{j=1}^n Z_j$, satisfies $\E_{Z_1, \ldots, Z_n \sim P} (\hat\gamma_n -  \nu)^2 = 
1/n$, which translates into an expected excess risk of
\begin{equation*}
  \E_{Z_1,\ldots,Z_n,Z \sim P}[-\log f_{\hat{\gamma}_n}(Z) +\log f^*(Z)] =
  \frac{1}{2n},
\end{equation*}
such that ERM obtains a fast rate in expectation. One would therefore
want a condition that aims to capture fast rates to be satisfied as well. 
For the central condition, this is the case with $\eta = 1$, as follows
from Example~\ref{ex:log-loss}. However, as we show next, the
$(1,B)$-Bernstein condition does not hold for any constant $B$.

Consider $P \in \cP$ with density $f_\nu$, and abbreviate $U_\mu(z) =
-\log f_\mu(z) + \log f_\nu(z) = \frac{\mu^2 - \nu^2}{2} + z (\nu - \mu)$.
Then
\begin{align*} 
\E_{Z \sim P} [ U_\mu(Z) ] &= \frac{\mu^2 + \nu^2}{2} - \mu \nu \\ 
\E_{Z \sim P} [ U_\mu^2(Z) ] &= (\nu - \mu)^2 \E_{Z \sim P} [Z^2] + 2 (\nu - \mu) \E_{Z \sim P} [Z] \frac{\mu^2 - \nu^2}{2} + \left( \frac{\mu^2 - \nu^2}{2} \right)^2 \\  
                     &= (\nu - \mu)^2 (1 + \nu^2) + (\nu - \mu) \nu (\mu^2 - \nu^2) + \left( \frac{\mu^2 - \nu^2}{2} \right)^2 . 
\end{align*} 
First consider the case that the `true' mean $\nu \geq 0$. Then for all constants $B$ the $(1,B)$-Bernstein condition fails to hold. To see this, first observe that for any $\mu$ satisfying $\mu \leq 0$ and $-\mu \geq \nu$, we have $\E_{Z \sim P} [U_\mu^2(Z)] \geq \left( \frac{\mu^2 - \nu^2}{2} \right)^2$ since $\nu - \mu \geq 0$ and $\nu \geq 0$. Second, observe that $\E_{Z \sim P} [U_\mu(Z)] \leq \mu^2 + \nu^2$ since $-\mu \nu \leq \frac{\mu^2 + \nu^2}{2}$. Hence, the following condition is weaker than the $(1,B)$-Bernstein condition: 
\begin{align*} 
(\mu^2 - \nu^2)^2 \leq 4 B (\mu^2 + \nu^2) . 
\end{align*} 
Choosing $\mu$ to satisfy $\nu \leq \frac{\mu^2}{2}$ leads to the even weaker condition  
$\left( \frac{\mu^2}{2} \right)^2 \leq 4 B (2 \mu^2)$  
which fails as soon as $|\mu| > \sqrt{32 B}$. It remains to show that the $(1,B)$-Bernstein also fails to hold for all $B$ if the true mean $\nu < 0$; this is shown using a symmetric argument by considering $\mu  > 0$ and $-\mu < \nu$. The result follows.
\end{myexample} 
Critically, the Bernstein condition cannot hold because of the
two-sided nature of the second moment, which is large, not just if some
$f_\mu$ is better than $f^*$ with significant probability, but also if
it is much worse. Thus, the fact that certain $f_\mu$ are so highly
suboptimal that they suffer high empirical
excess risk with high probability (and hence are easily avoided by ERM) ironically is what
causes the Bernstein condition to fail; a related point is made by
\cite{mendelson2014learning}.
The next example shows that, if $Z$ has two-sided, polynomial tails then the
opposite phenomenon can also occur: the $v$-central condition does not
hold for any $v$, but we do have the $u$-Bernstein condition for
constant $u$. 
\begin{myexample}\label{ex:lastminute}
  Let $\cP$ be an arbitrary collection of distributions over $\reals$
  such that for all $P \in \cP$, the mean $\mu_{P} := \E_{Z \sim P}[Z]
  \in [-1,1]$. Consider the squared loss $\sqloss_f(z) = \half(z-f)^2$,
  with $\model = [-1,1]$. Assume that $\cP$ contains a distribution
  $P^*$ with $\mu_{P^*} = 0$ and, for some constants $c_1, c_2 > 0$,
  for all $z \in \reals$ with $|z | > c_1$, the density $p^*$ of $P^*$
  satisfies $p^*(z) \geq c_2/z^6$. The predictor in $\cF$ that
  minimizes risk is given by $f^* = 0$.  Now with such a $\cP$, for
  all $\eta > 0$, all $\mu \neq 0$, and using that $\sqloss_{f^*}-
  \sqloss_\mu=  2 Z \mu -\mu^2$, we find for $c_3 = c_2 \cdot
  \exp(-\eta\mu^2)$, 
\begin{equation}
 \E_{Z\sim
   P}\left[e^{\eta\left(\sqloss_{f^*}(Z)-\sqloss_{{\mu}}(Z)\right)}\right] 
\geq \int_{c_1}^{\infty} \frac{c_3}{z^6} e^{\eta 2 z |\mu|} \mathrm{d}
z =   
\infty,
\end{equation}
so that the $v$-central condition fails for all $v$ of the form
required in Definition~\ref{def:g-stoch-mix}. Hence the $v$-central
condition does not hold --- although from Example~\ref{ex:jrtee}
below we see that $v$-stochastic mixability (and hence the $v$-PPC
condition) does hold for $v(x) \asymp \sqrt{x}$.

Now consider a $\cP$ with means in $[-1,1]$ and containing a $P^*$ as
above such that additionally for all $P
\in \cP$, the fourth moment is uniformly bounded, i.e.
  there is an $A > 0$ such that for all $P\in \cP$, $\E_{Z \sim P}[Z^4]
  < A$. Clearly we can construct such a $\cP$ and by the above it will
  not satisfy the $v$-central condition for any allowed $v$. However,
  the $u$-Bernstein condition holds with $u(x) = (4A^{1/2} +1) x$, since,
  using again  $
  \sqloss_\mu(Z) - \sqloss_{f^*}(Z) =  -2 Z \mu +\mu^2$, we find
\begin{multline}
\E_{Z \sim P^*} \left(\sqloss_\mu(Z) - \sqloss_{f^*}(Z) \right)^2 = \E\left[4 Z^2 \mu^2
  +\mu^4 - 4 Z \mu^3\right] \leq 4 \sqrt{A} \mu^2 + \mu^4 \leq
u(\mu^2) = \\
u\left( \E_{Z \sim P^*} \left(\sqloss_\mu(Z) - \sqloss_{f^*}(Z)
  \right)\right). \nonumber
\end{multline}
\end{myexample}
\subsection{$\gComparator$-Stochastic Mixability and the JRT Conditions}
\label{sec:JRTother}
Just as Definition~\ref{def:g-stoch-mix} weakened the $\eta$-central and PPC
conditions to the $\gComparator$-central and PPC conditions, 
we similarly may weaken the main conditions of
Section~\ref{sec:four-conditions}, stochastic mixability and its
special case stochastic exp-concavity, to their
$\gComparator$-versions:
\begin{definition}[$\gComparator$-Stochastic Mixability and $\gComparator$-Stochastic Exp-Concavity]\label{def:g-stoch-mixb}
Let $\gComparator \colon \allowbreak \nonnegreals \to \nonnegreals$ be a bounded,
non-decreasing function satisfying $\gComparator(x) > 0$ for all $x > 0$. We
say that $\dpfour$ is
\emph{$\gComparator$-stochastically mixable} if, for all $\epsilon \geq 0$,
there exists a function $\phi: \cP \rightarrow
\decisionset$ such that \eqref{eqn:mix} is satisfied with $\eta =
\gComparator(\epsilon)$. If $\decisionset \supseteq \convhull(\model)$
and this holds for the function  $\pred(\Pi)
= \E_{f \sim \Pi} [f]$ for all $\epsilon > 0$, then we say that
$\dpfour$ is {\em $\gComparator$-stochastically-exp-concave}.
\end{definition}
The main insight of Sections~\ref{sec:convex-face}
and~\ref{sec:four-conditions} was that the $\eta$-central condition, $\eta$-PPC condition and
$\eta$-stochastic mixability are all equivalent under some assumptions.
One may of course conjecture that the same holds for their weaker
$\gComparator$-versions. We shall defer discussion of this issue to
Section~\ref{sec:discussion} and for now focus on the usefulness of
$\gComparator$-stochastic exp-concavity, which can lead to intermediate rates
even for unbounded losses. 

A special case of $\gComparator$-stochastic exp-concavity, which we will call the JRT-I
condition, was stated by 
\cite{juditsky2008learning}; 
recall that we discussed the JRT-II condition in Section~\ref{sec:jrt}.
The JRT-I condition\footnote{The assumption is stated in basic form in their
  Theorem 4.1; their $Q_2$ is our $\ell^{(2)}$ and their $R$ is our
  $r_{\eta}$; the dependence of $r_{\eta}$ on $\eta$ (their $1/\beta$)
  is made explicit in their Corollary 5.1.}  states that, for every
$\eta > 0$, the excess loss can be decomposed as
\begin{equation*}
  \loss_f(z) - \loss_{f^*}(z) \geq \loss_\eta^{(2)}(z,f, f^*) - r_\eta(z)
  \qquad \text{for all $z$, any $f, f^* \in \convhull(\cF)$,}
\end{equation*}
where $r_\eta \colon \cZ \to \reals$ does not depend on $f,f^*$, and,
for any $f^* \in \convhull(\cF)$, $\ell_\eta^{(2)}(z,f^*,f^*) = 0$ and
$\loss_\eta^{(2)}(z,f,f^*)$ is $1$-exponentially concave as a function
of $f \in \convhull(\cF)$ (\ie, \eqref{eq:expconcave} holds with $\eta\loss_f(z) =
\loss_\eta^{(2)}(z,f,f^*)$). Note that the choice of $\loss_\eta^{(2)}$
and $r_\eta$ in general depends on $\eta$.
\citet{juditsky2008learning} show that, under this condition, fast rates
can be obtained in,
for example, regression problems with a finite number of regression
functions, where the rate depends on how $\epsilon_{\eta} := \sup_{P
  \in \cP} \E_{Z\sim P}\left[ r_{\eta}(Z)\right]$ varies with $\eta$.
 
We now connect the JRT-I assumption to $\gComparator$-stochastic exp-concavity.
Consider again the substitution function $\pred(\Pi) := \E_{g
  \sim \Pi} [g]$ as in Definition~\ref{def:expconcave}. Letting
$\bar{g} = \pred(\Pi)$, the JRT-I assumption implies that
\begin{align*}
 \E_{Z\sim P} &\left[\ell_{\E_{g \sim
  \Pi} [g]}(Z) + \frac{1}{\eta}  \log \E_{g \sim
  \Pi} e^{- \eta \ell_g(Z)} \right]  =
\E_{Z\sim P} \left[\frac{1}{\eta}  \log \E_{g \sim
  \Pi} e^{\eta \ell_{\bar{g}}(Z)- \eta \ell_g(Z)} \right] \\
&\leq \E_{Z\sim P} \left[\frac{1}{\eta}  \log \E_{g \sim
  \Pi} e^{- \eta \ell^{(2)}(Z,g, \bar{g}) + \eta r_{\eta}(Z)} \right]\\
& \overset{(a)}{\leq} 
\E_{Z\sim P} \left[\frac{1}{\eta}  \log  e^{- \eta \ell^{(2)}(Z,\bar{g}, \bar{g}) + \eta r_{\eta}(Z)} \right]
= \E_{Z\sim P}\left[ r_{\eta}(Z)\right] \leq \epsilon_{\eta},
\end{align*}
where (a) follows by the $\eta$-exp-concavity of $\ell^{(2)}$.
The
derivation shows that, if the JRT-I condition holds for each $\eta$ with function
$r_{\eta}(z)$ then we have $\eta$-stochastic exp-concavity up to
$\epsilon_{\eta} := \sup_{P\in \cP} \Exp_{Z \sim P} [r_{\eta}(Z)]$. In
their Theorem 4.1 they go on to show that, for finite $\cF$, by
applying the aggregating algorithm at learning rate $\eta$ and an
on-line to batch conversion, one can obtain rates of order $O(\log |
\cF|/ (n \eta) + \epsilon_{\eta})$, for each $\eta$. They go on to
calculate $\epsilon_{\eta}$ as function of $\eta$ in various examples
(regression, classification with surrogate loss functions, density
estimation) and, in each example, optimize $\eta$ as a function of $n$
so as to minimize the rate. Now for each function $\epsilon_{\eta}$ in
their examples, there is a
corresponding inverse function 
$\gComparator$ that maps $\epsilon$ to $\eta$ rather than vice versa,
so that if the JRT-I condition holds for $\epsilon_{\eta}$, then
$\gComparator$-stochastic exp-concavity holds. Rather than formalizing this in
general, we illustrate it informally using their regression example
\cite[Section 5.1]{juditsky2008learning}:
\begin{myexample}[JRT-I Condition and Regression]\label{ex:jrtee}
 JRT consider a
  regression problem in which $\cF$ is finite and 
  $\sup_{P \in \cP} \| f\|_{P,\infty} < \infty$ for all $f \in \cF$, where $\| \cdot \|_{P,\infty}$ denotes the
  $L_{\infty}(P_X)$-norm. They further assume that a weak moment
  assumption holds: for all $P \in \cP$, $\E_{(X,Y) \sim P} [|Y|^s] <
  \infty$ for some $s \geq 2$. They show that in this setting there
  exist constants $c_1, c_2, c_3, c_4 > 0$ such that for all $y \in
  \reals$, $r_{\eta}(y) \leq c_1 |y| \cdot \ind{|y| > c_2/\eta}
  + \eta c_3 y^2 \cdot \ind{|y| \geq c_4/\sqrt{\eta}}$.  Bounding
  expectations of the form $|y|^{a} \cdot \ind{|y| > b}$ in the same way as
  one bounds expectations of indicator variables $\ind{|y| > b}$ in
  the proof of Markov's inequality, this gives that
$$
\epsilon_{\eta} = O\left(\eta^{s/2}\right),
$$
which is strictly increasing in $\eta$. Thus, the inverse
$\gComparator(\epsilon)$ of $\epsilon_{\eta}$ is well-defined on
$\epsilon > 0$ and satisfies $\gComparator(\epsilon) =
O(\epsilon^{2/s})$. Since the JRT-I condition implies that, for all $\eta
> 0$, we have $\eta$-stochastic exp-concavity up to $\epsilon$ if $\epsilon \geq
\epsilon_{\eta}$, it follows that for all $\epsilon > 0$, we must have
$\eta$-stochastic exp-concavity up to $\epsilon$ for $\eta \leq
\gComparator(\epsilon)$. It follows that $\gComparator$-stochastic exp-concavity
holds with $\gComparator(\epsilon) = O(\epsilon^{2/s})$. 
In this unbounded loss case, we can easily obtain a rate by using the aggregating algorithm with online-to-batch conversion. 
Applying Proposition \ref{prop:stoch-mix-AA} with the optimal choice of $\epsilon$ yields a rate of 
$2 \left( \frac{\log |\cF|}{n} \right)^{-s/(s+2)}$, which coincides
with the rate obtained by \citet*{juditsky2008learning} in their
Corollary 5.2 and the minimax rate for this problem \citep{audibert2009fast}.
\end{myexample}

\subsection{The $\gComparator$-Central Condition and Existence of Unique Risk-Minimizers}
\label{sec:nonunique}
Corollary~\ref{cor:teeth} showed that, under Assumption~\ref{ass:simple}, strong $\eta$-fast rate (i.e. central and PPC)
conditions imply uniqueness of optimal $f^*$'s. Here we extend
this result, for bounded loss, to the $\gComparator$-fast rate
conditions, and also provide a converse, thus completely
characterizing uniqueness of $f^*$ in terms of the
$\gComparator$-central condition, for bounded losses. To understand
the proposition, note that for two predictors  with the
same risk, $R(P,f) = R(P,f^*)$, it holds that $f$ and $f^*$ achieve the same loss almost
surely, so they essentially coincide, if and only if $\Var_{Z \sim
  P}[\ell_f(Z) - \ell_{f^*}(Z)] = 0$. In the proposition we use
$\cF_{\epsilon} = \{ f^* \} \cup \{ f \in \cF: \Var_{Z \sim P}[\ell_f(Z)
- \ell_{f^*}(Z)] \geq \epsilon\}$ to denote the subset of $\cF$ where all
$f$'s that are very similar to, but not identical with, $f^*$
have been taken out.

\begin{proposition}{\ \bf ($\gComparator$-central condition and (non-)uniqueness of risk minimizers)\ }\label{prop:nonunique}
  Fix $(\ell, \{P \}, \cF)$ such that the loss $\ell$ is bounded and
  Assumption~\ref{ass:simple} holds, and let $f^*$ be an $\cF$-risk
  minimizer for $P$. Exactly one of the following two situations is the
    case:
  \begin{enumerate}
\item The $\gComparator$-central condition holds
    for some $\gComparator$ that is sublinear 
    at 0, i.e.\ $\lim_{x
      \downarrow 0} \gComparator(x)/x = \infty$. In this case, $f^*$
    is essentially unique, in the sense that for every sequence $f_1,
    f_2, \ldots \in \cF$ such that $\Exp_{Z \sim P}[\ell_{f_j}(Z)]
    \rightarrow \Exp_{Z \sim
      P}[\ell_{f^*}(Z)]$, we have $\Var_{Z \sim P} \left[\ell_{f_j}(Z) - \ell_{f^*}(Z)\right] \rightarrow
    0$. Moreover, for every $\epsilon > 0$, $(\ell,\{P\},
    \cF_{\epsilon})$ satisfies the $\eta$-central condition for some
    $\eta > 0$.
\item The $\gComparator$-central condition only holds trivially in the
  sense of Example~\ref{ex:tsybakov}, i.e.\ it does not hold for any
  $\gComparator$ with $\lim_{x \downarrow 0} \gComparator(x)/x =
  \infty$. In this case, $f^*$ is essentially {\em non\/}-unique, in the
  sense that there exists $\epsilon > 0$ and  a sequence $f_1, f_2,
  \ldots \in \cF$ (possibly identical for all large $j$) such that
    $\Exp_{Z \sim P}[\ell_{f_j}(Z)] \rightarrow  \Exp_{Z \sim
    P}[\ell_{f^*}(Z)]$, but, for all sufficiently large $j$, $\Var_{Z \sim P} \left[\ell_{f_j}(Z) - \ell_{f^*}(Z)\right] \geq \epsilon$. Moreover, for some $\epsilon > 0$,  $(\ell,\{P\}, \cF_{\epsilon})$ does not satisfy the $\eta$-central
    condition for any $\eta > 0$.
\end{enumerate}
\end{proposition}
\begin{proof} For Part 1,
  Proposition~\ref{prop:comparatortoconvexface} implies that the
  $\gComparator$-PPC condition holds. Now Part 2 of
  Theorem~\ref{thm:BernsteinComparator} implies that the
  $\gBernstein$-Bernstein condition holds with $\gBernstein$ such that
  $\lim_{x \downarrow 0} \gBernstein(x) = \lim_{x \downarrow 0}
  x/\gComparator(x) = 0$ by assumption.  Then it follows from the
  definition of the $\gBernstein$-Bernstein condition that $f^*$ is
  essentially unique. Moreover, by Part 1(b) of
  Theorem~\ref{thm:BernsteinComparator},
  there exists a function $\gComparator'$ with $\gComparator'(x) > 0$ for $x > 0$, such that
  for every $\delta > 0$, $(\ell, \{P \}, \{ f^* \} \cup \cG)$ satisfies
  the $\eta$-central condition with $\eta =
  \gComparator'(\delta) > 0$ for any subset $\cG \subseteq \{ f
  \in \cF: R(P,f) - R(P,f^*) \geq \delta \}$. Now since the
  $\gBernstein$-Bernstein condition holds with
  $\lim_{x \downarrow 0} \gBernstein(x) = 0$, we know that, for every $\epsilon > 0$,
  there is a $\delta > 0$ such that $\Var_{Z \sim P}[\ell_f(Z) - \ell_{f^*}(Z)] \geq \epsilon$ 
  implies $R(P,f) - R(P,f^*) > \delta$. For this 
  $\delta$, $\cG = \{f \in \cF: \Var_{Z \sim P}[\ell_f(Z) - \ell_{f^*}(Z)] \geq \epsilon\}$ 
  is a subset of $\{ f \in \cF: R(P,f) - R(P,f^*) \geq \delta \}$,
  and consequently, as already established, $(\ell, \{P\}, \{ f^* \} \cup \cG)$
  must satisfy the $\eta$-central condition
  for $ \eta > 0$, which is what we had to prove.

  For Part 2, to show nonuniqueness of $f^*$, note that by
  Theorem~\ref{thm:BernsteinComparator}, Part 1, the
  $\gBernstein$-Bernstein condition cannot hold for any
  $\gBernstein$ with $ \lim_{x \downarrow 0} \gBernstein(x) = 0$. This
  already shows that there exists a sequence as required, for some
  $\epsilon > 0$, so that $f^*$ is essentially non-unique.  Since
  $\Var_{Z \sim P}[\ell_{f_j}(Z) - \ell_{f^*}(Z) ] \geq \epsilon$ for all elements of
  the sequence and $R(P,f_j) \rightarrow R(P)$, the first inequality
  of Lemma~\ref{lem:MomentTaylor} applied with $X = \eta
  (\ell_{f_j}(Z) - \ell_{f^*}(Z))$ now gives that, for all $\eta > 0$,
  there exists $f_j$ such that $\log \E_{Z \sim P} e^{\eta (\ell_{f^*}(Z) -
    \ell_{f_j}(Z))} > 0$, so that the $\eta$-central condition does not
  hold.
\end{proof}

\section{From Fast Rates for Actions to Fast Rates for Functions}
\label{sec:xyz}

Let $\loss \colon \cA \times \cY \to \reals$ be a loss function, where
$\cY$ is a set of possible outcomes and $\cA$ is a set of possible
\emph{actions}. Then our abstract formulation in terms of $\dpfour$ can
accommodate \emph{unconditional} problems, where distributions $P
\in \cP$ are on $\cZ = \cY$ and both $\model$ and $\decisionset$ are
subsets of $\actions$; but it can also capture the \emph{conditional}
setting, where we observe additional \emph{features} from a covariate
space $\cX$. In that case, outcomes are pairs $(X,Y)$ from $\cZ' = \cX
\times \cY$, the model $\model'$ and decision set $\decisionset'$ are
both sets of functions $\{f \colon \cX \to \model\}$ from features to
actions, and the loss is commonly defined
in terms of the unconditional loss as $\loss'\big(f,(x,y)\big) =
\loss(f(x),y)$.

It may often be easier to establish properties like stochastic
mixability for the unconditional setting than for the conditional
setting. In this section we therefore consider when we can lift
conditions for unconditional problems with loss $\loss$ to the
conditional setting with loss $\loss'$. For the condition of being
$\eta$-stochastically mixable, this is done by
Proposition~\ref{prop:xyz} below. And, in Example~\ref{ex:misspecified},
it will be seen that, in some cases, this also allows us to obtain the
$\eta$-central condition for the conditional setting.

Proposition~\ref{prop:xyz} is based on the construction of a
substitution function $\pred' \colon \Delta(\model') \to \decisionset'$
for the conditional setting from the substitution function $\pred \colon
\Delta(\model) \to \decisionset$ for the unconditional setting. This
works by applying $\pred$ conditionally on every $x \in \cX$: first, note
that any distribution $\Pi$ on functions $f \in \model'$, induces, for
every $x \in \cX$, a distribution $\Pi_x$ on actions $\cA$ by drawing $f
\sim \Pi$ and then evaluating $f(x)$. We may therefore define
$\pred'(\Pi) = f_\Pi$ with $f_\Pi$ the function
\begin{equation}\label{eq:predx}
  f_\Pi(x) = \pred(\Pi_x).
\end{equation}
The conditions of the proposition then amount to the requirement that
this is a valid substitution function in the conditional setting.
\begin{proposition}\label{prop:xyz} 
  Let $\dpfour$ and $(\loss',\cP',\model',\decisionset')$ correspond to
  the unconditional and conditional settings described above, and assume
  all of the following:
  \begin{itemize}
    \item $\dpfour$ satisfies $\eta$-stochastic mixability up to $\epsilon$ with substitution function $\pred$;
    \item $P(Y|X) \in \cP$ for every $P \in \cP'$;
    \item the function $f_{\Pi}$
    from \eqref{eq:predx} is measurable and contained in
    $\decisionset'$, for every $\Pi \in \Delta(\model')$.
  \end{itemize}
  Then $\eta$-stochastic mixability up to $\epsilon$ is satisfied in the
  conditional setting. In particular, $f_\Pi$ is contained in
  $\decisionset'$ if:
  \begin{itemize}
    \item $\decisionset'$ is the set of \emph{all} measurable functions
    from $\cX$ to $\cA$; or
    \item $\dpfour$ is $\eta$-stochastically exp-concave up
    to $\epsilon$, and $\decisionset'$ contains the convex hull of
    $\model'$. In this case, $(\loss',\cP',\model',\decisionset')$ is
    also $\eta$-stochastically exp-concave up to $\epsilon$.
  \end{itemize}
\end{proposition}
We recall from Section~\ref{sec:expconcave} that $\eta$-stochastic exp-concavity is
the special case of $\eta$-stochastic mixability where the substitution
function maps $\Pi$ to its mean. In addition, for $\eta$-stochastic exp-concavity
the weak and strong versions of the condition coincide.

\begin{proof}
  We verify $\eta$-stochastic mixability up to $\epsilon$ for
  $(\loss',\cP',\model',\decisionset')$ by using $\eta$-stochastic
  mixability for $\dpfour$ conditional on each $x \in \cX$: for any $P \in
  \cP'$ and $\Pi \in \Delta(\model')$,
  \begin{align*}
    \E_{P(X,Y)}\left[\loss'_{\pred'(\Pi)}(X,Y) \right]
    &= \E_{P(X)} \E_{P(Y \mid X)}
    \left[\loss_{\pred(\Pi_X)}(Y) \right]\\
    &\leq \E_{P(X)} \E_{P(Y \mid X)}
    \left[-\tfrac{1}{\eta} \log \E_{\Pi_X(A)}\left[e^{-\eta
    \loss_A(Y)}\right]\right] + \epsilon\\
    &= \E_{P(X,Y)}
    \left[-\tfrac{1}{\eta} \log \E_{\Pi(f)}\left[e^{-\eta
    \loss'_f(X,Y)}\right]\right] + \epsilon,
  \end{align*}
  which was to be shown.

  Verifying that $f_\Pi \in \decisionset'$ is trivial if $\decisionset'$
  is the set of all measurable functions. And if $\dpfour$ is
  $\eta$-stochastically exp-concave up to $\epsilon$, then $f_\Pi(x) = \E_\Pi[f(x)]$
  for all $x$, and therefore $f_\Pi$ is the mean of $\Pi$ also in the
  conditional setting.
\end{proof}

The most important application is when $\cP$ contains all possible
distributions on $\cY$, which means that the unconditional problem is
classically mixable in the sense of Vovk (see
Section~\ref{sec:classical-mixability}). Then the requirement that
$P(Y \mid X) \in \cP$ is automatically satisfied.

\begin{myexample}[Squared Loss for Misspecified
Model]\label{ex:misspecified}
As discussed in Example~\ref{ex:exp-concavity}, the squared loss is
$\eta$-exp-concave in the unconditional setting on a bounded domain
$\decisionset \supseteq \model = \cZ = [-B,B]$, for $\eta = 1/4B^2$. If
we make the setting conditional by adding features, and consider any
set of regression functions $\model'$ and any set of joint
distributions $\cP'$, then Proposition~\ref{prop:xyz} implies that we
still have exp-concavity 
as long as we allow ourselves to make
decisions in the convex hull of $\model'$, i.e.\ if $\decisionset'
\supseteq \convhull(\model')$. Note that this holds even if the model
$\model$ is misspecified in that it does not contain the true
regression function $x \mapsto \E[Y\mid X=x]$.  If, furthermore,
the model $\model'$ is itself convex and satisfies
Assumption~\ref{ass:simple} relative to $\cP'$, i.e.  the minimum risk
$\min_{f \in \model'} \E_{(X,Y) \sim P}(Y-f(X))^2$ is achieved for
all $P \in \cP'$, then we may take $\decisionset' = \model'$ and
recover the setting considered by \citet{lee1998importance}. Even
though this does not require $\model'$ to be well-specified, the
strong version of Assumption~\ref{ass:preciseBayes} (which implies
Assumption~\ref{ass:simple}) is then still satisfied, and hence
Proposition~\ref{prop:fromsmtoppcc} and
Theorem~\ref{thm:convexfacetocomparator} tell us that
$(\loss',\cP,\model)$ satisfies both the strong
$\eta$-pseudoprobability convexity condition and the strong
$\eta$-central condition.
\end{myexample} The example raises the question whether we cannot
directly conclude, under appropriate conditions, that, if the
$\eta$-central condition holds for some unconditional $\dpthree$, then
it should also hold for the corresponding conditional
$(\loss',\cP',\model')$. We can indeed prove a trivial analogue of
Proposition~\ref{prop:xyz} for this case, as long as $\model'$
contains {\em all\/} measurable functions from $\cX$ to $\cY$; we
implicitly used this result in
Example~\ref{ex:regression}. Example~\ref{ex:misspecified}, however,
shows that, if one can first establish $\eta$-stochastic exp-concavity for
$\dpfour$, one can sometimes reach the stronger conclusion that
$(\loss',\cP',\model')$ satisfies the $\eta$-central condition as long
as $\model'$ is merely convex, rather than the set of all functions
from $\cX$ to $\cY$.

\section{The Central Condition Implies Fast Rates}
\label{sec:fast-rates}

In this section, we show how a statistical learning problem's satisfaction of the strong $\eta$-central condition implies fast rates of $O(1/n)$ under a bounded losses assumption.  
\cref{thm:finite-fast-rates} herein establishes via a rather direct argument that the strong $\eta$-central condition implies an exact oracle inequality (i.e.~with leading constant 1) with a fast rate for finite function classes, and \cref{thm:vc-type-fast-rates} extends this result to VC-type classes.  
We emphasize that the implication of fast rates from the strong 
$\eta$-central condition under a bounded losses assumption is not itself 
new. Specifically, for bounded losses, the central condition is 
essentially equivalent to the Bernstein condition by 
Theorem~\ref{thm:BernsteinComparator}, and therefore implies fast rates 
via existing fast rate results for the Bernstein condition. 
For instance, for finite classes Theorem 4.2 of \cite{zhang2006information} implies a fast $O(1/n)$ rate by letting $\loss_\theta$ be our excess loss $\loss_f - \loss_{f^*}$ assumed to satisfy the bounded loss condition therein, setting $\alpha = 0$, taking $\Pi$ to be the uniform prior over a finite class $\F$, and taking $\rho$ as $\frac{C}{K M}$ for some sufficiently small constant $C$. 
In addition, \cite{audibert2004pac} showed fast rates for classification under the Bernstein condition\footnote{Audibert actually introduces multiple conditions, referred to as variants of the margin condition, but these actually are closer to Bernstein-type conditions as they take into account the function class $\F$.}; see for example Theorem 3.4 of \cite{audibert2004pac} along with the discussion of how the variant of the (CA3) condition needed there is related to the (CA1) condition connected to VC-classes.  
However, since we posit the one-sided central condition rather than the two-sided Bernstein condition 
as our main condition, it is interesting to take a direct route based on 
the central condition itself, rather than proceeding via 
the Bernstein condition. As an added benefit, this approach turns out to 
give better constants and a better dependence on the upper bound on the loss. 
 
We proceed via the standard Cram\'er-Chernoff method, which also lies at 
the heart of many standard (and advanced) concentration inequalities 
\citep{boucheron2013concentration}. This method requires an upper bound 
on the cumulant generating function. We solve this subproblem by solving an 
optimization problem that is an instance of the general moment problem, a problem on which \cite{kemperman1968general} has conducted a detailed geometric study. 
This strategy leads to a fast rates bound for finite classes, which can be extended to parametric (VC-type) classes, as shown in \cref{sec:fast-rates-main-results}.

\subsection{The Strong Central Condition and ERM} 
\label{sec:strategy} 
 
For the remainder of Section~\ref{sec:fast-rates}, we will consider the 
conditional setting, where the loss $\loss_f(Z)$ takes values in the 
bounded range $[0,V$] for outcomes $Z = (X,Y) \in \cX \times \cY$ and 
functions $f$ from $\model = \{f \colon \cX \to \actions\}$. 
We take $\cP = \{P\}$ to be a single fixed distribution and we will 
assume throughout that 
$(\loss,\{P\},\model)$ satisfies the strong $\eta$-central 
condition for some $\eta > 0$. That is, there exists $f^* \in \model$ 
such that 
\begin{align}\label{eqn:fast-strong-central} 
\log \E_{Z \sim P} \exp(-\eta W_f) \leq 0 
\qquad \text{for all $f \in \model$,} 
\end{align} 
where we have abbreviated the excess loss by $W_f(Z) = \loss_f(Z) - 
\loss_{f^*}(Z)$; for brevity we further abbreviate $W_f(Z)$ to $W_f$ in this section. Then, by Jensen's inequality, $f^*$ is $\model$-optimal for $P$. We let $\eta^*$ denote the largest $\eta$ for which 
\eqref{eqn:fast-strong-central} holds.
 
An empirical measure $\Probn$ associated with an $n$-sample 
$\mathbf{Z}$, comprising $n$ \emph{independent, identically
  distributed} (iid) observations $(Z_1,\ldots, Z_n) = ((X_1, Y_1), \ldots, (X_n, Y_n))$, operates on functions as $\Probn f = \frac{1}{n} \sum_{j=1}^n f(X_j)$  
and on losses as $\Probn \loss_f = \frac{1}{n} \sum_{j=1}^n \loss_f(Z_j)$.  
 
\paragraph{\rm {\em Cram\'er-Chernoff.}}
 
We will bound the probability that the ERM estimator
\begin{equation}
  \erm := \argmin_{f \in \F} \frac{1}{n} \sum_{i=1}^n \loss_f(Z_i)
\end{equation}
selects a hypothesis with excess risk $R(P,f) - R(P,f^*) = \E[W_f]$ above
$\frac{a}{n}$ for some constant $a > 0$. For any real-valued random variable $X$,
let $\eta \mapsto \Lambda_X(\eta) = \log \E e^{\eta X}$ denote its
\emph{cumulant generating function} (CGF), which is known to be convex and satisfies $\Lambda'(0) = \E[X]$.
 
\begin{lemma}[Cram\'er-Chernoff] 
\label{lem:cramer-chernoff} 
For any $f \in \model$, $\eta > 0$ and $t \in \reals$,
\begin{equation}
  \Pr\left(
    \frac{1}{n} \sum_{j=1}^n \loss_f(Z_j) \leq \frac{1}{n} \sum_{j=1}^n \loss_{f^*}(Z_j)
      + t
    \right)
    \leq \exp\left(\eta n t + n \Lambda_{-W_f}(\eta)\right).
\end{equation}
\end{lemma}

\begin{proof}
Applying Markov's inequality to $e^{-\eta n\Probn W_f}$ and using the
fact that $\Lambda_{-n\Probn W_f}(\eta) = n \Lambda_{-W_f}(\eta)$ for
iid observations, yields
\begin{equation*}
  \Pr\left(
    -\Probn W_f > -t
    \right)
    \leq \exp\left(\eta n t + \Lambda_{-n \Probn W_f}(\eta)\right)
    = \exp\left(\eta n t + n \Lambda_{-W_f}(\eta)\right),
\end{equation*}
from which the lemma follows.
\end{proof}

\subsection{Semi-infinite Linear Programming and the General Moment Problem} 
\label{sec:general-moment-problem} 

We first consider the canonical case that $W_f$ takes values in $[-1,1]$
(\ie, $V = 1$), that $\Lambda_{-W_f}(\eta^*) = 0$ with equality (as opposed
to the inequality in Equation~\ref{eqn:fast-strong-central}) and that
$\E[W_f] = a/n$ for some constant
$a > 0$ that does not depend on $f$. These restrictions allow us to
formulate the goal of bounding the CGF as an instance of the general
moment problem of \citet{kemperman1968general,Kemperman:1987aa}. We will
later relax them to allow general $V$, $\Lambda_{-W_f}(\eta^*) \leq 0$
and $\E[W_f] \geq a/n$.

As illustrated by Figure~\ref{fig:cgf}, our approach will be to bound
$\Lambda_{-W_f}(\eta)$ at $\eta = \eta^*/2$ from above by maximizing
over all possible random variables $W_f$ subject to the given
constraints. This is equivalent to minimizing $-\E[\exp((\eta^*/2) S)]$
over $S = -W_f$ and may be formulated as an instance of the general
moment problem, which we describe next.

\begin{figure}[htb] 
  \centering 
  \includegraphics[width=80mm]{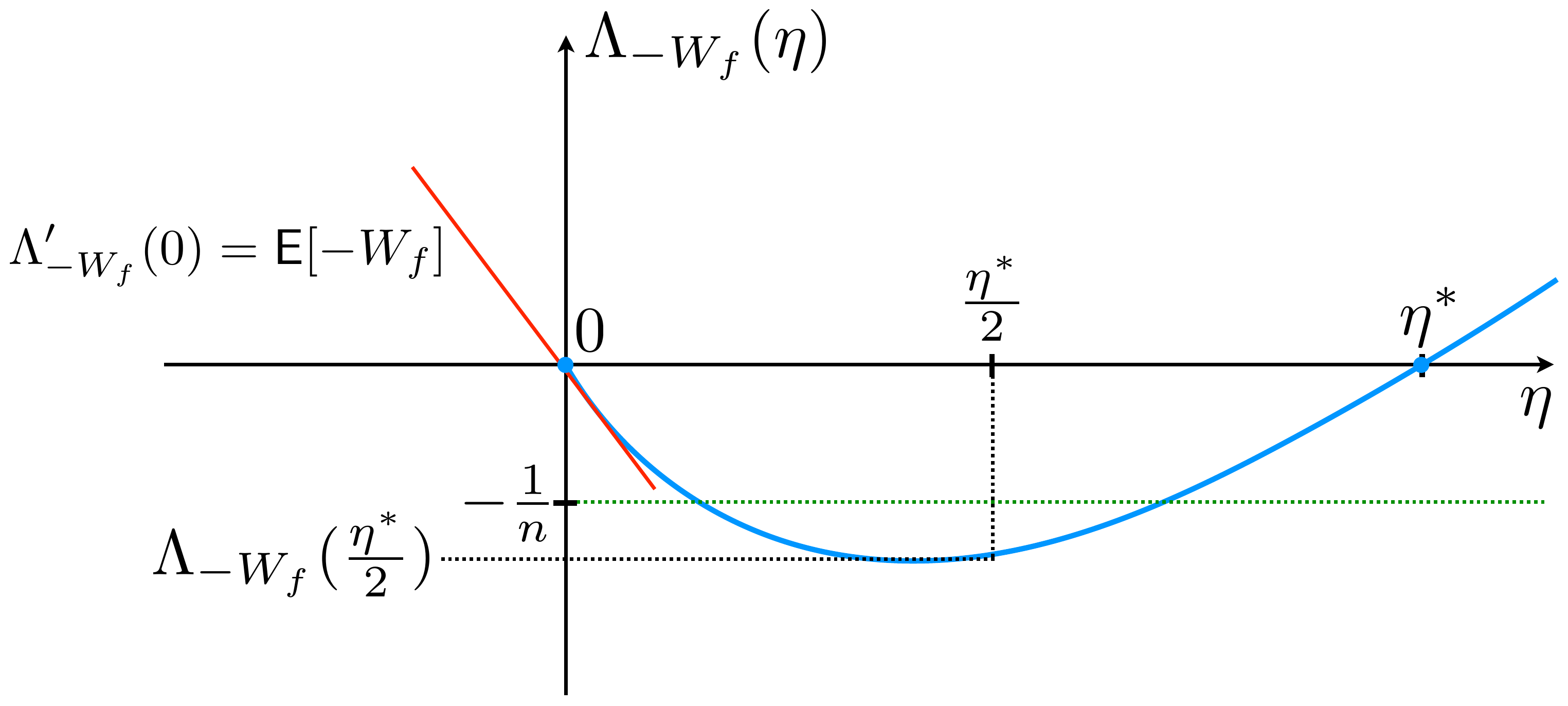} 
  \caption{\label{fig:cgf} Control of the CGF of $-W_f$ for a function
  $f$ with excess loss $\E[W_f]$ of order $\frac{1}{n}$. The derivative
  at $0$ equals $-\E[W_f]$.} 
\end{figure} 
 
\paragraph{\rm {\em The general moment problem.}}

Let $\Delta(\cS)$ be the set of all probability measures over a
measurable space $\cS$. Then for any real-valued measurable
functions $h$, $g_1,\ldots,g_m$ on $\cS$ and constants
$k_1,\ldots,k_m$, the general moment problem is the semi-infinite linear program
\begin{align} 
\begin{aligned} 
& \inf_{P \in \Delta(\cS)}  
& & \E_{S \sim P} h(S) \\ 
& \text{subject to} 
& &     \E_{S \sim P} g_j(S) = k_j, \quad j = 1, \ldots, m. 
\end{aligned} \label{eqn:the-general-moment-problem} 
\end{align} 
Define the vector-valued map $g \colon \cS \to \reals^m$ as $g(s) =
(g_1(s),\ldots,g_m(s))$ and the vector $k = (k_1,\ldots,k_m)$. Then
Theorem 3 of \citet{kemperman1968general}, which was also shown independently by \cite{richter1957parameterfreie} and \cite{karlin1966tchebycheff}, states that, if $k \in
\interior \convhull(g(\cS))$, the optimal value of problem
\eqref{eqn:the-general-moment-problem} equals 
\begin{align} 
\sup \biggl\{ d_0 + \sum_{j=1}^m d_j k_j : d^* = (d_0, d_1, \ldots, d_m) \in D^* \biggr\} , \label{eqn:the-optimal-value} 
\end{align} 
where $D^* \subseteq \reals^{m+1}$ is the set
\begin{align} 
D^* := \biggl\{ d^* = (d_0, d_1, \ldots, d_m) \in \reals^{m+1} : 
  h(s) \geq d_0 + \sum_{j=1}^m d_j g_j(s) \text{ for all } s \in \cS  
\biggr\} . \label{eqn:the-D-star} 
\end{align} 
Instantiating, we choose $\cS = [-1, 1]$ and define
\begin{align*} 
h(s) = -e^{(\eta^* / 2) s}, && 
g_1(s) = s, &&  
g_2(s) = e^{\eta^* s}, && 
k_1 = -\frac{a}{n}, && 
k_2 = 1,
\end{align*} 
which yields the following special case of problem \eqref{eqn:the-general-moment-problem}: 
\begin{subequations} 
\label{eqn:our-general-moment-problem} 
\begin{align} 
\inf_{P \in \Delta([-1, 1])} 
& \quad -\E_{S \sim P} e^{(\eta^* / 2) S} \label{eqn:our-general-moment-problem-a} \\ 
\text{subject to} 
& \quad \E_{S \sim P} S = -\frac{a}{n} \label{eqn:our-general-moment-problem-b} \\ 
& \quad \E_{S \sim P} e^{\eta^* S} = 1 \label{eqn:our-general-moment-problem-c} . 
\end{align} 
\end{subequations} 
Equation~\ref{eqn:the-optimal-value} from the general moment problem now instantiates to 
\begin{align} 
\sup \left\{ d_0 - \frac{a}{n} d_1 + d_2 : d^* = (d_0, d_1, d_2) \in D^* \right\} , \label{eqn:our-optimal-value} 
\end{align} 
with $D^*$ equal to the set 
\begin{align} 
\left\{ d^* = (d_0, d_1, d_2) \in \reals^3 :  
-e^{(\eta^*/2) s} \geq d_0 + d_1 x + d_2 e^{\eta^* s} \text{ for all } s \in [-1, 1] \right\} . \label{eqn:our-D-star} 
\end{align} 

Applying Theorem 3 of \citet{kemperman1968general} requires $k \in
\interior \convhull g([-1, 1])$. We first characterize when $k \in \convhull g([-1,1])$ holds and handle the $\interior \convhull g([-1,1])$ version after \cref{thm:stochastic-mixability-concentration}. The proof of the next result, along with all subsequent results in this section, can be found in \cref{app:fastratesproof}.
\begin{lemma}
\label{lemma:feasible-moments} 
For $a > 0$, the point $k = \left( -\frac{a}{n}, 1 \right) \in \convhull(g([-1, 1]))$ if and only if 
\begin{align} 
\frac{a}{n}  
\leq \frac{e^{\eta^*} + e^{-\eta^*} - 2}{e^{\eta^*} - e^{-\eta^*}}  
= \frac{\cosh(\eta^*) - 1}{\sinh(\eta^*)} . \label{eqn:feasibility} 
\end{align}
Moreover, $k \in \interior \convhull(g([-1, 1]))$ if and only if the inequality in \eqref{eqn:feasibility} is strict.
\end{lemma} 

Note that \eqref{eqn:feasibility} is guaranteed to hold, because otherwise the semi-infinite linear program \eqref{eqn:our-general-moment-problem} is infeasible (which in turn implies that such an excess loss random variable cannot exist).

The next theorem is a key result for using the strong central condition to control the CGF. 
\begin{theorem}
\label{thm:stochastic-mixability-concentration} 
Let $f$ be an element of $\F$ with $\xslossat{Z}$ taking values in $[-1,
1]$, $n \in \nat$, $\E_{Z \sim P} \xslossat{Z} = \frac{a}{n}$ for some $a > 0$, and $\Lambda_{-\xslossat{Z}}(\eta^*) = 0$ for some $\eta^* > 0$. If 
\begin{align} 
\frac{a}{n} < \frac{\cosh(\eta^*) - 1}{\sinh(\eta^*)} , \label{eqn:interior-point} 
\end{align} 
\begin{flalign*} 
\text{then} &&\Lambda_{-\xslossat{Z}}(\eta^* / 2) \leq \frac{-0.21 (\eta^*
\opwedge 1) a}{n}. && 
\end{flalign*} 
\end{theorem}

\begin{corollary}
\label{cor:stochastic-mixability-concentration}
The result of \cref{thm:stochastic-mixability-concentration} also holds when the strict inequality in \eqref{eqn:interior-point} is replaced with inequality, i.e.~$\frac{a}{n} \leq \frac{\cosh(\eta^*) - 1}{\sinh(\eta^*)}$.
\end{corollary}
 
We now present an extension of this result for losses with range $[0, \bound]$.
\begin{corollary}
\label{cor:bounded-losses}
Let $g_1(x) = x$ and $y_2 = 1$ be common settings for the following two problems. 
The instantiation of problem \eqref{eqn:the-general-moment-problem} with $\cS = [-\bound, \bound]$, $h(x) = -e^{(\eta / 2) x}$, $g_2(x) = e^{\eta x}$, and $y_1 = -\frac{a}{n}$ has the same optimal value as the instantiation of problem \eqref{eqn:the-general-moment-problem} with $\cS = [-1, 1]$, $h(x) = -e^{(\bound \eta / 2) x}$, $g_2(x) = e^{ (\bound \eta) x}$, and $y_1 = -\frac{a / \bound}{n}$. 
\end{corollary}

\subsection{Fast Rates} 
\label{sec:fast-rates-main-results} 
 
We now show how the above results can be used to obtain an exact oracle inequality with a fast rate. We first present a result for finite classes and then present a result for VC-type classes (classes with logarithmic universal metric entropy). 
 
\begin{theorem}
\label{thm:finite-fast-rates}  
Let $(\loss, \Prob, \F)$ satisfy the strong $\eta^*$-central condition, where $|\F|  = N$, $\loss$ is a nonnegative loss, and $\sup_{f \in \F} \loss_f(Z) \leq \bound$ a.s.~for a constant $\bound$. Then for all $n \geq 1$, with probability at least $1 - \delta$ 
\begin{align*} 
\E_{Z \sim P} [ \loss_{\erm}(Z) ]  
\leq \E_{Z \sim P} [ \loss_{f^*}(Z) ]  
       + \frac{5 \max\left\{ \bound, \frac{1}{\eta^*} \right\} \left( \log \frac{1}{\delta} + \log N  
\right)}{n} . 
\end{align*} 
\end{theorem} 
 
Before presenting the result for VC-type classes, we require some definitions.  
For a pseudometric space $(\G, d)$, for any $\varepsilon > 0$, let
$\N(\varepsilon, \G, d)$ be the $\varepsilon$-covering number of $(\G,
d)$; that is, $\N(\varepsilon, \G, d)$ is the minimal number of balls
of radius $\varepsilon$ needed to cover $\G$. We will further
constrain the cover (the set of centers of the balls) to be a subset
of $\G$ (i.e.~to be proper), thus ensuring that the strong central
condition assumption transfers to any (proper) cover of $\F$. Note
that the `proper' requirement at most doubles the constant $K$ below,
as shown in Lemma 2.1 of  \cite{vidyasagar2002learning}. 
 
We now present the fast rates result for VC-type classes. The proof, which can be found as the proof of Theorem 7 of 
\cite{mehta2014stochastic}, uses Theorem 6 of \cite{mehta2014stochastic} and the proof of \cref{thm:finite-fast-rates}.  
Below, we denote the loss-composed version of a function class $\F$ as $\lossclass := \{\loss_f : f \in \F\}$. 
 
\begin{theorem}
\label{thm:vc-type-fast-rates} 
Let $(\loss, \Prob, \F)$ satisfy the strong $\eta^*$-central condition with $\lossclass$ separable, where,  
for a constant $K \geq 1$, for each $\varepsilon \in (0, K]$ we have $\N(\lossclass, L_2(\Prob), \varepsilon) \leq \left( \frac{K}{\varepsilon} \right)^\capacity$, and $\sup_{f \in \F} \lossatof{Y}{f(X)} \leq \bound$ a.s.~for a constant $\bound \geq 1$. Then for all $n \geq 5$ and $\delta \leq \frac{1}{2}$, with probability at least $1 - \delta$,
{\small 
\begin{multline} 
\E_{Z \sim P} [ \loss_{\erm} (Z) ]  
\leq \\ \E_{Z \sim P} [ \loss_{f^*} (Z) ]  
+ \frac{1}{n} \max\left\{  
\begin{array}{c} 
  8 \max\left\{ \bound, \frac{1}{\eta^*} \right\}  
  \left( \capacity \log(K n) + \log \frac{2}{\delta} \right) , \\ 
  2 \bound \left(  
    1080 \capacity \log(2 K n)  
    + 90 \sqrt{\left( \log \frac{2}{\delta} \right) \capacity \log(2 K n)}  
    + \log \frac{2 e}{\delta}  
  \right) 
\end{array} 
\right\} + \frac{1}{n}. \nonumber 
\end{multline} 
} 
\end{theorem} 
We have shown the fast rate of $O(1/n)$ under the best case of the $\gComparator$-central condition, i.e.~when $\gComparator$ is constant; however, it also is possible to recover intermediate rates for the case of general $\gComparator$. 
\begin{corollary}\label{cor:finite-intermediate-rates} 
Let $(\loss, \Prob, \F)$ satisfy the $\gComparator$-central condition hold for a finite class $\F$. 
Then, for some constant $c$, for all $n$ satisfying  
$\gComparator\left( w^{-1} \left( \frac{5 (\log \frac{1}{\delta} +
      \log N)}{c n} \right) \right) \leq \frac{1}{c V}$,  
we get an intermediate rate of 
$w\left( \frac{5 (\log \frac{1}{\delta} + \log N)}{c n} \right)$, where $w$ is the inverse of the function $x \mapsto x \gComparator(x)$. 
\end{corollary} 
\begin{proof} 
From part (2) of \cref{thm:BernsteinComparator}, the $\gComparator$-central condition implies the $\gBernstein$-Bernstein condition for $\gBernstein(x) \asymp x/\gComparator(x)$, and from part (1b) of \cref{thm:BernsteinComparator}, we then have the $\eta$-central condition for $\eta = c \gComparator(\delta)$ for the subclass of functions with excess risk above $\delta$, for some constant $c$. From here, a simple modification of the proof of \cref{thm:finite-fast-rates} yields the desired result as follows. Let $\epsilon$ correspond to the excess risk threshold above which ERM should reject all functions with high probability. Then, similar to the proof of \cref{thm:finite-fast-rates}, we upper bound the probability of ERM picking a function with excess risk $\epsilon$ or higher: 
\begin{align*} 
N \exp(n \Lambda_{-W_f}(c \gComparator(\epsilon))  
&= N \exp(n \Lambda_{-W_f/V}(c V \gComparator(\epsilon)) \\ 
&\leq N \exp\left( -0.21 n \bigl( c V \gComparator(\epsilon) \opwedge 1 \bigr) \frac{\epsilon}{V} \right) . 
\end{align*} 
For $\epsilon$ satisfying $\gComparator(\epsilon) \leq \frac{1}{c V}$, the failure probability $\delta$ is at most 
  $N \exp(-0.21 c n \epsilon \gComparator(\epsilon))$,  
and hence by inversion we get the rate  
$w\left( \frac{5 (\log \frac{1}{\delta} + \log N)}{c n} \right)$. 
\end{proof}

\section{Discussion, Open Problems and Concluding Remarks}
\label{sec:discussion}
In this paper we identified four general conditions for fast and
intermediate learning rates. The two main ones, which subsumed many previously identified conditions, where the central condition and stochastic
mixability. We provided sufficient assumptions under which the four conditions  become equivalent via the implications
\begin{equation}\label{eq:almost}
\eta \text{-central} \Rightarrow \eta \text{-predictor} \Rightarrow
\eta \text{-stochastic mixability} \Rightarrow \eta \text{-PPC}
\Rightarrow \eta \text{-central}.
\end{equation}
In Section~\ref{sec:convex-face} and~\ref{sec:four-conditions} we
considered the versions of these conditions for fixed $\eta > 0$, as
given by Theorem~\ref{thm:secondmain},
Proposition~\ref{prop:frompredictortoconvexface}, 
Proposition~\ref{prop:fromsmtoppcc} and
Theorem~\ref{thm:convexfacetocomparator}, respectively. For this fixed
$\eta > 0$ case,
all implications except one hold under surprisingly weak conditions,
in particular allowing for unbounded loss functions. The exception is
`central $\Rightarrow$ predictor'
(Theorem~\ref{thm:secondmain}). Although even this result was
applicable to some noncompact decision sets $\cF$ with unbounded
losses (Example~\ref{ex:unbounded-sq-pc}), it requires tightness and
convexity of the set $\cP$, although Example~\ref{ex:loglossassD}
shows that sometimes the implication holds even though $\cP$ is
neither tight nor convex. An important open question is whether
Theorem~\ref{thm:secondmain} still holds under weaker versions of
Assumption~\ref{ass:minimax} or
Assumption~\ref{ass:convexityContinuity}. 

Another restriction of Theorem~\ref{thm:secondmain} is that, via
Assumption~\ref{ass:convexityContinuity}, it requires convexity of the
decision set $\decisionset$, which fails for the $0/1$-loss $\zoloss$
and its conditional version, the classification loss $\classloss$.
However, we may extend the definition of $\zoloss$ to $\cF = [0,1]$
and define the resulting {\em randomized $0/1$\/} or {\em absolute\/}
loss as $\absloss_{f}(z) := |y - f|$. This can be interpreted as the
$0/1$-loss a decision maker expects to make if she is allowed to
randomize her decision by flipping a coin with bias $f$ --- a standard
concept in PAC-Bayesian approaches
\citep{audibert2004pac,catoni2007pac}. For the absolute loss, we can
consider $\eta$-stochastic mixability for $\decisionset =
\convhull(\cF) = [0,1]$, which is convex; hence, the requirement of
convex $\decisionset$ in Theorem~\ref{thm:secondmain} is not such a
concern.

In Section~\ref{sec:comp-mix-margin} we discussed weakenings of the
four conditions to their $\gComparator$-versions. Now for {\em
  bounded\/} losses, the four implications above still hold under
similar conditions as for the fixed $\eta$-case. Since the first three
implications in (\ref{eq:almost}) were proven in an `up to $\epsilon$'
form for all $\epsilon > 0$, it immediately follows that for arbitrary
functions $\gComparator$, the implications continue to hold under the
same assumptions if the $\eta$-conditions are replaced by the
corresponding $\gComparator$-conditions. This does not work for the
fourth implication, since Theorem~\ref{thm:convexfacetocomparator} is
not given in an `up to $\epsilon$' form (indeed, we conjecture that it
does not hold in this form). However, we can work around this issue by
using instead a detour via the Bernstein condition: by using first
part 2 and then part 1 in Theorem~\ref{thm:BernsteinComparator}, it
follows that the $\gComparator$-PPC condition implies the
$\gComparator'$-central condition for $\gComparator'(\epsilon) \asymp
\gComparator(\epsilon)$, so the four $\gComparator$-conditions still
imply each other, under the same assumptions as before, up to constant
factors. However, the Bernstein-detour works only for bounded losses,
and Example~\ref{ex:comparator-vs-bernstein},~\ref{ex:lastminute}
and~\ref{ex:jrtee} together indicate that in general
it cannot be made to work and indeed the analogue of (\ref{eq:almost})
for the $v$-conditions does not hold for unbounded losses: for decision problems
with polynomial rather than exponential tails on the losses,
$v$-stochastic mixability and  the $v$-PPC condition may hold whereas
the $v$-central condition does not. Thus there is the
question whether the central condition can be weakened such that the 
four implications for the $\gComparator$-versions
continue to hold, under weak conditions, for unbounded losses --- and
we regard this as the main open question posed by this work. Another
issue here is that, if in a decision problem $\dpthree$ that satisfies
a $\gComparator$-condition, we replace $\cP$ by its convex closure,
then the $\gComparator$-condition may very well be broken, so, once
again, a weakening of Assumption~\ref{ass:convexityContinuity} to
nonconvex $\cP$ seems required. Finally, it would be of considerable
interest if one could show an analogue for unbounded losses of
Proposition~\ref{prop:nonunique}, which connects --- for bounded
losses --- the central condition to the existence of a unique risk
minimizer. Relatedly, it would be desirable to link this proposition
to the results by \cite{mendelson2008lower} who also connects slow
rates with nonunique risk minimizers, and to
\cite{koltchinskii2006local} who gives a version of the Bernstein
condition that does hold if nonunique minimizers exist, indicating
that our $\eta$-central condition (which via
Proposition~\ref{prop:unique} implies unique minimizers) might
sometimes be too strong.

Apart from these implications in the `main quadrangle' of Figure \ref{fig:map-of-paper}
on page~\pageref{fig:map-of-paper}, it would be good to strengthen
some of the other connections shown in that figure, such as the
precise relation between $\eta$-mixability and $\eta$-exp-concavity.
It would also be desirable to establish connections to results in {\em
  defensive forecasting\/} \citep{Chernov2010} in which conditions
similar to both the central condition and mixability play a role;
their Theorem 9 is reminiscent of the special case of our
Theorem~\ref{thm:secondmain} for the case that $\cZ$ is finite and
$\cP$ consists of all distributions on $\cZ$.

We focused on showing {\em
  equivalence\/} of fast rate conditions and not on showing that one
can actually always {\em obtain\/} fast rates under these conditions.
For stochastic mixability, this immediately follows, under no
further conditions, from Proposition~\ref{prop:stoch-mix-AA}. For the central condition, the
situation is more complicated: in this paper we only showed that it
implies fast rates for bounded loss functions.  We
know that, for the unbounded log-loss, fast rates can be obtained
under the central condition (and no additional conditions) in a weaker
sense, involving R\'enyi and squared Hellinger distance
(Section~\ref{sec:overview}); in work in progress, we aim at showing
that the central condition implies fast rates in the standard sense
even for unbounded loss functions. This does appear possible, up to
log-factors, however it seems that here one does need weak additional
conditions such as existence of certain moments different from the
exponential moment in (\ref{eqn:basiccomparatorconditionpre}).

Second, by `fast' rates we merely meant rates of order $1/n$; it would
of course be highly desirable to characterize when the rates that are achieved
under our conditions by appropriate algorithms (ERM, Bayes MAP-style and
MDL methods for the central condition, the aggregating algorithm for
stochastic mixability) are indeed minimax optimal. Similarly, one
would need examples showing that if a condition fails, then the
corresponding fast or intermediate rates {\em cannot\/} be obtained in
general. While several such results are  available, they either focus
on showing that, in the worst-case over all $P \in \cP$, {\em no\/}
learning algorithm, proper or improper, can achieve a certain rate (in particular
\cite{audibert2009fast} gives very general results), or that a
particular proper learning algorithm such as ERM cannot achieve a
certain rate \citep{mendelson2008lower}. Currently unexplored, it
seems,  are minimax results where one looks at the optimal (not just
ERM) algorithm, but  within the restricted class of all proper learning
algorithms. 

In the spirit of Vapnik and Chervonenkis, who discovered under what conditions one can learn from a finite amount of data at all, we continue our quest for conditions under which one can learn from data using not too many examples.


\acks{\refstepcounter{dummy}\label{sec:acks}
We thank Olivier Catoni for raising the issue of unbounded losses discussed in Example~\ref{ex:comparator-vs-bernstein}, 
Wouter Koolen for suggesting the connection to minimax theorems, 
and Andrew Barron for various in-depth discussions over the last 16 years. 
Most of the results in Section~\ref{sec:fast-rates} were published
before in the conference paper by \cite{mehta2014stochastic}; 
very preliminary versions of Theorem~\ref{thm:convexfacetocomparator}, Theorem~\ref{thm:secondmain}
 (only for $\cP$ the set of all distributions on $\cZ$) and Theorem~\ref{thm:BernsteinComparator} were published before by \cite{vanerven2012mixability}, in which we used the phrase `$\dpthree$ is stochastically mixable' to denote what we now refer to as `$\dpthree$ satisfies the central condition'. 
We thank both the referees of the present paper and the referees of these earlier conference papers for their useful feedback. }  This work was supported in part by the Australian Research Council, and by
NICTA which is funded by the Australian Government, as well as by by
the Netherlands Organiszation for Scientific Research (NWO)
Project  639.073.904.

\vskip 0.2in

\appendix

\section{Additional Proofs}
\subsection{Proof of Theorem~\ref{thm:convexfacetocomparator} in Section~\ref{sec:convex-face}}
\label{app:convex-faceproofs}
\begin{proof}
We first consider the case that Assumption~\ref{ass:simple} holds, and
then the case of bounded loss. 
\paragraph{\rm {\em Under Assumption~\ref{ass:simple}}.}
Under our Assumption~\ref{ass:simple}, we can, for each $P
\in \cP$, define $\phi(P) := f^* \in \model$ to
be optimal in the sense of \eqref{eq:bayesact}. Note that $f^*$
depends on $P$, but not on any $\Pi$. Since we also assume the weak
$\eta$-pseudoprobability convexity condition, we must have that for every $\epsilon >
0$, the $\eta$-pseudoprobability convexity condition holds up to $\epsilon$ for some
function $\phi_{\epsilon}$. It follows that 
for all $\epsilon > 0$, $\E_{Z \sim P}[\loss_{f^*}(Z)] \leq 
\E_{Z \sim P}[\loss_{\phi_{\epsilon}(P)}(Z)] \leq  
\E_{Z \sim
  P}[m_{\Pi}^\eta(Z)] + \epsilon$, so that also
\begin{equation}\label{eqn:universalminimizer}
  \E_{Z \sim P}[\loss_{f^*}(Z)] \leq \E_{Z \sim P}[m_{\Pi}^\eta(Z)]
\end{equation}
for all $\Pi \in \Delta(\cF)$. Now fix arbitrary $P \in \cP$, let
$f^*= \phi(P)$ and let
$f \in \model$ be arbitrary and consider the special case that $\Pi =
(1-\lambda) \delta_{f^*} + \lambda \delta_f$ for $\lambda \in
[0,\thalf]$, where $\delta_f$ is a point-mass on $f$. Let
\begin{equation*}
  \chi(\lambda,z) = \eta m_\Pi^\eta(z) = -\log \left((1-\lambda) e^{-\eta \loss_{f^*}(z)} +
  \lambda e^{-\eta \loss_f(z)}\right)
\end{equation*}
be the corresponding mix loss multiplied by $\eta$, and let
\begin{equation*}
  \chi(\lambda)
    = \E_{Z \sim P}[\chi(\lambda,Z)]
    = \eta \E_{Z \sim P}[m_\Pi^\eta(Z)]
\end{equation*}
be its expected value. Then from \eqref{eqn:universalminimizer} it
follows that $\chi(\lambda)$ is minimized at $\lambda = 0$, which
implies that the right-derivative $\chi'(0)$ at $0$ is nonnegative:
\begin{equation}\label{eqn:derivativeAtZero}
  \chi'(0) \geq 0.
\end{equation}
In order to compute $\chi'(0)$, we first observe that, for any $z$,
$\chi(\lambda,z)$ is convex in $\lambda$, because it is the composition
of the negative logarithm with a linear function.
Convexity of $\chi(\lambda,z)$ in $\lambda$ implies that the slope
$s(d,z) = \frac{\chi(0+d,z)-\chi(0,z)}{d}$ is non-decreasing in $d \in
(0,\thalf]$ and achieves its maximum value at $d = 1/2$, where it never
exceeds $2 \log 2$:
\begin{equation*}
 s(1/2,z)
  = 2 \log \frac{e^{-\eta \loss_{f^*}(z)}}{\half e^{-\eta\loss_{f^*}(z)}
 + \half e^{-\eta \loss_f(z)}}
  \leq 2 \log \frac{e^{-\eta \loss_{f^*}(z)}}{\half e^{-\eta\loss_{f^*}(z)}}
  = 2 \log 2.
\end{equation*}
Hence $\E_{Z \sim P}[s(\half,Z)] \leq 2 \log 2 < \infty$ and by the
monotone convergence theorem \citep{shiryaev1996probability}
\begin{equation}\label{eqn:computedchiprime}
  \chi'(0)
    = \lim_{d \downarrow 0} \E_{Z \sim P} \left[s(d,Z)\right]   
    = \E_{Z \sim P} \left[\lim_{d \downarrow 0} s(d,Z)\right]   
    = \E_{Z \sim P} \left[\frac{\der}{\der \lambda} \chi(\lambda,Z)|_{\lambda =
    0}\right]
    = 1-\E_{Z \sim P} \left[\frac{e^{-\eta \loss_f(Z)}}{e^{-\eta
    \loss_{f^*}(Z)}}\right].
\end{equation}
Together with \eqref{eqn:derivativeAtZero} and the fact that $\phi(P)
= f^*$ and that $P$ was chosen arbitrarily, this implies the strong
$\eta$-central condition as required.

\paragraph{\rm {\em When the Loss is Bounded}.}
Let $P \in \cP$ be arbitrary. The $\eta$-pseudoprobability convexity condition implies that for any
$\gamma > 0$ we can find $f^* \in \model$ such that
\begin{equation*}
  \E_{Z\sim P}\left[\loss_{f^*}(Z)\right]
 \leq \E_{Z \sim P}\left[m_{\Pi}^\eta(Z)\right] + \gamma
\end{equation*}
for all distributions $\Pi \in \Delta(\model)$. Choose any $f \in \model$
and consider again the special case $\Pi = (1-\lambda) \delta_{f^*} +
\lambda \delta_f$ for $\lambda \in [0,\thalf]$, which gives
\begin{equation}\label{eqn:inexactMinimizer}
  \chi(0) \leq \chi(\lambda) +  \eta \gamma
\end{equation}
for $\chi(\lambda)$ as above. This time $\chi(0)$ is not necessarily the
exact minimum of $\chi(\lambda)$, but \eqref{eqn:inexactMinimizer}
expresses that it is close. To control $\chi'(0)$, we use that
\begin{equation*}
  \chi(\lambda,z) = \chi(0,z) + \lambda \tfrac{\der}{\der \lambda}
  \chi(0,z) + \thalf \lambda^2 \tfrac{\der^2}{\der \lambda^2} \chi(\xi,z)
  \qquad \text{for some $\xi \in [0,\lambda]$}
\end{equation*}
by a second-order Taylor expansion in $\lambda$, which implies that
\begin{equation*}
  \chi(\lambda) - \chi(0) - \lambda \chi'(0)
  \leq \frac{\lambda^2}{2} \max_{z,\lambda'} \left(\frac{e^{-\eta \loss_{f^*}(z)} -
  e^{-\eta \loss_f(z)}}{(1-\lambda') e^{-\eta \loss_{f^*}(z)} +
  \lambda' e^{-\eta \loss_f(z)}}\right)^2
  \leq \frac{\lambda^2}{2} \left(e^{\eta 2 B} - 1\right)^2.
\end{equation*}
Together with \eqref{eqn:inexactMinimizer} the choice $\lambda =
\sqrt{\gamma}$ (which requires $\gamma \leq 1/4$) then allows us to
conclude that
\begin{align*}
  -\eta \gamma
    &\leq \chi(\sqrt{\gamma}) - \chi(0)
    \leq \sqrt{\gamma} \chi'(0) + \frac{\gamma}{2} \left(e^{\eta 2
    B} - 1\right)^2\\
  \chi'(0) &\geq -c \sqrt{\gamma}
\end{align*}
for $c = \eta + \half(e^{\eta 2 B} -1)^2$. Since
\eqref{eqn:computedchiprime} still holds, taking $\gamma$ small enough
that $1 + c \sqrt{\gamma} \leq e^{\eta \epsilon}$ gives us the central
condition \eqref{eqn:comparator} for any $\epsilon > 0$.
\end{proof}

\subsection{Proof of Lemma~\ref{lem:convexityContinuityToMinimax} in
  Section~\ref{sec:four-conditions}}
\label{app:minimaxproof}

\begin{proof}
Theorem~6.1 of \citet{GrunwaldD04}, itself a direct consequence of a
  minimax theorem due to \citet{Ferguson67}, states the following: if
  a set of distributions $\closedcP$ is convex, tight and closed in the weak topology, and $L \colon
  \cZ \times \decisionset \to \reals$ is a function such that, for all
  $f$, $L(z,f)$ is bounded from above and upper semi-continuous in $z$,
  then 
  \begin{equation}\label{eqn:PetersMinimax}
    \sup_{P \in \closedcP} \lowinf_{f \in \decisionset} \E_{Z \sim P}[L(Z,f)]
      = \lowinf_{\rho \in \Delta(\decisionset)} \sup_{P \in \closedcP} \E_{Z
      \sim P} \E_{f \sim \rho}[L(Z,f)].
  \end{equation}
  Let $\Pi \in \Delta(\decisionset)$ be arbitrary, and observe that
  $S_\Pi^\eta(P,f)$ is related to $\xi_{Z,f}$ via 
  \begin{equation*}
    S_\Pi^\eta(P,f) = \E_{Z \sim P}[\xi_{Z,f}],
  \end{equation*}
  so we will aim to apply \eqref{eqn:PetersMinimax} with $L(z,f)$
  approximately equal to $\xi_{z,f}$. Although $\xi_{z,f}$ is not
  necessarily bounded above, rewriting
  \begin{equation*}
    \xi_{z,f}
      = e^{\eta\loss_f(z)}\E_{g \sim \Pi}\left[e^{-\eta
      \loss_g(z)}\right],
  \end{equation*}
  we find that it is continuous in $z$, because $\loss_f(z)$
  is continuous in $z$ and $\E_{g \sim \Pi}\left[e^{-\eta
  \loss_g(z)}\right]$ is also continuous in $z$ by continuity of
  $\loss_g(z)$ and the dominated convergence theorem
  \citep{shiryaev1996probability}, which applies because $|e^{-\eta
  \loss_g(z)}| \leq 1$. Letting $a \bmin b$ denote the minimum of $a$
  and $b$, it follows that $\xi_{z,f} \bmin b$ is also continuous in $z$
  for any number $b$.

  Thus we can apply \eqref{eqn:PetersMinimax} to the function $L(z,f) =
  \xi_{z,f} \bmin b$, with $\closedcP$ the closure of $\cP$ in the weak
  topology, to obtain
  \begin{equation}\label{eqn:applyMinimax}
    \lowinf_{\rho \in \Delta(\decisionset)}\sup_{P \in {\cP}} \E_{Z \sim
    P}\E_{f \sim \rho}[\xi_{Z,f} \bmin b]
    \leq \lowinf_{\rho \in \Delta(\decisionset)}\sup_{P \in {\closedcP}} \E_{Z \sim
    P}\E_{f \sim \rho}[\xi_{Z,f} \bmin b]
    = \sup_{P \in {\closedcP}} \lowinf_{f \in \decisionset} \E_{Z \sim
    P}[\xi_{Z,f} \bmin b].
  \end{equation}
  We will show that 
  \begin{equation}\label{eq:closure}
    \sup_{P \in {\closedcP}} \lowinf_{f \in \decisionset} \E_{Z \sim
    P}[\xi_{Z,f} \bmin b]
      \leq \sup_{P \in \cP} \lowinf_{f \in \decisionset} \E_{Z \sim
      P}[\xi_{Z,f} \bmin b].
  \end{equation}
  If $\cP$ is closed itself (first possibility in
  \ref{ass:convexityContinuity}.\ref{it:equicont}), then $\closedcP =
  \cP$ and this is immediate. The second possibility will be covered at
  the end of the proof.

  Together, \eqref{eqn:applyMinimax} and \eqref{eq:closure} imply that
  \begin{equation*}
      \lowinf_{\rho \in \Delta(\decisionset)}\sup_{P \in {\cP}} \E_{Z \sim
    P}\E_{f \sim \rho}[\xi_{Z,f} \bmin b]
      \leq \sup_{P \in \cP} \lowinf_{f \in \decisionset} \E_{Z \sim
      P}[\xi_{Z,f} \bmin b]
      \leq \sup_{P \in \cP} \lowinf_{f \in \decisionset} \E_{Z \sim
      P}[\xi_{Z,f}]
  \end{equation*}
  for any finite $b$. We will show that, for every $\epsilon > 0$, there
  exists a $b$ such that
  \begin{equation}\label{eqn:epsilonb}
    \E_{Z \sim P} \E_{f \sim \rho}[\xi_{Z,f} \bmin b]
      \geq \E_{Z \sim P} \E_{f \sim \rho}[\xi_{Z,f}] - \epsilon
    \qquad \text{for all $\rho \in \Delta(\decisionset)$ and $P \in {\cP}$.}
  \end{equation}
  By letting $\epsilon$ tend to $0$, we can therefore conclude that
  \begin{equation}\label{eqn:halfminimax}
    \sup_{P \in {\cP}} \lowinf_{f \in \decisionset} \E_{Z \sim
    P}[\xi_{Z,f}]
    \geq \lowinf_{\rho \in \Delta(\decisionset)}\sup_{P \in {\cP}} 
      \E_{Z \sim P} \E_{f \sim \rho}[\xi_{Z,f}]
    = \lowinf_{f \in \decisionset}\sup_{P \in {\cP}} 
      \E_{Z \sim P} [\xi_{Z,f}],
  \end{equation}
  where the identity follows from the requirement that $e^{\eta
  \loss_f(z)}$ is convex in $f$, which implies that $\xi_{Z,f}$ is also
  convex in $f$, and hence the mean of $\rho$ is always at least as good
  as $\rho$ itself: $\xi_{Z,\E_{f \sim \rho}[f]} \leq \E_{f \sim
  \rho}[\xi_{Z,f}]$. Since the $\sup \inf$ never exceeds the $\inf
  \sup$, \eqref{eqn:halfminimax} implies \eqref{eqn:minimaxEquality},
  which was to be shown.

  To prove \eqref{eqn:epsilonb}, we observe that
  \begin{equation*}
    \E_{Z \sim P} \E_{f \sim \rho}[\xi_{Z,f} \bmin b]
    \geq \E_{Z \sim P} \E_{f \sim \rho}[\xi_{Z,f} \ind{\xi_{Z,f} < b}]\\
    = \E_{Z \sim P} \E_{f \sim \rho}[\xi_{Z,f}]
        - \E_{Z \sim P} \E_{f \sim \rho}[\xi_{Z,f} \ind{\xi_{Z,f} \geq b}],
  \end{equation*}
  and, by uniform integrability, we can take $b$ large enough that
  $\E_{Z \sim P} \E_{f \sim \rho}[\xi_{Z,f} \ind{\xi_{Z,f} \geq b}]
  \leq \epsilon$ for all $\rho$ and $P$, as required.

  Finally, it remains to establish \eqref{eq:closure} for the second
  possibility in
  Assumption~\ref{ass:convexityContinuity}.\ref{it:equicont}. To this
  end, let $\epsilon > 0$ be arbitrary and let $\cZ' \subseteq \cZ$ be a
  compact set such that $P(\cZ') \geq 1-\epsilon$ for all $P \in \cP$. In
  addition, let $\delta > 0$ be small enough that
  \begin{equation*}
    \sup_{z \in \cZ'} |\loss_f(z)-\loss_g(z)| < \epsilon
    \qquad \text{for all $f,g \in \decisionset$ such that $d(f,g) <
    \delta$,}
  \end{equation*}
  which is possible by the assumption of uniform equicontinuity. Since
  $\decisionset$ is totally bounded, it can be covered by a finite
  number of balls of radius $\delta$.  Let $\discdecisionset \subseteq
  \decisionset$ be the (finite) set of centers of those balls. Then we
  can bound the left-hand side of \eqref{eq:closure} as follows:
  \begin{equation*}
    \sup_{P \in {\closedcP}} \lowinf_{f \in \decisionset} \E_{Z \sim
    P}[L(Z,f)]
    \leq \sup_{P \in {\closedcP}} \min_{f \in \discdecisionset} \E_{Z \sim
    P}[L(Z,f)]
    = \sup_{P \in {\cP}} \min_{f \in \discdecisionset} \E_{Z \sim
    P}[L(Z,f)] ,
  \end{equation*}
  where the equality holds by continuity of $\E_{Z \sim P}[L(Z,f)]$ and
  hence $\min_{f \in \discdecisionset} \E_{Z \sim P}[L(Z,f)]$ in $P$. We
  now need to relate $\discdecisionset$ back to $\decisionset$, which is
  possible because, for every $f \in \decisionset$, there exists $\discf
  \in \discdecisionset$ such that $d(f,\discf) < \delta$ and hence
  $|\loss_{\discf}(z)- \loss_f(z)| < \epsilon$ for all $z \in \cZ'$. It
  follows that $L(z,\discf) \leq e^{\eta \epsilon} L(z,f)$ and therefore
  \begin{multline*}
    \sup_{P \in {\cP}} \min_{f \in \discdecisionset} \E_{Z \sim
    P}[L(Z,f)]
    \leq \sup_{P \in {\cP}} \min_{f \in \discdecisionset} \E_{Z \sim
    P}[\ind{Z \in \cZ'} L(Z,f)] + \epsilon b\\
    \leq e^{\eta \epsilon} \sup_{P \in {\cP}} \lowinf_{f \in
    \decisionset} \E_{Z \sim P}[\ind{Z \in \cZ'} L(Z,f)] + \epsilon b
    \leq e^{\eta \epsilon} \sup_{P \in {\cP}} \lowinf_{f \in
    \decisionset} \E_{Z \sim P}[L(Z,f)] + \epsilon b,
  \end{multline*}
  and letting $\epsilon$ tend to $0$ we obtain \eqref{eq:closure}, which
  completes the proof.
\end{proof}
\subsection{Proof of Theorem~\ref{thm:BernsteinComparator} in
  Section~\ref{sec:comp-mix-margin}}
\label{app:comp-mix-marginproofs}

\begin{proof}
  We prove the two cases in turn.

  \paragraph{\rm {\em Bernstein $\Rightarrow$ Central.}}

  Fix arbitrary $P \in \cP$, and let $f^*$ be $\cF$-optimal, i.e. satisfying (\ref{eq:bayesact}). In this part of the proof, all
  expectations $\E$ are taken over $Z \sim P$.  

  Suppose that the $\gBernstein$-Bernstein condition holds.  Fix
  arbitrary $f \in \cF$ and let $X = \loss_f(Z) - \loss_{f^*}(Z)$.
  Let $\epsilon \geq 0$ and set $\eta = \gComparator(\epsilon)\leq
  c_1^b \epsilon/\gBernstein(\epsilon)$. We deal with $\epsilon = 0$
  later and for now focus on the case $\epsilon > 0$, which implies
  $\eta > 0$. Then  Lemma~\ref{lem:MomentTaylor}, applied to the random
  variable $\eta X$, gives
  \begin{equation*}
    \E[X] +\frac{1}{\eta} \log \E[e^{-\eta X}]
      \leq \kappa(2b a)\eta \Var(X)
      \leq \kappa(2b a)\eta \gBernstein(\E[X])
      \leq \frac{\epsilon}{\gBernstein(\epsilon)} \gBernstein(\E[X]).
  \end{equation*}
  If $\epsilon \leq \E[X]$, then the assumption that
  $\frac{\gBernstein(\epsilon)}{\epsilon}$ is non-increasing in $\epsilon$
  implies that
  \begin{equation}
    \frac{\epsilon}{\gBernstein(\epsilon)} \gBernstein(\E[X])
    \leq \frac{\E[X]}{\gBernstein(\E[X])} \gBernstein(\E[X])
    = \E[X],
  \end{equation}
  and we can conclude that $\frac{1}{\eta} \log \E[e^{-\eta X}] \leq 0
  \leq \epsilon$. This inequality establishes (b), and it establishes
  (a) for the case $0 < \epsilon \leq \E[X]$.  If $\epsilon > \E[X]$,
  then the assumption that $\gBernstein$ is non-decreasing implies
  that
  \begin{equation}
    \frac{\epsilon}{\gBernstein(\epsilon)} \gBernstein(\E[X])
    \leq \frac{\epsilon}{\gBernstein(\E[X])} \gBernstein(\E[X])
    = \epsilon,
  \end{equation}
  and, using that $\E[X] \geq 0$, we again find that $\frac{1}{\eta}
  \log \E[e^{-\eta X}] \leq \epsilon$, as required for (a). To finish the proof of (a) we now consider $\epsilon = 0$. If we also have $\gComparator(0)= 0$ then the
  central condition \eqref{eqn:comparator} holds trivially for
  $\epsilon = 0$, so we may assume without loss of generality that
  $\gComparator(0) > 0$. Then we must have $\eta = \gComparator(0) =
  \liminf_{x \downarrow 0} x/\gBernstein(x) > 0$. Now fix a decreasing
  sequence $\{ \epsilon_j\}_{j = 1, 2, \ldots }$ tending to $0$, where
  the $\epsilon_j$ are all positive and let $\eta_j =
  \gComparator(\epsilon_j)$. By the argument above, the
  $\eta_j$-central condition holds up to $\epsilon_j$. This implies (Fact~\ref{fact:ccppcc})
that for all $j$, all $\eta \leq \eta_j$, in particular for $\eta=
  \gComparator(0)$, the $\eta$-central condition also holds up to
  $\epsilon_j$. Thus, the $\eta$-central condition holds up to
  $\epsilon$ for all $\epsilon > 0$. By Proposition~\ref{cor:teeth} it
  then follows that the strong $\eta$-central condition holds, i.e. it
  also holds for $\epsilon = 0$.

  \paragraph{\rm {\em Pseudoprobability  $\Rightarrow$ Bernstein}.}
  
Suppose that the $\gComparator$-PPC condition holds. Fix some $\epsilon\geq  0$ and let
$\eta = \gComparator(\epsilon)$. Fix arbitrary $P \in \cP$ and let
$f^*$ be $\cF$-optimal for $P$, achieving (\ref{eq:bayesact}). Fix arbitrary  $f \in \cF$ and let
$\Pi$ be the distribution on $\cF$ assigning mass $1/2$ to $f^*$ and
mass $1/2$ to $f$, and let $\bar{f} \in \{ f, f^*\}$ be the
corresponding random variable.  For $z \in \cZ$, let $Y_{z,\bar{f}} = \eta(
\loss_{\bar{f}}(z) - \loss_{f^*}(z))$ and let $\epsilon_z = \eta^{-1}
\log \E_{\bar{f} \sim \Pi} \left[ e^{- Y_{z,\bar{f}}}\right]$.  Note that $Y_{z,\bar{f}}$
is a random variable under distribution $\Pi$ (not $P$, since $z$ is
fixed), and that
\begin{equation}\label{eq:expy}
\E_{\bar{f} \sim \Pi} [Y_{z,\bar{f}}] = \frac{1}{2} \eta \left( \loss_{f}(z) - \loss_{f^*}(z)\right).
\end{equation}
Lemma~\ref{lem:MomentTaylor} then gives, for each $z \in \cZ$,
\begin{equation}\label{eq:var}
  \kappa(-2ab) \Var_{\bar{f} \sim \Pi} [Y_{z,\bar{f}}] \leq \E_{\bar{f} \sim \Pi} [Y_{z,\bar{f}}] + \log \E_{\bar{f} \sim \Pi}\left[e^{-Y_{z,\bar{f}} }\right] 
  = \frac{1}{2} \eta \left( \loss_{f}(z) - \loss_{f^*}(z)\right)+ \eta \epsilon_z,\end{equation}
where we used the definition of $\Pi$ and $\epsilon_z$. We may assume from the definition of 
the $\gComparator$-pseudoprobability convexity condition 
that 
\eqref{eqn:strong-convexface} holds for the given $\epsilon$ and $\eta$ and $\Pi$; rearranging this equation it is seen to be equivalent to $
\Exp_{Z \sim P} [\epsilon_Z] \leq \epsilon. 
$
By taking expectations over $Z$ on both sides of  (\ref{eq:var}) this gives
\begin{equation}\label{eq:varc}
\kappa(-2ab) \Exp_{Z \sim P} \Var_{\bar{f} \sim \Pi}\left[Y_Z\right] \leq \frac{1}{2} \eta \Exp_{Z \sim P} \left[ \loss_{f}(Z) - \loss_{f^*}(Z)\right] + \eta \epsilon.
\end{equation}
The $\Pi$-variance on the left can be rewritten, using (\ref{eq:expy}), as
\begin{align*}
\Var_{\bar{f} \sim \Pi}\left[Y_{z,\bar{f}}\right] & =   \frac{1}{2} 
\left(  \eta (\loss_{{f}}(z) - \loss_{f^*}(z))
- \Exp_{\bar{f} \sim \Pi}\left[Y_{z,\bar{f}}\right]\right)^2 + \frac{1}{2} \left( 
\eta \cdot 0
- \Exp_{\bar{f} \sim \Pi}\left[Y_{z,\bar{f}}\right]\right)^2  \\ & = 
\frac{1}{2} \left(\frac{1}{2} \eta (\loss_{{f}}(z) - \loss_{f^*}(z)) \right)^2 
+ \frac{1}{2} \left(- \frac{1}{2} \eta (\loss_{{f}}(z) - \loss_{f^*}(z)) \right)^2
=  \frac{1}{4} \eta^2 (\loss_{{f}}(z) - \loss_{f^*}(z))^2.
\end{align*}
Plugging this into (\ref{eq:varc}) and dividing both sides by $\eta^2 / (4 \kappa(-2 ab))$ gives 
\begin{equation}\label{eq:vard}
\Exp_{Z \sim P} (\loss_{{f}}(Z) - \loss_{f^*}(Z))^2 \leq 
\frac{2}{\kappa(-2ab) \cdot \eta} \left( \Exp_{Z \sim P} \left[ \loss_{f}(Z) - \loss_{f^*}(Z)\right] + 2 \epsilon \right).
\end{equation}
This holds for all $\epsilon \geq 0$ and $\eta =
\gComparator(\epsilon)$, as long as $\eta = \gBernstein(\epsilon) > 0$
(if $\eta = 0$ we cannot divide by $\eta^2$ to go from (\ref{eq:varc})
to (\ref{eq:vard})).  Thus, we may set $\epsilon = \Exp_{Z \sim P}
\left[ \loss_{f}(Z) - \loss_{f^*}(Z)\right] \geq 0$; if $\eta =
\gBernstein(\epsilon) > 0$ then (\ref{eq:vard}) must hold for
$\epsilon$. With these values the right-hand side becomes $6 \eta^{-1}
\kappa^{-1}(2ab) \epsilon = c_2 \epsilon/ \gComparator(\epsilon) =
\gBernstein(\epsilon),$ and the result follows by our choice of
$\epsilon$. It remains to deal with the case $\eta = 0$, which by
definition of $\gComparator$ can only happen if $\epsilon = \Exp_{Z
  \sim P} \left[ \loss_{f}(Z) - \loss_{f^*}(Z)\right] = 0$. In this
case, (\ref{eq:vard}) still holds for all values of $\epsilon > 0$. We
thus infer that the left-hand side of (\ref{eq:vard}) is bounded by
$\inf_{\epsilon > 0} 4 \epsilon/ (\kappa(-2ab) \gComparator (\epsilon)$,
and the result follows by our definition of $0/\gComparator(0)$.
\end{proof}

\subsection{Proofs for Section~\ref{sec:fast-rates}}
\label{app:fastratesproof}

\begin{lemma}{ \bf (Hyper-Concentrated Excess Losses)}
\label{lemma:hyper-concentrated}
Let $Z$ be a random variable with probability measure $P$ supported on $[-V, V]$. Suppose that $\lim_{\eta \rightarrow \infty} \E [ \exp(-\eta Z) ] < 1$ and $\E [ Z ] = \mu > 0$. Then there is a suitable modification $Z'$ of $Z$ for which $Z' \leq Z$ with probability 1, the mean of $Z'$ is arbitrarily close to $\mu$, and $\E [ \exp(-\eta Z') ] = 1$ for arbitrarily large $\eta$.
\end{lemma}

\begin{proof}
First, observe that $Z \geq 0$ a.s. If not, then there must be some finite $\eta > 0$ for which $\E [ \exp(-\eta Z) ] = 1$. 
Now, consider a random variable $Z'$ with probability measure $Q_\epsilon$, a modification of $Z$ (with probability measure $P$) constructed in the following way. 
Define $A := [\mu, V]$ and $A^- := [-V, -\mu]$. Then for any $\epsilon > 0$ we define $Q_\epsilon$ as
\begin{align*}
{\mathrm d} Q_\epsilon(z) = 
\begin{cases} 
  (1 - \epsilon) {\mathrm d} P(z) & \text{if } z \in A \\
  \epsilon {\mathrm d} P(-z) & \text{if } z \in A^- \\
  {\mathrm d} P(z) & \text{otherwise} .
\end{cases}
\end{align*}

Additionally, we couple $P$ and $Q_\varepsilon$ such that the couple $(Z, Z')$ is a coupling of $(P, Q_\epsilon)$ satisfying
\begin{align*}
\E_{(Z, Z') \sim (P, Q_\epsilon)} \ind{Z \neq Z'} = \min_{(P', Q_\epsilon')} \E_{(Z,Z') \sim (P', Q_\epsilon')} \ind{Z \neq Z'} ,
\end{align*}
where the $\min$ is over all couplings of $P$ and $Q_\varepsilon$. 
This coupling ensures that  $Z' \leq Z$ with probability 1; i.e. $Z'$ is dominated by $Z$.

Now,
\begin{align}
\E [ \exp(-\eta Z') ] 
&= \int_{-V}^V e^{-\eta z}  {\mathrm d} Q_\epsilon(z) \nonumber \\
&= \int_{A^-} e^{-\eta z} {\mathrm d} Q_\epsilon(z) + \int_A e^{-\eta z}  {\mathrm d} Q_\epsilon(z) + \int_{[0, V] \setminus A} e^{-\eta z} {\mathrm d} Q_\epsilon(z) \nonumber \\
&= \epsilon \int_{A^-} e^{-\eta z}  {\mathrm d}P(-z) + (1 - \epsilon) \int_A e^{-\eta z} {\mathrm d} P(z) + \int_{[0, V] \setminus A} e^{-\eta z}  {\mathrm d} P(z) \nonumber \\
&= \epsilon \int_{A} e^{\eta z}  {\mathrm d} P(z) + (1 - \epsilon) \int_A e^{-\eta z}  {\mathrm d} P(z) + \int_{[0, V] \setminus A} e^{-\eta z} {\mathrm d} P(z) \nonumber \\
&\geq\epsilon e^{\mu \eta} P(A) + (1 - \epsilon) \int_A e^{-\eta z}  {\mathrm d} P(z) + \int_{[0, V] \setminus A} e^{-\eta z} {\mathrm d} P(z) \label{eqn:last-line-hyper} .
\end{align}

Now, on the one hand, for any $\eta > 0$, the sum of the two right-most terms in \eqref{eqn:last-line-hyper} is strictly less than 1 by assumption. On the other hand, $\eta \rightarrow \epsilon P(A) e^{\mu \eta}$ is exponentially increasing since $\epsilon > 0$ and $\mu > 0$ (and hence $P(A) > 0$ as well) by assumption; thus, the first term in \eqref{eqn:last-line-hyper} can be made arbitrarily large by increasing $\eta$. Consequently, we can choose $\epsilon > 0$ as small as desired and then choose $\eta < \infty$ as large as desired such that the mean of $Z'$ is arbitrarily close to $\mu$ and $\E [ \exp(-\eta Z') ] = 1$ respectively.
\end{proof}

\begin{proof}{(of Lemma \ref{lemma:feasible-moments})} 
Let $W$ denote the convex hull of $g([-1, 1])$. We need to see if
$\left( -\frac{a}{n}, 1 \right) \in W$. Note that $W$ is the convex set formed by starting with the graph of $x \mapsto e^{{\eta^*} x}$ on the domain $[-1, 1]$, including the line segment connecting this curve's endpoints $(-1, e^{-{\eta^*}})$ to $(1, e^{{\eta^*} x})$, and including all of the points below this line segment but above the aforementioned graph. That is, $W$ is precisely the set 
\begin{align*} 
W = \left\{ (x,y) \in \reals^2 :  e^{{\eta^*} x} \leq y \leq \frac{e^{{\eta^*}} + e^{-{\eta^*}}}{2} + \frac{e^{{\eta^*}} - e^{-{\eta^*}}}{2} x , 
\, x \in [-1, 1]  \right\} . 
\end{align*} 
We therefore need to check that $-1 \leq -\frac{a}{n} \leq 1$ 
and that $1$ is sandwiched between the lower and upper bounds at $x = -\frac{a}{n}$. 
Clearly $-1 \leq -\frac{a}{n} \leq 1$ holds since the loss is in $[0, 1]$ by assumption. 
Using that 
$\cosh({\eta^*}) = \frac{e^{{\eta^*}} + e^{-{\eta^*}}}{2}$ 
and $\sinh({\eta^*}) = \frac{e^{{\eta^*}} - e^{-{\eta^*}}}{2}$, 
this means that $k \in W$ if and only if
\begin{equation*}
  e^{-{\eta^*} a/n} \leq 1 \leq \cosh({\eta^*}) + \sinh({\eta^*}) \frac{-a}{n}.
\end{equation*}
Also, since $a > 0$ the inequality $e^{-{\eta^*} a/n} \leq 1$ holds with \emph{strict} inequality. Thus, we end up with a single requirement characterizing when $k \in W$, which is equivalent to condition
\eqref{eqn:feasibility}. Moreover, $k \in \interior W$ is characterized by when \eqref{eqn:feasibility} holds strictly.
\end{proof}

\begin{proof}{(of \cref{thm:stochastic-mixability-concentration})} 
By assumption, the condition of Lemma~\ref{lemma:feasible-moments} is
satisfied, so we can apply Theorem~3 of \cite{kemperman1968general}.
This gives
\begin{equation}\label{eqn:toUse}
  -\exp\left(\Lambda_{-\xslossat{Z}}(\eta^* / 2)\right) \geq d_0 - \frac{a}{n} d_1 +
  d_2 ,
\end{equation}
for all $d^* = (d_0, d_1, d_2) \in \reals^3$ such that
\begin{equation}\label{eqn:constraint}
  d_0 + d_1 s + d_2 e^{\eta^* s} + e^{(\eta^*/2) s} \leq  0
  \qquad \text{ for all } s \in [-1, 1].
\end{equation}
To find a good choice of $d^*$, we will restrict attention to those
$d^*$ for which \eqref{eqn:constraint} holds with equality at $s = 0$,
yielding the constraint
\begin{equation}\label{eqn:d_0-constraint}
  d_0 = - d_2 - 1.
\end{equation}
Plugging this into \eqref{eqn:constraint} and changing variables to $c_1
= -d_1 / \eta$,\footnote{We scale by $\eta$ here because we are chasing
a certain $\eta$-dependent rate.} and $c_2 = -d_2$, we obtain the
constraint
\begin{align*}
u(s) := 1 + c_2 (e^{\eta s} - 1) -e^{(\eta/2) s} + \eta c_1 s \geq 0
\qquad \text{for all $s \in [-1,1]$.}
\end{align*}

\subsubsection{Constraints from the Local Minimum at $\mathbf{0}$}

Since $u(0) = 0$, we need $s = 0$ to be a local minimum of $u$, and so
we require the first and second derivative to satisfy
\begin{enumerate}[label=(\alph*)]
\item $u'(0) = 0$
\item $u''(0) \geq 0,$
\end{enumerate}
since otherwise there exists some small $\varepsilon > 0$ such that either $u(\varepsilon) < 0$ or $u(-\varepsilon) < 0$. 

For (a), we compute
\begin{align*}
u'(s) = \eta c_2 e^{\eta s} -\frac{\eta}{2} e^{(\eta/2) s} + \eta c_1 .
\end{align*}
Since we require $u'(0) = 0$, we pick up the constraint
\begin{align*}
\eta \left( c_2 -\frac{1}{2} + c_1 \right) = 0 ,
\end{align*}
and since $\eta > 0$ by assumption, we have 
\begin{align}
c_1 = \frac{1}{2} - c_2 . \label{eqn:c_1-constraint}
\end{align}
Thus, we can eliminate $c_1$ from $u(s)$:
\begin{align*}
u(s) = 1 + c_2 (e^{\eta s} - 1) -e^{(\eta/2) s} + \eta \left(
\frac{1}{2} - c_2 \right) s.
\end{align*}

For (b), observe that
\begin{align*}
u''(s) = \eta^2 c_2 e^{\eta s} - \frac{\eta^2}{4} e^{(\eta / 2) s} ,
\end{align*}
so that $u''(0) = \eta^2 \left( c_2 - \frac{1}{4} \right) \geq 0$, and
hence we require
\begin{align}
c_2 \geq \frac{1}{4} \label{eqn:c_2-one-fourth-lower-bound}.
\end{align}

\subsubsection{The Other Minima of $u$}

Thus far, we have picked up the constraints \eqref{eqn:d_0-constraint},
\eqref{eqn:c_1-constraint}, and \eqref{eqn:c_2-one-fourth-lower-bound},
and it remains to choose a value of $c_2$ such that $u(s) \geq 0$ for
all $s \in [-1,1]$. To this end, observe that $u'(s)$ has at most two
roots, because with the substitution $y = e^{(\eta / 2) s}$, we have
\begin{align*}
u'(s) = \eta c_2 y^2 - \frac{\eta}{2} y + \eta \left( \frac{1}{2} - c_2 \right) ,
\end{align*}
which is a quadratic equation in $y$ with two roots:
\begin{align*}
y \in \left\{\frac{1 - 2 c_2}{2 c_2}, 1 \right\} 
\quad \Rightarrow \quad 
s \in \left\{\frac{2}{\eta} \log \frac{1 - 2 c_2}{2 c_2}, 0 \right\} .
\end{align*}
Now, since we are taking $c_2 \geq \frac{1}{4}$, the first root is
negative, and we find that $u$ is non-decreasing on $[0, 1]$. As we
already ensured that $u(0) = 0$, this means that $u$ is non-negative on
$[0,1]$. On the remaining interval, $[-1,0]$, we know that $u$ is
increasing up to $\frac{2}{\eta} \log \frac{1 - 2 c_2}{2 c_2}$ and then
decreasing until $s = 0$. Since $u(0) = 0$, we therefore need to ensure
only that $u(-1) \geq 0$ by finding appropriate conditions on $c_2$,
where
\begin{align*}
u(-1) 
&= 1 + c_2 (e^{-\eta} - 1) -e^{-(\eta/2)} - \eta \left( \frac{1}{2} - c_2 \right) \\
&= \left( 1 - \frac{\eta}{2} \right)  - e^{-(\eta/2)}  
      + c_2 \left( e^{-\eta} - (1 - \eta) \right) \\
c_2 &\geq \frac{e^{-\eta/2} + \frac{\eta}{2} - 1}{e^{-\eta}+ \eta -1}
  = \frac{1}{4} \frac{\kappa(-\eta/2)}{\kappa(-\eta)},
\end{align*}
where $\kappa(x) = (e^x - x - 1)/x^2$ is increasing in $x$, which
implies that this condition always ensures that $c_2 \geq 1/4$.

We consider the cases $\eta \leq 1$ and $\eta > 1$ separately.

\paragraph{\rm {\em Case ${\eta \leq 1}$.}}

For $\eta \leq 1$, we will take the value of the constraint at $\eta =
1$. That is,
\begin{equation*}
  c_2 = \frac{1}{4} \frac{\kappa(-1/2)}{\kappa(-1)}
      = e^{1/2} - \frac{e}{2}.
\end{equation*}
This is allowed because $\frac{\kappa(-\eta/2)}{\kappa(-\eta)}$ is
non-decreasing, as may be verified by observing that
\begin{equation*}
  \frac{\der}{\der \eta} \frac{e^{-\eta/2} + \frac{\eta}{2} -
  1}{e^{-\eta}+ \eta -1}
    = \frac{e^{\eta/2}(e^{\eta/2}-1)(e^{\eta} -1 + e^{\eta/2}\eta)}
      {2(1 + e^{\eta} (\eta-1))^2},
\end{equation*}
which is non-negative if $g(\eta) = e^{\eta} -1 + e^{\eta/2}\eta \geq
0$. This in turn is verified by noting that $g(0) = 0$ and $g'(\eta) =
e^{\eta/2} (e^{\eta/2} - \frac{\eta}{2} -1)$ is positive.

\paragraph{\rm {\em Case ${\eta > 1}$.}}
Let $c_2 = \frac{1}{2} - \frac{\alpha}{\eta}$ for some $\alpha \geq 0$. With this substitution, we have
\begin{align*}
u(-1) 
&= 1 + c_2 (e^{-\eta} - 1) -e^{-(\eta/2)} - \eta \left( \frac{1}{2} - c_2 \right) \\
&= 1 + \left( \frac{1}{2} - \frac{\alpha}{\eta} \right) (e^{-\eta} - 1) -e^{-(\eta/2)} - \alpha \\
&= \left( \frac{1 + e^{-\eta}}{2} - e^{-\eta/2} \right) 
      + \alpha \left( -1 + \frac{1}{\eta} \left( 1 -e^{-\eta} \right) \right) .
\end{align*}
Since we want the above to be nonnegative for all $\eta > 1$, we arrive at the condition
\begin{align}
\alpha
\leq \inf_{\eta \geq 1} \left\{ 
  \frac{ \frac{1 + e^{-\eta}}{2} - e^{-\eta/2} }{ 1 - \frac{1}{\eta} \left( 1 -e^{-\eta} \right) } 
\right\} .
\label{eqn:alpha-condition}
\end{align}
Plotting suggests that the minimum is attained at $\eta = 1$, with the value $\frac{1}{2} (\sqrt{e} - 1)^2=0.2104\ldots$. We will fix $\alpha$ to this value and verify that 
\begin{align}
\left( \frac{1 + e^{-\eta}}{2} - e^{-\eta/2} \right) 
 + \left( \frac{1}{2} (\sqrt{e} - 1)^2 \right) \left( -1 + \frac{1}{\eta} \left( 1 -e^{-\eta} \right) \right) \geq 0 . \label{eqn:big-expression}
\end{align}
This is true with equality at $\eta = 0$. 
The derivative of the LHS with respect to $\eta$ is
\begin{align*}
\frac{1}{2} e^{-\eta} \left( e^{\eta/2} - 1 
- \frac{(\sqrt{e} - 1)^2 (e^\eta - \eta - 1)}{\eta^2} \right) .
\end{align*}
The derivative is positive at $\eta = 1$, so 0 is a candidate minimum. Eventually, $\frac{(\sqrt{e} - 1)^2 (e^\eta - \eta - 1)}{\eta^2}$ grows more quickly than $e^{\eta/2} - 1$ and surpasses the latter in value. 
The derivative is therefore negative for all sufficiently large $\eta$, and so we need only take the minimum of the LHS of \eqref{eqn:big-expression} evaluated at $\eta = 1$ and the limiting value as $\eta \rightarrow \infty$. 
We have 
\begin{align*}
\lim_{\eta \rightarrow \infty} \left( \frac{1 + e^{-\eta}}{2} - e^{-\eta/2} \right) 
 + \left( \frac{1}{2} (\sqrt{e} - 1)^2 \right) \left( -1 + \frac{1}{\eta} \left( 1 -e^{-\eta} \right) \right) = \sqrt{e} - \frac{e}{2} \geq 0 .
\end{align*}
Hence, \eqref{eqn:big-expression} indeed holds for $\alpha \leq 0.21 \leq  \frac{1}{2} (\sqrt{e} - 1)^2$. We conclude that $u(-1) \geq 0$ when $\alpha \leq \frac{1}{2} (\sqrt{e} - 1)^2$. 

\subsubsection{Putting it All Together}

Tracing back our substitutions, we have $d_0 + d_2 = -1$ and $d_1 = -\eta/2 +
\eta c_2$, which gives
\begin{equation*}
  d_0 - \frac{a}{n} d_1 + d_2
    = -1 + \frac{a \eta }{n}\left(\half - c_2\right)
    \geq -e^{-\frac{a \eta }{n}\left(\half - c_2\right)}.
\end{equation*}
In the regime $\eta \leq 1$, we choose $c_2 = e^{1/2} - e/2$, which
leads to 
\begin{equation}
  d_0 - \frac{a}{n} d_1 + d_2
    \geq -e^{-\frac{0.21 \eta a}{n}}.
\end{equation}
In the regime $\eta > 1$, we take $c_2 = \frac{1}{2} - \frac{1}{2 \eta}
(\sqrt{e} - 1)^2$, which gives
\begin{equation}
  d_0 - \frac{a}{n} d_1 + d_2
    \geq -e^{-\frac{a}{2 n}}.
\end{equation}
Combining with \eqref{eqn:toUse} leads to the desired result.
\end{proof}

\begin{proof}{(of Corollary \ref{cor:stochastic-mixability-concentration})} 
Define the function $\Gamma(\eta) := \frac{\cosh(\eta) - 1}{\sinh(\eta)}$. 
For any negative excess loss random variable $S'$, let $\eta_{S'}$ be the maximum $\eta$ for which $-S'$ is stochastically mixable.

Let $W$ be a stochastically mixable excess loss random variable taking values in $[-1,1]$ and satisfying $\E [ W ] = \Gamma(\eta_S) > 0$, and let $S = -W$ be the corresponding negative excess loss random variable. 

Let $k_S \in \reals^2$ be the moments vector of $S$, defined as
\begin{align*}
k_S := 
\begin{pmatrix} 
\E [ S ] \\
\E [ e^{\eta_S S} ]
\end{pmatrix} 
=
\begin{pmatrix} 
-\Gamma(\eta_s) \\
1
\end{pmatrix} .
\end{align*}
Because $-\E[ S ] = \Gamma(\eta_S)$, from Lemma
\ref{lemma:feasible-moments} the point $k_S$ is extremal with respect
to \\ $\convhull(g([-1,1]))$. Recall that the goal of this proof is to establish that \cref{thm:stochastic-mixability-concentration} holds even for the extremal random variable $S$.

Since $\E [ S ] < 0$, there exists $A \subset \{x \in \reals \colon x < 0\}$ for which we have $\Pr(S \in A) =: p > 0$. 
Now, consider the following two perturbed versions of $S$, which we call (I) and (II). In both perturbations, we deflate $\Pr(S \in A)$ by the same (multiplicative) factor $\varepsilon > 0$ uniformly over $A$ so that the overall loss in probability mass over $A$ is $\varepsilon$; this is always possible for small enough $\varepsilon$ since $p > 0$, and throughout the rest of the proof we keep implicit that $\varepsilon$ is suitably small. The perturbations differ in where they allocate the mass taken from $A$:
\begin{enumerate}[label=(\Roman*)]
\item Allocate $\varepsilon$ additional mass to $\frac{3}{4}$.
\item Allocate $\frac{\varepsilon}{2}$ additional mass to $\frac{1}{2}$ and $\frac{\varepsilon}{2}$ additional mass to $1$.
\end{enumerate}

We refer to these new random variables as $S_I$ and $S_{II}$. Observe that
\begin{align*}
\E[ S_I ] = \E[ S_{II} ] \geq \E[ S ] + \frac{3}{4} \varepsilon .
\end{align*}
Because $\E[ S_I ] = \E[ S_{II} ]$, it follows that if we can show that $\eta_{S_I} \neq \eta_{S_{II}}$, then $k_{S_I}$ and $k_{S_{II}}$ cannot both are extremal since $\Gamma$ is strictly increasing.

Now, by definition, $\E \exp \left( \eta_{S_I} S_I \right) = 1$. But observe that by strict convexity, for any $\eta > 0$, we have
\begin{align*}
e^{3 \eta / 4} < \frac{1}{2} \left( e^{\eta/2} + e^\eta \right) .
\end{align*}
Therefore, $\E [ \exp \left( \eta_{S_I} S_I \right) ] > 1$, and so $\eta_{S_{II}} < \eta_{S_I}$. Therefore, $k_{S_I}$ cannot be extremal, and \cref{thm:stochastic-mixability-concentration} can be applied to the excess loss random variable $-S_I$.

Now, for each (suitably small) $\varepsilon$, we refer to the corresponding $S_I$ more precisely via the notation $S_\varepsilon$, and we define $\eta_\varepsilon := \eta_{S_\varepsilon}$. Since for all $\varepsilon > 0$,
\begin{align*}
\left| \exp \left( \frac{\eta_\varepsilon}{2} S_\varepsilon \right) \right| \leq \exp \left( \frac{\eta_S}{2} \right) ,
\end{align*}
and since for each $S_\varepsilon$ we have
\begin{align*}
\E \left[ \exp \left( \frac{\eta_\varepsilon}{2} S_\varepsilon \right) \right] 
\leq 1 - 0.21 (\eta_\varepsilon \bmin 1) \E [ -S_\varepsilon ] ,
\end{align*}
from the dominated convergence theorem it follows that
\begin{align*}
\E \left[ \exp \left( \frac{\eta_S}{2} S \right) \right] \leq 1 - 0.21 (\eta_S \bmin 1) \E [ -S ] ,
\end{align*}
i.e.~using the familiar notation $\eta^* = \eta_S$:
\begin{align*}
\E \left[ \exp \left( -\frac{\eta^*}{2} W \right) \right] \leq 1 - 0.21 (\eta^* \bmin 1) \E [ W ] .
\end{align*}
\end{proof}

\begin{proof}{(of Corollary \ref{cor:bounded-losses})}
Let $X$ be a random variable taking values in $[-\bound, \bound]$ with mean $-\frac{a}{n}$ and $\E [ e^{\eta X} ] = 1$, and let $Y$ be a random variable taking values in $[-1, 1]$ with mean $-\frac{a / \bound}{n}$ and $\E [ e^{(\bound \eta) Y} ] = 1$. 
Consider a random variable $\tilde{X}$ that is a $\frac{1}{\bound}$-scaled independent copy of $X$; observe that $\E [ \tilde{X} ] = -\frac{a / \bound}{n}$ and $\E [ e^{(\bound \eta) \tilde{X}} ] = 1$. 
Let the maximal possible value of $\E [ e^{(\eta / 2) X} ]$ be $b_X$, and let the maximal possible value of $\E [ e^{(\bound \eta / 2) Y} ]$ be $b_Y$. We claim that $b_X = b_Y$. Let $X$ be a random variable with a distribution that maximizes $\E [ e^{(\eta / 2) X} ]$ subject to the previously stated constraints on $X$. Since $\tilde{X}$ satisfies $\E [ e^{(\bound \eta / 2) \tilde{X}} ] = b_X$, setting $Y = \tilde{X}$ shows that in fact $b_Y \geq b_X$. A symmetric argument (starting with $Y$ and passing to some $\tilde{Y} = \bound Y$) implies that $b_X \geq b_Y$.
\end{proof}

\begin{proof}{(of \cref{thm:finite-fast-rates})}
Let $\gamma_n = \frac{a}{n}$ for a constant $a$ to be fixed later. 
For each $\eta > 0$, let $\F_{\gamma_n}^{(\eta)} \subset \F_{\gamma_n}$ correspond to those functions in $\F_{\gamma_n}$ for which $\eta$ is the largest constant such that 
$\E [ \exp(-\eta W_f) ] = 1$. 
Let $\F_{\gamma_n}^{\mathrm{hyper}} \subset \F_{\gamma_n}$ correspond to functions $f$ in $\F_{\gamma_n}$ for which $\lim_{\eta \rightarrow \infty} \E [ \exp(-\eta W_f) ] < 1$. 
Clearly, $\F_{\gamma_n} = \bigl( \bigcup_{\eta \in [\eta^*, \infty)} \F_{\gamma_n}^{(\eta)} \bigr) \cup \F_{\gamma_n}^{\mathrm{hyper}}$. 
The excess loss random variables corresponding to elements $f \in \F_{\gamma_n}^{\mathrm{hyper}}$ are `hyper-concentrated' in the sense that they are infinitely stochastically mixable. However, Lemma \ref{lemma:hyper-concentrated} above shows that for each hyper-concentrated $W_f$, there exists another excess loss random variable $W'_f$ with mean arbitrarily close to that of $W_f$, with $\E [ \exp(-\eta W'_f) ] =1$ for some arbitrarily large but finite $\eta$, and with $W'_f \leq W_f$ with probability 1. The last property implies that the empirical risk of $W'_f$ is no greater than that of $W_f$; hence for each hyper-concentrated $W_f$ it is sufficient (from the perspective of ERM) to study a corresponding $W'_f$. From now on, we implicitly make this replacement in $\F_{\gamma_n}$ itself, so that we now have $\F_{\gamma_n} = \bigcup_{\eta \in [\eta^*, \infty)} \F_{\gamma_n}^{(\eta)}$.

Consider an arbitrary $a > 0$. For some fixed $\eta \in [\eta^*, \infty)$ for which $| \F_{\gamma_n}^{(\eta)} | > 0$, consider the subclass $\F_{\gamma_n}^{(\eta)}$. 
Individually for each such function, we will apply Lemma~\ref{lem:cramer-chernoff} as follows. From Lemma \ref{cor:bounded-losses}, 
we have $\Lambda_{-W_f}(\eta / 2) = \Lambda_{-\frac{1}{\bound} W_f}(\bound \eta / 2)$. From Corollary \ref{cor:stochastic-mixability-concentration}, the latter is at most
$-\frac{0.21 (\bound \eta \operatorname{\wedge} 1) (a / \bound)}{n} 
= -\frac{0.21 \eta a}{(\bound \eta \operatorname{\vee} 1) n}$ .
Hence, Lemma~\ref{lem:cramer-chernoff} with $t = 0$ and the $\eta$ from
the lemma taken to be $\eta / 2$ implies that the probability of the event $\Probn \lossof{f} \leq \Probn \lossof{f^*}$ is at most 
$\exp\left( -0.21 \frac{\eta}{\bound \eta \operatorname{\vee} 1} a \right)$. 
Applying the union bound over all of $\F_{\gamma_n}$, we conclude that
\begin{align*}
\Pr \left\{ 
\exists f \in \F_{\gamma_n} : \Probn \loss_f \leq \Probn \loss_{f^*} 
\right\} 
\leq N \exp \left( -\eta^* \left( \frac{0.21 a}{\bound \eta^* \operatorname{\vee} 1} \right) \right) .
\end{align*}

Since ERM selects hypotheses on their empirical risk, from inversion it holds that with probability at least $1 - \delta$ ERM will not select any hypothesis with excess risk at least 
$\frac{5 \max\left\{ \bound, \frac{1}{\eta^*} \right\} \left( \log \frac{1}{\delta} + \log N 
\right)}{n}$.
\end{proof}

\bibliography{stoch_mix_journal}

\end{document}